%% file: main.tex
\documentclass{article}
\usepackage[T1]{fontenc}
\usepackage{vendor/iclr2027/iclr2027_conference,times}
\usepackage{amsmath,amssymb,amsthm,booktabs,graphicx,longtable}
\usepackage{url,float}
\newfloat{algorithm}{tbp}{loa}
\floatname{algorithm}{Algorithm}
\usepackage{hyperref}
\input{macros}
\iclrfinalcopy
\hypersetup{hidelinks,pdftitle={Target-Dependent Limits of Causal Repair: A Leading-Log Frontier in a Gaussian Model},pdfauthor={Qinchuan Cheng; Jiaqi Liu; Ruixuan Xie},pdfsubject={Preprint}}
\title{Target-Dependent Limits of Causal Repair:\\A Leading-Log Frontier in a Gaussian Model}
\author{%
Qinchuan Cheng$^{1}$ \And Jiaqi Liu$^{2}$ \And Ruixuan Xie$^{3}$\\
$^{1}$Xi'an Jiaotong University \\
$^{2}$Sichuan University \\
$^{3}$Southwestern University of Finance and Economics
}
\begin{document}
\maketitle
\begin{abstract}
Knowing how much a causal predictor could improve need not reveal the gain of the repair actually learned. We quantify this gap in a scalar Gaussian causal experiment with known intervention geometry: auxiliary data identify effect magnitude up to bounded contamination, while diagnostics identify direction. The target is the squared-loss gain of the realized trained repair relative to a fitted reference. Jointly optimizing the learner and assessor under uniform learning MSE $\eta$ avoids the trivial solution of making no repair. At the usual $1/k$ learning scale, every feasible learner incurs a $k^{-2}$ assessment floor, even when oracle potential is estimable at a faster rate. In the magnitude-rich regime, we characterize a sharp \emph{leading-log} frontier: the assessment exponent is $\min\{\ell_k,2k\eta_k/U\}$ to first relative order, where $\ell_k=\log(1/(k^2E_k))$ and $E_k$ is auxiliary precision. A diagnostic-abstention rule attains this exponent with unknown nuisance parameters. We also bound the critical allowance window and transfer the frontier to adaptive sampling by exact Gaussian simulation. Finite-grid experiments distinguish sign-tail suppression from total MSE and expose conservative finite-budget behavior. The result isolates how the assessment target changes information requirements in this experiment; it is not a general causal identifiability claim.
\end{abstract}
\input{sections/introduction}
\input{sections/model}
\input{sections/frontier}
\input{sections/construction}
\input{sections/targets}
\input{sections/experiments}
\input{sections/discussion}
\input{sections/related}
\section{Conclusion}
In this Gaussian experiment, learning an accurate repair, assessing its realized gain, and assessing oracle potential require different information. A uniform learning constraint prevents a vacuous zero-gain solution; diagnostic abstention then trades some learning accuracy for a larger assessment exponent. The leading-log frontier is sharp, the corrected critical allowance is localized within a proved window, and exact simulation carries these limits to adaptive designs. The finite-budget comparisons show why these statements should guide target selection and asymptotic reasoning without being read as universal numerical superiority. Extending the attainable frontier beyond identifiable magnitude channels and known geometry remains open.
\label{main:end}
\clearpage
\input{sections/statements}
\bibliography{references}
\bibliographystyle{vendor/iclr2027/iclr2027_conference}
\clearpage
\appendix
\input{appendix/roadmap}
\input{appendix/causal}
\input{appendix/scope}
\input{appendix/quality}
\input{appendix/sharp}
\input{appendix/critical_window}
\input{appendix/simulation}
\input{appendix/strength}
\input{appendix/proofs}
\input{appendix/learned}
\input{appendix/cost}
\input{appendix/relaxation}
\input{appendix/robust}
\input{appendix/experiments}
\end{document}

%% file: macros.tex
\newtheorem{theorem}{Theorem}
\newtheorem{proposition}[theorem]{Proposition}

\newcommand{\E}{\mathbb E}
\newcommand{\R}{\mathbb R}
\newcommand{\clip}{\operatorname{clip}}
\newcommand{\ind}{\mathbf 1}

%% file: sections/introduction.tex
\section{Introduction}
An experiment can establish that a predictor is improvable without establishing how much a particular trained correction improves it. Consider a fitted reference with a systematic causal error. Auxiliary interventions may reveal the \emph{size} of that error, while only a small diagnostic sample reveals the direction in which to repair it. Oracle improvement depends on size alone. The gain of the trained repair also depends on the direction the learner actually chose. More precise auxiliary data can therefore improve one assessment problem while leaving the other limited by the diagnostics.

This distinction is familiar in loss estimation and model assessment \citep{caiguo,bates}. Our question is quantitative: \emph{how much learning accuracy must be relinquished to assess a repair's realized gain at the precision supported by auxiliary data?} Without a learning requirement, returning the reference makes the gain identically zero and perfectly assessable. We therefore optimize the repair and its assessor jointly, subject to a uniform learning-error constraint. The reference is fixed, and the target is the gain of the actual trained repair, averaging over fresh evaluation data only.

We answer this question in a deliberately structured scalar Gaussian experiment. A linear causal model with known intervention directions produces an effect-dependent variance that separates magnitude from sign. The variance identity is an assumption with a concrete structural realization, not a generic property of causal interventions. This setting makes it possible to identify an assessment obstruction independently of high dimensionality, optimization error, or uncertainty about the causal graph.

The main contribution is a \emph{sharp leading-log learning--assessment frontier} in this experiment. Every repair learned at MSE $O(1/k)$ retains a $k^{-2}$ gain-assessment floor. In a magnitude-rich regime, increasing the learning allowance to order $\ell_k/k$ permits assessment at auxiliary precision, with leading threshold constant $U/2$. A single diagnostic-abstention rule attains the exponent uniformly over the unknown effect, noise, contamination, and every reference satisfying the stated moment contract. A joint tail-stability argument further localizes a corrected critical center within an $O(\sqrt{\ell_k\log\ell_k})$ window. This is first-order sharpness on a logarithmic risk scale; multiplicative risk constants and the exact second-order threshold remain unresolved.

Two comparisons explain the result. First, population gain, fitted oracle potential, and realized learned gain have different assessment rates under the same sampling experiment, with their decision classes made explicit below. Second, a parameter-independent Gaussian simulator transfers the joint lower bounds to adaptive sampling at the cost of one extra cheap source observation. This rules out an improvement of the leading frontier within the specified policy class, while allowing finite-sample constant improvements. Controlled simulations and a broader finite-budget stress grid test these mechanisms and document where the asymptotic construction is conservative.

%% file: sections/model.tex
\section{Experiment, reference, and assessment target}\label{sec:model}
Fix known bounds $R,C_0>0$ and $0<L<U$. The parameter $q=(c,h,\tau,\zeta)$ ranges over
\[
 |c|\le C_0,\qquad |h|\le R,\qquad \tau\in[L,U],\qquad
 0\le\zeta\le\Delta\le\tfrac12,
 \quad \theta=c+h,\quad V=\tau+h^2+\zeta.
\]
The contamination \emph{bound} $\Delta$ is known; $\zeta$ is not. An observational sample $O$ contains $m\ge32$ independent pairs
$Z_i\sim N(0,1)$, $Y_i\mid Z_i\sim N(cZ_i,\tau)$.
The experiment provides at most $n$ cheap responses $N(0,V)$ and $k$ diagnostic responses $N(x\theta,V)$, $|x|\le1$. All responses use fresh noise. We allow $n=0$ and impose no restriction $k\le n$. Unless specified otherwise, $k\ge16$.

\paragraph{Two sampling classes.}
In the fixed direct experiment every diagnostic uses $x=1$ and all samples are collected. An adaptive policy may interleave the two channels, choose $x$, randomize, and stop; its cheap and diagnostic caps must hold on every path. The transcript $X$ includes $O$, all recorded assignments and responses, and visible independent random seeds. Procedures may depend on budgets and known bounds, but not on $q$.

\paragraph{Causal interpretation.}
Here $c$ is the observational regression coefficient and $\theta$ the intervention coefficient. A two-dimensional linear Gaussian structural model realizes these laws: the confounding correction $h$ also contributes $h^2$ to interventional variance. Cheap interventions lie in a known direction orthogonal to the outcome-relevant direction. Appendix~\ref{app:causal} gives the structural equations. Known geometry and common Gaussian experimental variance are assumptions of this experiment; they are what make the magnitude channel identifiable up to $\zeta$.

\paragraph{A fitted reference.}
Fix an observational training rule $b=b_m(O)$ satisfying
\begin{equation}\label{eq:reference}
 |b|\le B,\qquad \sup_q\E_q(b-c)^2\le A_{\rm ref}/m.
\end{equation}
The reference is not optimized together with the repair. Clipped full-sample OLS is one example; the results do not require its symmetry. A learner $f$ outputs a coefficient $a_f(X)$. On a fresh assignment uniform on $\{-1,1\}$ with conditional response mean $\theta Z_*$, the realized population squared-loss gain is
\begin{equation}\label{eq:gain}
 G_f(q,X)=\underbrace{(\theta-b)^2}_{A(q,O)}-(\theta-a_f)^2
          =2(a_f-b)(\theta-b)-(a_f-b)^2.
\end{equation}
Only fresh evaluation data are averaged out. Thus $G_f$ remains a function of the actual training transcript and may be negative. The fitted oracle potential $A$ sets the repair to the unknown $\theta$; the population-reference gain $H=h^2$ also replaces $b$ by $c$.

\paragraph{Joint learning and assessment.}
For fixed sampling define
\begin{equation}\label{eq:joint}
 \mathcal F_\eta=\{f:\sup_q\E_q(a_f-\theta)^2\le\eta\},\qquad
 \mathcal R_{\rm fix}(\eta)=\inf_{f\in\mathcal F_\eta}\inf_g\sup_q\E_q(g(X)-G_f(q,X))^2.
\end{equation}
The adaptive risk $\mathcal R_{\rm ad}(\eta)$ additionally optimizes the sampling policy under the same uniform learning constraint. An empty feasible class has infinite risk. We write $\mathcal R_k$ for either risk when a statement applies to both, suppressing the other budgets. Set
\begin{equation}\label{eq:precision}
 E_N=\Delta^2+\frac1{n+k}+\frac1m,\qquad r_N=\min(E_N,k^{-1}),
 \qquad V_*=U+R^2+\tfrac12.
\end{equation}
Constants in $\lesssim$ and $\asymp$ depend only on fixed model/reference bounds unless stated otherwise. The subscript $N$ denotes fixed budgets; $E_k$ denotes the same expression along a sequence indexed by $k$.

\begin{center}
\fbox{\begin{minipage}{.95\linewidth}
\small
\textbf{Which decision problem does each result solve?} All procedures use the observed transcript $X$; the assessor sees the learner's data and visible randomization. The reference $b(O)$ is fixed throughout.
\smallskip

\begin{tabular}{@{}p{.26\linewidth}p{.32\linewidth}p{.35\linewidth}@{}}
\textbf{Assessment target} & \textbf{What is optimized?} & \textbf{Type of conclusion}\\[3pt]
$H=h^2$, $A=(\theta-b)^2$ & Assessor; also policy if adaptive & Target-only minimax rates (Thm.~\ref{thm:targets})\\[3pt]
$G_a=A-(\theta-a)^2$ & Assessor; clipped-mean $a$ and direct design fixed & Fixed-learner rate (Thm.~\ref{thm:targets})\\[3pt]
$G_f=A-(\theta-a_f)^2$ & Learner and assessor; also policy if adaptive; learning MSE $\le\eta$ & Rate floor (Thm.~\ref{thm:quality}); leading log (Thm.~\ref{thm:sharp}); critical window (Thm.~\ref{thm:window})
\end{tabular}
\end{minipage}}
\end{center}

%% file: sections/frontier.tex
\section{A sharp leading-log learning--assessment frontier}\label{sec:frontier}
The first result is a minimax rate statement. The next two identify a leading logarithmic exponent and an asymptotic critical window, respectively; neither claims exact finite-sample minimax risk or attainability at the critical center. The diagnostic mean has learning MSE at most $V_*/k$. At this scale, even joint choice of the repair cannot eliminate a second-order assessment cost.
\begin{theorem}[Assessment floor at diagnostic learning accuracy]\label{thm:quality}
Fix $C_2\ge V_*$. For all $m\ge32$, $n\ge0$, $\Delta\in[0,1/2]$, and sufficiently large $k$ depending only on fixed bounds and $C_2$,
\[
 \mathcal R_{\rm fix}(C_2/k)\asymp\mathcal R_{\rm ad}(C_2/k)
 \asymp r_N+k^{-2}.
\]
The lower bound holds for every feasible learner; the upper bound is attained by one fixed direct construction. No fourth-moment constraint is imposed on the learner class.
\end{theorem}
For clean $m=n=k^3$, the oracle potential has assessment risk $k^{-3}$, while every repair with the stated $1/k$ learning accuracy incurs gain-assessment risk at least of order $k^{-2}$. The next theorem characterizes how this obstruction recedes as the learning allowance increases.

\begin{theorem}[Leading-log frontier]\label{thm:sharp}
Consider budget and contamination-bound sequences with
\[
 E_k=\Delta_k^2+(n_k+k)^{-1}+m_k^{-1},\quad k^2E_k\to0,
 \quad \eta_k\to0,\quad z_k=k\eta_k\to\infty.
\]
Put $\ell_k=\log(1/(k^2E_k))$. For both sampling classes,
\begin{equation}\label{eq:frontier}
 \frac{\log(1/(k^2\mathcal R_k(\eta_k)))}
      {\min\{\ell_k,2z_k/U\}}\longrightarrow1.
\end{equation}
Algorithm~\ref{alg:gate}, with its declared fast slack schedule, attains this leading logarithm. If $\ell_k=o(k)$, every fixed $C>U/2$ gives
$\mathcal R_k(C\ell_k/k)\asymp E_k$, while $\mathcal R_k(\eta_k)=O(E_k)$ requires
$\liminf k\eta_k/\ell_k\ge U/2$.
\end{theorem}
The exponent grows linearly with learning allowance until it reaches the precision supplied by magnitude information (Figure~\ref{fig:frontier}). Equation~\eqref{eq:frontier} is a relative first-order statement about logarithms; it does not assert multiplicative equivalence of the underlying risks. The coefficient $U/2$ is sharp away from equality. At equality, the signal's own contribution to experimental variance matters.

\begin{figure}[t]
\centering
\includegraphics[width=.88\linewidth]{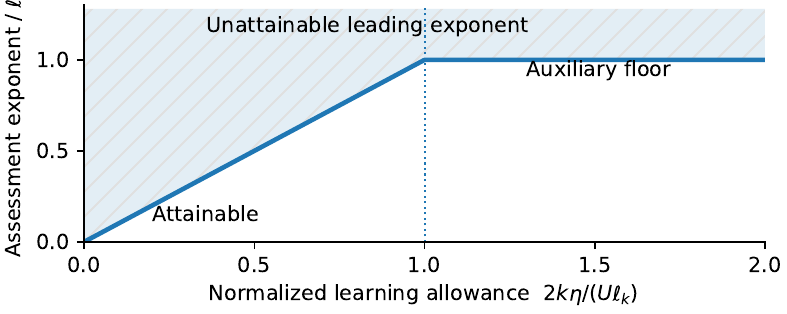}
\caption{Leading-log frontier, normalized by $\ell_k$. The admissible assessment exponent saturates at auxiliary precision $E_k$. The curve depicts Theorem~\ref{thm:sharp}, not an exact finite-sample risk. Theorem~\ref{thm:window} shifts the first-order center $U\ell_k/2$ to $z_c=U\ell_k/(2-\ell_k/k)$. That correction can exceed the stated allowance window even when $\ell_k=o(k)$; it is not discarded in the refined result.}
\label{fig:frontier}
\end{figure}

\begin{theorem}[Asymptotic critical allowance window]\label{thm:window}
Suppose $k^2E_k\to0$ and $\ell_k=o(k)$, and define
\[
 z_c=\frac{U\ell_k}{2-\ell_k/k},\qquad w_k=\sqrt{\ell_k\log\ell_k}.
\]
Reaching $\mathcal R_k(\eta_k)=O(E_k)$ requires $k\eta_k\ge z_c-O(w_k)$.
Conversely, for a sufficiently large fixed $C_{\rm win}$, Algorithm~\ref{alg:gate} attains $\mathcal R_k(\eta_k)\asymp E_k$ at $k\eta_k=z_c+C_{\rm win}w_k$.
If additionally $\ell_k^3/(k^2\log\ell_k)\to\infty$, then the uncorrected allowance $\eta_k=U\ell_k/(2k)$ yields $\mathcal R_k(\eta_k)/E_k\to\infty$.
\end{theorem}
This localizes the sufficient and necessary allowances; it does not identify an exact second-order coefficient or settle attainability at $z_c$. Indeed,
$z_c-U\ell_k/2=U\ell_k^2/(4k-2\ell_k)$, which can dominate $w_k$ within $\ell_k=o(k)$.
Under $k^2E_k\to0$, the same sign lower bound implies that consistent learning together with assessment at full auxiliary precision is possible exactly when $\ell_k=o(k)$ (Appendix~\ref{app:relaxation}).

%% file: sections/construction.tex
\section{Diagnostic abstention: construction and proof mechanism}\label{sec:construction}
Let $M$ be the direct diagnostic mean. Estimate $\tau$ by the full-sample OLS residual sum of squares divided by $m-1$, using the \emph{raw}, un-clipped fit. Pool cheap energy and centered diagnostic energy:
\begin{equation}\label{eq:pooled}
 \widehat V=\frac{\sum_{i=1}^n C_i^2+\sum_{j=1}^k(D_j-M)^2}{\nu},
 \qquad \nu=n+k-1,\qquad Q=\clip(\widehat V-\widehat\tau,0,R^2).
\end{equation}
Gaussian orthogonal decomposition gives $\nu\widehat V/V\sim\chi_\nu^2$ and independence of $M$ from $(Q,b)$. In particular,
$\E(Q-h^2)^2\le C_Q E_N$ with $C_Q=\max(1,4V_*^2)$. Diagnostic residuals remain informative when there are no cheap observations.

\begin{algorithm}[htbp]
\centering
\begin{minipage}{.97\linewidth}
\hrule\medskip
\textbf{Diagnostic-abstention repair and gain assessment}
\smallskip

\textbf{Inputs:} $(b,M,Q)$; diagnostic count $k$; learning allowance $\eta$; known variance bound $V_*$.\par
\textbf{Slack:} $z=k\eta$; set $\epsilon=\min\{1/4,\sqrt{\log z/z}\}$ for $z>e$, and $\epsilon=1/4$ otherwise.\par
\textbf{Magnitude gate:} Set $t^2=(1-\epsilon)\eta$. If $Q<t^2$, return $(a,g)=(b,0)$.\par
\textbf{Diagnostic gate:} Otherwise calculate
\[
 d(Q)=\sqrt{\frac{8V_*}{k}\log\frac{8Q}{\epsilon\eta}},\qquad
 T(Q)=\max\{0,\sqrt Q-d(Q)\}.
\]
If $|M-b|<T(Q)$, return $(a,g)=(b,0)$.\par
\textbf{Repair and report:} Return $a=b+\operatorname{sign}(M-b)\sqrt Q$, $g=Q$.
\medskip\hrule
\end{minipage}
\caption{The same declared rule attains the leading exponent and the critical-window upper bound, with eventual uniform learning feasibility in their respective regimes. No unknown effect, variance, or sign is an algorithm input.}\label{alg:gate}
\end{algorithm}

\paragraph{Why gate on diagnostic evidence?}
With exact magnitude $r=|h|$ and an exact reference, a magnitude-only rule accepts either sign whenever $r^2\ge t^2$. A wrong repair then occurs when diagnostic noise crosses approximately $-r$. Requiring $|M-b|\gtrsim r$ moves an \emph{accepted} wrong-sign crossing to approximately $-2r$. Squaring the Gaussian tail distance changes the exponent from $kr^2/(2V)$ to $2kr^2/V$. Correct-sign rejection is controlled by $d(Q)$; its cost is absorbed by the learning slack. This is the mechanism behind the factor $2/U$.

\paragraph{Upper bound with unknown nuisance parameters.}
Writing $e=c-b$, $\delta=Q-h^2$ and $W$ for a wrong sign, the exact gain identity implies the useful envelope
\[
 (g-G)^2\le12\delta^2+48I Qh^2\ind_W+12R^2e^2,
\]
where $I$ denotes acceptance. The first and last terms cost $O(E_k)$. A joint square-root-tail inequality controls the middle term without paying a nuisance multiplier that diverges as $\epsilon$ shrinks. It yields, for every fixed $a_0<2/U$,
\begin{equation}\label{eq:upper-main}
 \sup_q\E_q(g-G)^2\le C_1E_k+C_2k^{-2}e^{-a_0k\eta_k},
 \qquad \sup_q\E_q(a-\theta)^2\le\eta_k
\end{equation}
eventually for Algorithm~\ref{alg:gate}. Appendix~\ref{app:window} proves the finite inequality; Appendix~\ref{app:sharp} shows why the single fast schedule suffices for the whole leading-log regime.

\paragraph{Lower bound that survives learner choice.}
At a shared transcript, changing $\theta$ changes the gain by $2(a_f-b)\,d\theta$. Uniform learning accuracy forces $a_f-b$ away from zero on a positive-probability event at suitable nonzero effects. Feasible pairs that conceal changes in experimental variance, observational calibration, or contamination yield the $E_k$ obstruction. For the exponential obstruction, take $c=0$, $\tau=U$, $\zeta=0$ and $h=\pm r$, with $r^2=(1+\varepsilon)\eta_k$. Magnitude channels have identical laws, while under the positive sign the log likelihood ratio is
\[
 \log\Lambda=2s+2\sqrt s\,Z,\qquad s=\frac{kr^2}{U+r^2},\qquad Z\sim N(0,1).
\]
The learning constraint forces a nontrivial repair despite the residual sign ambiguity. A pointwise posterior square identity gives the lower exponent $2/U$ without assuming fourth moments for arbitrary feasible learners. Keeping the $r^2$ contribution in the denominator gives the refined exponent $\psi_k(z)=2z/(U+z/k)$; solving $\psi_k(z)=\ell_k$ produces $z_c$.

\paragraph{Why adaptive diagnostics do not escape the frontier.}
An orthogonal Gaussian reservoir transforms a fixed experiment with counts $(m,n+1,k)$ into the exact transcript of any policy capped at $(n,k)$. It leaves $O$ unchanged and preserves the simulated learner's gain and learning error. Therefore
\begin{equation}\label{eq:simulation-sandwich}
 \mathcal R_{\rm fix}(m,n+1,k;\eta)
 \le\mathcal R_{\rm ad}(m,n,k;\eta)
 \le\mathcal R_{\rm fix}(m,n,k;\eta).
\end{equation}
The one-sample change perturbs $\ell_k$ by $o(1)$, preserving both leading results. Appendix~\ref{app:simulation} proves the simulator and explains why removing that extra sample does not generally give an exact experiment comparison.

%% file: sections/targets.tex
\section{Changing the target changes the information limit}\label{sec:targets}
The frontier concerns actual trained gain. Replacing it by a population or oracle target changes the problem even though every sample is unchanged. For a prescribed target $T$, write $\mathcal R(T)=\inf_g\sup_q\E_q(g-T)^2$, with an additional policy infimum in the adaptive case. There is no learner infimum unless explicitly stated.
\begin{theorem}[Target-dependent assessment rates]\label{thm:targets}
Let $K=\min(k,m)$. For population gain $H=h^2$ and fitted oracle potential $A=(\theta-b)^2$, the fixed and adaptive minimax MSEs have orders
\[
 \mathcal R(H)\asymp\min(E_N,K^{-1}),\qquad
 \mathcal R(A)\asymp\min(E_N,k^{-1}).
\]
For the fixed direct learner $a=\clip(M,-C_0-R,C_0+R)$,
\[
 \mathcal R_{\rm fix}(G_a)\asymp\min(E_N,k^{-1})+k^{-2}.
\]
The first two statements extend to $k=0$ with $1/0=\infty$ and $n+k\ge16$.
\end{theorem}
Table~\ref{tab:targets} exhibits both strict separations; it is not a comparison of three identical learner optimizations. When diagnostics are plentiful, the fitted target avoids a floor incurred by the population-reference target. When magnitude information is plentiful, replacing a learned repair by its oracle removes an assessment obstruction instead.
\begin{table}[t]
\centering
\caption{Same data, explicitly different decision classes. Clean data ($\Delta=0$); entries are assessment-MSE orders. The first two columns optimize only assessment (and policy if adaptive); the last fixes both the direct design and the clipped-mean learner. Joint learner selection is the separate problem in Theorem~\ref{thm:quality}.}
\label{tab:targets}
\begin{tabular}{lccc}
\toprule
Budgets $(m,n,k)$ & Population $H$ & Fitted oracle $A$ & Learned gain $G_a$\\
Decision class & Target only & Target only & Fixed clipped mean\\
\midrule
$(t,t^2,t^2)$ & $t^{-1}$ & $t^{-2}$ & $t^{-2}$\\
$(t^3,t^3,t)$ & $t^{-3}$ & $t^{-3}$ & $t^{-2}$\\
\bottomrule
\end{tabular}
\end{table}

A simple transfer bound clarifies where the targets agree:
\[
 \left|\sqrt{\mathcal R_{\rm fix}(G_a)}-\sqrt{\mathcal R_{\rm fix}(A)}\right|
 \le\frac{\sqrt3V_*}{k}.
\]
This follows from $A-G_a=(\theta-a)^2$ and the clipped mean's fourth error moment. It gives equivalent risks when $k^2r_N\to\infty$, but cannot establish the additional $k^{-2}$ lower bound. That bound requires a sign pair. Proofs, extensions to $|h|^p$ with $1<p<2$, and acquisition-cost consequences appear in Appendices~\ref{app:strength}--\ref{app:cost}.

%% file: sections/experiments.tex
\section{Numerical evidence}\label{sec:experiments}
We retain seven mechanism/baseline suites (241 independently seeded cells; 4.34 million draws) and add 1,920 finite-budget configurations and 42 variance-bound configurations (23.54 million draws). Methods within a cell share exact Gaussian sufficient-statistic draws; every report is scored against its own realized reference and learner. Conditional Gaussian integration resolves rare diagnostic errors. These are finite-grid procedure risks, not estimates of minimax suprema or uniform feasibility. Code, grids, paired Monte Carlo SEs, log tails and source hashes accompany the paper; Appendix~\ref{app:experiments} gives the full protocols.

\paragraph{The target changes the observed risk scale.}
For clean $m=n=k^3$, $c=0$, $h=.3$ and $\tau=2$, the identical variance report $Q$ assesses $A$ and the clipped-mean repair's $G_a$. Across $k=16$ to $4096$, $k^3$-scaled oracle MSE ranges from 15.38 to 17.57, while $k^2$-scaled learned-gain MSE ranges from 13.08 to 14.21 (Figure~\ref{fig:mechanisms}). Subtracting $\widehat V/k$ improves the learned-gain constant without removing this scale difference. The full grid also includes local effects and a Stein assessor.

\begin{figure}[t]
\centering
\includegraphics[width=.49\linewidth]{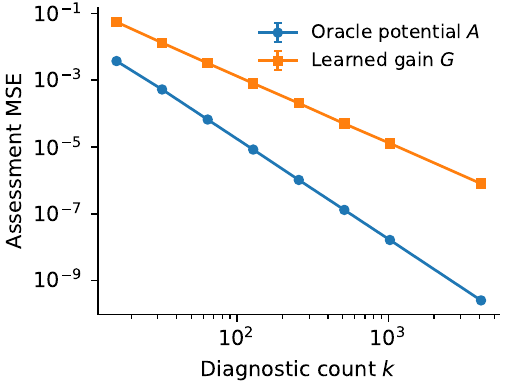}\hfill
\includegraphics[width=.49\linewidth]{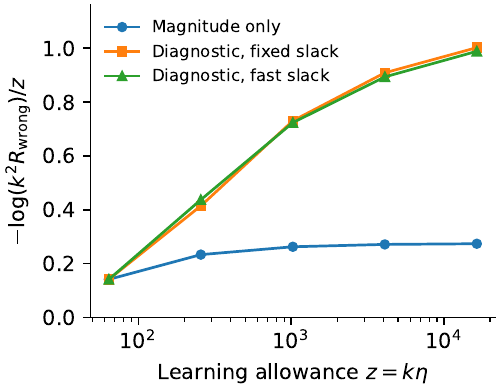}
\caption{Controlled mechanism checks. \textbf{Left:} identical report $Q$, different targets, at $m=n=k^3$, $h=.3$; bars are $\pm2$ Monte Carlo SEs. \textbf{Right:} integrated wrong-sign MSE at $k=2^{24}$, $h^2=1.1\eta>0$, and $z=k\eta$. The fixed-slack pair isolates the diagnostic gate; the fast schedule is Algorithm~\ref{alg:gate}. Tail suppression need not materially reduce total MSE.}
\label{fig:mechanisms}
\end{figure}

\paragraph{Tail exponents and total MSE answer different questions.}
At $k=2^{24}$ and $h^2=1.1\eta$, increasing $z$ from 64 to 16384 raises the fixed-gate normalized wrong-sign exponent from .141 to 1.003, versus .141 to .274 for magnitude-only gating; fast slack reaches .989. Yet total MSE at large $z$ stays near $8E_k$ because calibration dominates. Table~\ref{tab:suite-summary} summarizes \emph{all seven} suites using high-quantile and worst observed MSE ratios. The oracle count rule and joint strength rules have unfavorable cells, and Stein loses to the plug-in at an active clipping boundary. Rate attainment and unbiasedness do not imply finite-sample dominance.
\input{appendix/table_suite_summary}

\paragraph{Moderate budgets, nuisance variation, and reference rules.}
The new grid crosses $k\in\{64,256,1024,4096\}$, $c=\pm.6$, $\tau\in\{1.5,2.5\}$, six signed/zero/local effects, two bounded accurate references, five auxiliary regimes, and allowances at one and four times the clean-path center. The regimes vary $m,n,\Delta,\zeta$, including $n=0$ and contamination .1. Both gates use identical fast slack. Table~\ref{tab:finite-summary} reports joint learning and assessment performance at the larger allowance; all cells, both allowances, and the unprojected fallback are retained in the results.
\input{appendix/table_finite_summary}

On the clean magnitude-rich grid, diagnostic gating changes total MSE by less than $4\times10^{-24}$ relative to matched magnitude gating: sign errors are already negligible wherever its extra threshold acts. Under the stress regimes, gate learning MSE can exceed the allowance; projected fallback has a grid-maximum learning ratio .291 at the displayed allowance, which also meets its theorem's condition $k\eta\ge4V_*$. This numerical maximum is not a proof of uniform feasibility. These comparisons support the mechanism and its limits rather than a claim that the frontier-attaining rule is the best finite-budget pipeline.

\paragraph{Sensitivity to a conservative variance bound.}
With $k=16384$, $h^2=1.1\eta$ and $z=256$, using $V_*=4$ gives log wrong-sign MSE $-130.3$; using $4V_*$ gives $-78.6$ and no positive diagnostic threshold in the nuisance draws. Both total MSEs are $3.75\times10^{-12}$ to three significant digits. The 42-cell ablation varies the bound by factors $1,2,4,16$ without changing data or allowance (Appendix~\ref{app:finite}). A larger valid bound delays tail suppression; improving that tail is useful for total MSE only before calibration dominates.

%% file: appendix/table_suite_summary.tex
\begin{table}[t]
\centering
\small
\caption{All seven original suites: cellwise total-MSE ratios (first report divided by second). P90 and maximum summarize the stated grid, not a minimax supremum. The Cells column gives positive-denominator cells / all cells; omitted cells have both observed MSEs zero. $\dagger$ Changes the learner and its target; every other row compares assessors of the same target.}\label{tab:suite-summary}
\setlength{\tabcolsep}{4pt}
\begin{tabular}{llrrr}
\toprule
Suite & Comparison & Cells & P90 & Max.\\
\midrule
Target scale & Corrected / $Q$ & 16/16 & $0.677$ & $0.681$\\
Diagnostic gate & Diagnostic / magnitude$^\dagger$ & 90/120 & $1$ & $1$\\
Robustness & Projected / raw fallback & 16/24 & $1$ & $1$\\
Residual strength & Joint / mean & 27/27 & $1.14$ & $1.9$\\
Fitted oracle & Selected / clipped unbiased & 24/24 & $3.67$ & $39.4$\\
Pooled budgets & Joint / mean & 18/18 & $1.59$ & $1.83$\\
Clipping / Stein & Stein / plug-in & 12/12 & $1.3$ & $5.01$\\
\bottomrule
\end{tabular}
\end{table}

%% file: appendix/table_finite_summary.tex
\begin{table}[t]
\centering
\small
\caption{Finite-budget pipelines at $k\eta=4z_c^{\rm clean}$: 192 clean, magnitude-rich cells and 768 stress cells (balanced, no-cheap, contamination-bound-only, and contaminated). $L_{\max}$ is the grid maximum of learning MSE divided by $\eta$; $R_{90}$ is P90 of assessment MSE divided by $E_N+k^{-2}$. Each pipeline assesses its own realized repair; only the first two rows share a learner. Values are estimates, not uniform feasibility or dominance certificates.}\label{tab:finite-summary}
\setlength{\tabcolsep}{4pt}
\begin{tabular}{lrrrr}
\toprule
Learner + report & \multicolumn{2}{c}{Clean, rich} & \multicolumn{2}{c}{Stress}\\
 & $L_{\max}$ & $R_{90}$ & $L_{\max}$ & $R_{90}$\\
\midrule
Clipped mean + Stein & $0.174$ & $6.12\times10^{3}$ & $0.178$ & $8.69$\\
Clipped mean + $Q-\widehat V/k$ & $0.174$ & $17.6$ & $0.178$ & $8.27$\\
Magnitude gate + $IQ$ & $0.5$ & $7.6$ & $10.7$ & $16.5$\\
Diagnostic gate + $IQ$ & $0.5$ & $7.6$ & $3.27$ & $14.5$\\
Fallback + projected report & $0.122$ & $21.2$ & $0.291$ & $8.37$\\
\bottomrule
\end{tabular}
\end{table}

%% file: sections/discussion.tex
\section{Scope and practical interpretation}\label{sec:scope}
\paragraph{What survives a different magnitude channel?}
The $k^{-2}$ obstruction uses a sign pair: the auxiliary laws agree at $h=\pm r$, while $k$ noisy diagnostics must distinguish them. It does not use the quadratic variance identity. Proposition~\ref{prop:sign-scope} shows that this lower bound survives arbitrary sign-invariant auxiliary observations and unequal channel variances, provided diagnostic variance stays bounded away from zero and infinity and the reference is sufficiently accurate. For a diagnostic variance tending to $v_0$, the same pair yields the one-sided exponent bound $2/v_0$. These are lower bounds, not a matching generalized frontier.

The attainable $E_k$ frontier needs more. If cheap variance is an \emph{unknown} monotone function of $|h|$, its level need not calibrate $h^2$; the bound $\E(Q-h^2)^2=O(E_k)$ is then unavailable. With unequal experimental variances, the pooled chi-square identity and parameter-independent reservoir no longer apply as written. Unknown directions can also destroy the sign-invariant channel. Thus neither the $2/U$ equality, the critical center, nor the adaptive sandwich is asserted outside the stated experiment. Appendix~\ref{app:scope} separates these dependencies.

\paragraph{A loose variance bound delays the gate.}
Replacing $V_*$ by a larger \emph{fixed} valid bound preserves the leading exponent, but enlarges the diagnostic discount $d(Q)$ and the sufficient-window constants. The diagnostic threshold is positive only if
$kQ>8V_*\log(8Q/(\epsilon\eta))$; otherwise Algorithm~\ref{alg:gate} is exactly the magnitude-only rule on that draw. Near the local threshold, the discount relative to magnitude is of order $\sqrt{V_*\log z/z}$. Thus the asymptotic constant $U/2$ does not imply an effective gate at moderate budgets. The finite-budget ablation measures this delay directly. Bounds growing with $k$ are outside the theorem's fixed-bound assumptions.

\paragraph{Learning accuracy is not a repair-selection guarantee.}
For each parameter, $\E_qG_f=\E_qA-\E_q(a_f-\theta)^2\ge\E_qA-\eta$.
The constraint therefore controls average excess loss relative to the oracle. It guarantees neither positive realized gain nor a bound on $\Pr_q(G_f<0)$; even positive average gain needs $\E_qA>\eta$. Abstention returns the fitted reference, with exactly zero realized gain, while spending learning allowance on the missed improvement. A deployment objective requiring positive gain, a lower confidence bound, or intervention costs defines a different feasible class and needs a separate analysis. The present frontier quantifies the learning cost of assessability under parameter MSE.

%% file: sections/related.tex
\section{Relation to prior work}\label{sec:related}
\paragraph{Assessing a learned rule.}
Loss estimation already distinguishes learning a parameter from assessing an estimator's accuracy. \citet{caiguo} derive lower bounds for the loss of rate-optimal estimators under constraints on their parameter-estimation performance; \citet{fourdrinier} study improved loss estimation for shrinkage rules. \citet{bates} distinguish a fitted model's prediction error from training-averaged error. \citet{gastpar} show accuracy--tightness trade-offs for algorithm-dependent generalization bounds in overparameterized settings. Our target is the realized gain of a jointly selected, quality-constrained repair in a two-channel causal experiment. When oracle potential is sufficiently easy to estimate, a root-risk transfer also connects our gain frontier to learned-loss assessment (Appendix~\ref{app:relaxation}). The contribution is the joint realized-gain criterion together with an experiment-specific leading exponent, an attaining rule, and a corrected allowance window; the distinction between estimation and loss assessment itself is established precedent.

\paragraph{Constrained risk, rejection, and nonlinear functionals.}
Constrained-risk inequalities provide a general route from good performance at one parameter to a lower bound elsewhere \citep{duchiruan}; optimal reject-option rules threshold conditional risk, with possible boundary randomization \citep{franc}. Here the learning constraint acts on the repair, while the competing risk assesses that repair's random gain. Joint control of magnitude error, reference error and accepted wrong-sign tails yields the uniform constant $2/U$ and the critical window. The supporting strength rates follow a modulus calculation \citep{donoholu}. Related functional-estimation problems include nonsmooth functionals \citep{cailow}, sparse linear and quadratic functionals \citep{collier}, and forms estimated from indirect observations \citep{juditsky}.

\paragraph{Intervention design and target-dependent information acquisition.}
\citet{zhangintervention} actively select interventions to move a causal model's post-intervention mean toward a desired target. \citet{cobed} use Bayesian experimental design for information-theoretic contextual optimization. \citet{arumugamvoi} jointly design learning targets and information acquisition through rate-distortion and information-directed sampling. These works optimize what to learn or which action to take for a downstream objective. We instead characterize minimax assessment of the \emph{realized gain of the chosen learner}, subject to a uniform parameter-MSE constraint. Our adaptive result compares statistical experiments for that random target; it neither proposes a generally optimal acquisition function nor rules out adaptive gains in other causal designs.

\paragraph{Confounding and Gaussian experiment comparison.}
Structural sensitivity analysis for multiple treatments \citep{zheng} and confounded-bandit decisions \citep{tennenholtz} address richer forms of causal uncertainty. Here known geometry isolates the magnitude--direction information split. The exact one-extra-sample simulator preserves the original learner's target and its second-moment quality constraint. Its same-budget approximation connects to sample amplification \citep{axelrod}, but a bounded-loss approximation alone would not transfer our unrestricted learner class.

%% file: sections/statements.tex
\section*{Reproducibility statement}
All mathematical statements are accompanied by proofs in the appendices. The supplement includes the complete experiment implementation, fixed parameter grids, independent-cell seeds, paired Monte Carlo standard errors, machine-readable results, numerical tests, and source/result SHA-256 hashes. The reproduction command reruns all nine numerical suites, regenerates every figure and numerical table, and checks the released outputs. Each suite samples the exact Gaussian sufficient-statistic distribution, subject to floating-point arithmetic. The experiments use only synthetic data and a CPU. Log-domain conditional integration is used for rare diagnostic errors. Appendix~\ref{app:experiments} identifies the finite-grid scope and numerical checks.

\section*{AI use statement}
Generative AI tools assisted with research framing, formulation and critical review of mathematical claims, proof writing and refinement, related-work search, experimental design, implementation, synthetic-data simulation code, interpretation of numerical results, and manuscript editing. This assistance included the fast-slack construction, proof review, the sign-invariant lower-bound extension, citation verification, and the executable synthetic experiments and their analysis. The synthetic observations are generated by specified Gaussian and chi-square random-number routines, not by language-model sampling. Verification artifacts include analytic derivations, numerical identity and distributional checks, deterministic replay, and compilation checks; these do not constitute formal proof verification. Responsibility for the submitted claims, code, citations, and disclosure rests with the authors.

\section*{Ethics statement}
The work uses a synthetic structural model and no human-subject or personal data. The reported guarantees concern statistical assessment MSE under explicit model assumptions. They do not certify that a realized intervention is beneficial or safe, and gain estimates should not be used as deployment authorization without an appropriate decision criterion and model validation.

%% file: appendix/roadmap.tex
\section{Proof and implementation guide}\label{app:guide}
The main frontier results share a sign obstruction but make different precision claims. Appendix~\ref{app:scope} proves which part survives a change to the magnitude channel and identifies what remains open at the critical center. Appendix~\ref{app:quality} proves the fixed-quality floor with lower bounds that apply to every feasible learner. Appendix~\ref{app:sharp} proves the leading exponent and the unified fast-slack construction. Appendix~\ref{app:window} supplies its finite joint-tail inequality and the critical-window bounds. Appendix~\ref{app:simulation} transfers these results to adaptive policies while preserving the original learner and reference on the simulated transcript. Appendices~\ref{app:strength} and~\ref{app:proofs} establish the supporting functional and oracle-target rates; Appendix~\ref{app:learned} handles the fixed clipped-mean learner. The magnitude-only bounds in Appendix~\ref{app:relaxation} supply the allowance-dependent lower bound and nuisance lower bound used in Appendix~\ref{app:sharp}, as well as the consistency boundary. Appendix~\ref{app:robust} gives the separate contamination fallback. Appendix~\ref{app:experiments} records all empirical protocols, including the finite-budget stress and variance-bound suites.

\begin{center}
\begin{tabular}{lp{.70\linewidth}}
\toprule
Symbol & Meaning\\
\midrule
$c,h,\theta$ & Observational coefficient, confounding correction, and causal coefficient $\theta=c+h$.\\
$\tau,\zeta,V$ & Observational residual variance, experimental variance increment, and $V=\tau+h^2+\zeta$.\\
$m,n,k$ & Observational count, cheap count/cap, and diagnostic count/cap.\\
$b,a,g$ & Fixed observational reference, trained repair, and reported gain.\\
$H,A,G$ & Population-reference gain $h^2$, fitted oracle potential $(\theta-b)^2$, and actual trained gain $A-(\theta-a)^2$.\\
$E_N,r_N$ & $\Delta^2+(n+k)^{-1}+m^{-1}$ and $\min(E_N,k^{-1})$.\\
$E_k,\ell_k,z_k$ & Sequence version of $E_N$, $\log(1/(k^2E_k))$, and $k\eta_k$.\\
$Q,\delta,e$ & Clipped magnitude-squared estimate, $Q-h^2$, and $c-b$.\\
\bottomrule
\end{tabular}
\end{center}
The notation $E$ below abbreviates $E_N$ when budgets are fixed. The un-clipped difference $\widehat V-\widehat\tau$ is stated explicitly when it is used in a minimum-distance objective; $Q$ in gate proofs is always clipped to $[0,R^2]$. All empirical results refer to the included implementation and result files, not to an external repository.

%% file: appendix/causal.tex
\section{Causal realization and sufficient statistics}\label{app:causal}
\subsection{A structural model with magnitude and directional channels}
For known orthonormal vectors $u,v\in\R^2$, choose known $\rho\in(0,1)$ with $R^2(1-\rho)/\rho<L$. Let independent primitive noises have laws
\[
 U_0\sim N(0,1),\quad \epsilon_W\sim N(0,I-\rho vv^\top),\quad
 \epsilon_Y\sim N\!\left(0,\tau-h^2\frac{1-\rho}{\rho}\right),
\]
and define
\[
 W=\sqrt\rho\,vU_0+\epsilon_W,\qquad
 Y=(c+h)v^\top W-hU_0/\sqrt\rho+\epsilon_Y.
\]
The noise covariance and variance are positive under the stated bounds. Since $W\sim N(0,I)$ and $\E(U_0\mid W)=\sqrt\rho\,v^\top W$, Gaussian conditioning gives
$Y\mid W\sim N(cv^\top W,\tau)$. Thus $Z=v^\top W$ yields the observational experiment; the orthogonal component carries no parameter information. Under the full-vector intervention $\operatorname{do}(W=x)$, the response has mean $\theta v^\top x$ and variance $\tau+h^2$. Independent measurement noise with variance $\zeta$ gives $V$. Interventions along $u$ have zero mean; assignments $xv$ with $|x|\le1$ give the diagnostic channel.

This construction identifies what is being assumed. The geometry is known, experimental measurement noise has the same increment across channels, and the scalar effect contributes quadratically to variance. Without such restrictions, an observed variance difference need not reveal a squared causal correction.

\subsection{Pooled variance for every diagnostic budget}\label{app:pooled}
Let $N=n+k\ge16$. For $k\ge1$ the direct diagnostic mean $M$ and its centered residuals are independent. With $\nu=N-1$, equation~\eqref{eq:pooled} obeys
\[
 \nu\widehat V/V\sim\chi^2_\nu,
 \qquad \widehat V\ \text{independent of }(O,M).
\]
For $k=0$, set $\nu=n$ and use the cheap sum alone. In either case $\nu\ge15N/16$. A single diagnostic contributes no residual degree of freedom. Using the raw OLS fit for the observational residual sum gives
\[
 (m-1)\widehat\tau/\tau\sim\chi^2_{m-1},\qquad
 \E(\widehat V-\widehat\tau-h^2)^2
 =\zeta^2+2V^2/\nu+2\tau^2/(m-1)\le C_QE_N.
\]
Clipping to $[0,R^2]$ decreases this error. Moreover $M$ is independent of $(Q,b)$ for any observational reference satisfying~\eqref{eq:reference}; the gate proofs do not assume independence of $Q$ and $b$.

For the OLS reference used in experiments, one can sample the whole sufficient-statistic law without constructing raw arrays:
\[
 S_Z\sim\chi^2_m,\quad c_{\rm raw}=c+\sqrt{\tau/S_Z}\,Z_c,\quad
 \widehat\tau=\tau S_\tau/(m-1),\quad S_\tau\sim\chi^2_{m-1},
\]
with independent $S_Z,Z_c,S_\tau$. Independently sample cheap energy $V\chi^2_n$, diagnostic residual energy $V\chi^2_{k-1}$, and $M=\theta+\sqrt{V/k}\,Z_M$. These are exact distributional identities, not large-sample approximations. Appendix~\ref{app:experiments} discusses floating-point precision for very large degrees of freedom.

%% file: appendix/scope.tex
\section{Which assumptions support which conclusions?}\label{app:scope}
The attainable frontier is specific to the experiment in Section~\ref{sec:model}.
The following lower-bound extension separates the sign obstruction from the
quadratic variance identity. It is useful even when auxiliary data do not
identify magnitude at all.

\begin{proposition}[Sign-invariant auxiliary information]\label{prop:sign-scope}
Consider an experiment containing, for every sufficiently small $r>0$, two
parameters with $c=0$ and $h=\theta=\pm r$. Let the auxiliary transcript $Y$
have the same law at both signs, and independently observe $k$ direct
diagnostics $D_j\sim N(\pm r,v(r))$, where
$0<v_-\le v(r)\le v_+<\infty$. The fixed reference is a function of $Y$
with $\E b^2\le A_{\rm ref}/m$. Learners may use the full transcript and
visible independent randomization. Let $\mathcal R(\eta)$ denote their joint
gain-assessment risk under uniform learning MSE at most $\eta$ on the
parameter class, for the same gain $G_f=(\theta-b)^2-(\theta-a_f)^2$.
Then:
\begin{enumerate}
\item For every fixed $C_2>0$, if $m\ge k$, there is a constant $c_*>0$
such that $\mathcal R(C_2/k)\ge c_*k^{-2}$ for all sufficiently large $k$.
No information about $r$ is required from the auxiliary channel.
\item If $v(r)\to v_0>0$ as $r\downarrow0$, $\eta_k\to0$,
$k\eta_k\to\infty$, and $A_{\rm ref}/(m_k\eta_k)\to0$, then
\[
 \limsup_k\frac{\log(1/[k^2\mathcal R(\eta_k)])}{k\eta_k}
 \le \frac{2}{v_0}.
\]
This is only a lower bound on risk; auxiliary identifiability and an
attaining procedure are not asserted.
\end{enumerate}
The same lower bounds apply to adaptive diagnostic assignments $|x|\le1$
with at most $k$ responses $N(x\theta,v(r))$, provided all other available
information is sign-invariant and independent before sampling.
\end{proposition}

\begin{proof}
It suffices to restrict to the indicated pair and reveal its common $r$
and $v(r)$. This can only help assessment. All remaining sign information
is in $M=k^{-1}\sum_jD_j$. Under the equal sign prior the pointwise square
identity gives the lower bound
\[
 \E\left[4r^2(a_f-b)^2\operatorname{sech}^2(krM/v(r))\right].
\]
The identity is pointwise and remains valid if a competing assessment risk
is infinite. It requires only a second moment of $a_f-b$, not a fourth
moment of the learner.

For the first assertion choose fixed
$K^2\ge64\max\{C_2,A_{\rm ref},1\}$ and $r=K/\sqrt k$.
Under the positive sign, the failures of
\[
 |a_f-r|\le r/2,\qquad |b|\le r/4,\qquad
 |\sqrt k(M-r)|\le4\sqrt{v_+}
\]
have probabilities at most $1/16$, $1/4$, and $1/16$.
On their intersection $a_f-b\ge r/4$ and
$|krM/v(r)|\le K(K+4\sqrt{v_+})/v_-$.
The intersection has mixture probability at least $1/4$.
The displayed posterior bound is therefore at least a fixed positive
constant times $r^4=K^4/k^2$, uniformly over feasible learners.

For the second assertion fix $\delta>0$ and set
$r^2=(1+\delta)\eta_k$,
$s=kr^2/v(r)$. Under the positive sign the likelihood ratio obeys
$\log\Lambda=2s+2\sqrt s\,Z$, $Z\sim N(0,1)$.
Choose fixed $d\in((1+\delta)^{-1/2},1)$ and $u=(1-d)/2$.
The learning constraint, the reference moment bound and a sufficiently
large fixed normal cutoff $H$ ensure that
\[
 |a_f-r|\le dr,\quad |b|\le ur,\quad
 \log\Lambda\le2s+2\sqrt s H
\]
has probability bounded below by a positive constant under the positive
sign. The same posterior identity, written under that sign, is
$\E_+[8r^2(a_f-b)^2/(1+\Lambda)]$.
Consequently
$k^2\mathcal R(\eta_k)\ge c_\delta(k\eta_k)^2
e^{-2s-2\sqrt s H}$.
Divide negative logarithms by $k\eta_k$, use $v(r)\to v_0$, and then
let the fixed $\delta$ decrease to zero.

For the adaptive extension, on each revealed pair the variance $v(r)$ is
known and independent of the sign. An unused direct observation $D_j$
and fresh $Z_j\sim N(0,1)$ generate a requested response as
$xD_j+\sqrt{1-x^2}\sqrt{v(r)}Z_j$.
The independent sign-invariant information can be supplied in advance.
This simulates any capped adaptive learner and assessor on the pair and
preserves their learning error and realized target. The direct-source
argument therefore applies. This pairwise lower-bound comparison does
not provide a parameter-independent simulator over the original unknown-
variance parameter class.
\end{proof}

\paragraph{Dependencies of the matching upper bound.}
The proof of Algorithm~\ref{alg:gate} uses three properties together:
(i) a magnitude statistic with $\E(Q-h^2)^2\le CE_k$ and a reference with
$\E(b-c)^2\le A_{\rm ref}/m$; (ii) independent diagnostic mean noise with
known valid upper variance bound; and (iii) the local variance relation
$V\le U+h^2+\Delta_k$, which fixes the attainable exponent and the
corrected center. Proposition~\ref{prop:sign-scope} requires none of the
magnitude-estimation property. If cheap variance is an unknown monotone
function of $|h|$, different calibrations of that function can agree on
the observed variance while assigning it different magnitudes. Monotonicity
alone therefore supplies no inverse calibration or uniform $E_k$ bound.
A known invertible function would need a separate inverse-error analysis,
especially where its derivative degenerates.

Unequal cheap and diagnostic variances need not remove the sign lower
bound, but the pooled variance law in Appendix~\ref{app:pooled} and the
unknown-variance rotations in Appendix~\ref{app:simulation} then cease
to have their stated distributions. Unknown intervention directions need
not admit the same sign pair or preserve a scalar diagnostic mean.
No unchanged frontier, critical center, or exact adaptive sandwich follows
in either setting from the present upper proof.

\paragraph{What the critical-window evidence can establish.}
Theorem~\ref{thm:window} places the necessary and sufficient allowances
within $O(w_k)$ of $z_c$, $w_k=\sqrt{\ell_k\log\ell_k}$.
Neither its constants nor simulations of a particular procedure locate
the minimax threshold inside that window. In particular, a gate's failure
at $z_c$ does not rule out another feasible learner--assessor pair there.
The bounds permit a smaller-order correction as well as an order-$w_k$
gap; possible dependence of its coefficient on the nuisance or reference
class remains unresolved. The finite-budget checks are reported as
procedure diagnostics rather than evidence resolving this open threshold.

%% file: appendix/quality.tex
\section{Joint choice under a learning-quality constraint}\label{app:quality}
Keep the pooled fixed experiment of Appendix~\ref{app:pooled} and the reference $b=b_m(O)$,
with $|b|\le B$ and $\sup_q\E(b-c)^2\le A_{\rm ref}/m$.
A learner $f$ may use the entire experiment and a visible independent random
seed, producing $a_f$. All functions are independent of unknown $q$; they may
use the known bounds, including $\Delta$, and budgets. Set
\[
 \mathcal F_2=\{f:\sup_q\E_q(a_f-\theta)^2\le C_2/k\},\qquad
 G_f=(\theta-b)^2-(\theta-a_f)^2,
\]
where $C_2\ge V_*=U+R^2+1/2$ is fixed independently of sample counts.

\paragraph{Proof of Theorem~\ref{thm:quality}.}
\begin{proof}
The clipped diagnostic mean belongs to $\mathcal F_2$ and has fourth error
moment at most $3V_*^2/k^2$. The pooled fitted-potential assessor of
Appendix~\ref{app:pooled} has risk $Cr_N$. Reporting it for this learner's gain gives
the upper bound by $A-G_f=(\theta-a_f)^2$ and the squared triangle inequality.

For a lower bound fix any $f\in\mathcal F_2$ and assessor. At a common full
transcript the gain is $2(a_f-b)\theta+b^2-a_f^2$. If two parameters have
gain gap at least $g$ on an event $B_0$, the shared-transcript inequality gives
\[
 \max(R_0,R_1)\ge(g^2/4)[P_0(B_0)-\operatorname{TV}(P_0,P_1)].
\]
Independent random seeds do not change the information bounds.

Set $r_0=R/4$, $\tau_0=(L+U)/2$, and choose a fixed sign $\sigma$ with
$P_{c=0,\tau_0}(\sigma b\le0)\ge1/2$. At a null $h_0=\sigma r_0$, use
$B_0=\{\sigma b\le0,\ |a_f-\sigma r_0|\le r_0/2\}$.
For $k\ge32C_2/r_0^2$, learning accuracy implies $P_0(B_0)\ge3/8$.
On this event $\sigma(a_f-b)\ge r_0/2$, so the alternative
$h_1=\sigma(r_0+d)$ has gain gap at least $r_0d$.

Put $H_0=(2r_0+1)^2/(4L^2)+1/(2L)$ and choose positive $\varepsilon$ at most
\[
 \min\{1,R/4,(U-L)/(2(2r_0+1)),1/(2r_0+1),(32H_0)^{-1/2}\}.
\]
With $D_2=2r_0d+d^2$, use the three pairs from Appendix~\ref{app:proofs} with this sign:
(i) $d=\varepsilon/\sqrt{n+k}$, unchanged $\tau$, zero increments;
(ii) $d=\varepsilon/\sqrt{m+k}$, $\tau_1=\tau_0-D_2$, zero increments;
(iii) $d=\varepsilon\min(\Delta,k^{-1/2})$, unchanged $\tau$,
$(\zeta_0,\zeta_1)=(\Delta,\Delta-D_2)$ when $\Delta>0$.
In every pair $c=0$. The displayed constraints ensure feasibility and
KL at most $1/32$, hence TV at most $1/8$. The third pair changes the null
variance but not its observational sign probability; learner accuracy holds
uniformly there. The second changes the alternative observational law, but
the overlap bound needs only the null event probability.
Each pair gives risk at least $r_0^2d^2/16$.
Their maximum bounds $r_N/6$ times a fixed constant, using
$n+k\ge k$ and $(m+k)^{-1}\ge\min(m^{-1},k^{-1})/2$.
Thus $R_{G_f}\ge c_r r_N$ uniformly in $f,m,n$.

For the additional floor first suppose $m\ge k$. Fix
$K^2\ge64\max(C_2,A_{\rm ref},1)$ and take $k\ge K^2/R^2$.
Compare $c=0$, $\tau=\tau_0$, $\zeta=0$, $h=\pm r$, $r=K/\sqrt k$.
Supply the assessor, not a retrained learner, with the common magnitude
and nuisance parameters. Under the equal prior, the diagnostic mean $M$
contains all sign information, and
\[
 \operatorname{Var}(\theta\mid X,S)=r^2\operatorname{sech}^2(krM/V).
\]
Conditional on data and seed the learner and reference are fixed, so the Bayes
gain risk is $\E[4r^2(a_f-b)^2\operatorname{sech}^2(krM/V)]$.
This lower bounds maximum risk even in the original experiment.
This step does not assume $G_f\in L_2$: at each transcript, with posterior
sign probability $p$ and $g_B=pG_++(1-p)G_-$,
\[
 p(g-G_+)^2+(1-p)(g-G_-)^2
  =(g-g_B)^2+p(1-p)(G_+-G_-)^2.
\]
Integrating these nonnegative terms is valid even for infinite competing
risks. The attained posterior-variance risk is finite because
$\E(a_f-b)^2<\infty$ under the second-moment constraint. Thus no hidden
fourth-moment assumption is introduced by the Bayes calculation.

Let $W=\sqrt kM$ and $D_0=K+4\sqrt{V_*}$. Under the positive sign,
the events $|a_f-r|>r/2$, $|b|>r/4$ and $|W|>D_0$ have probabilities at
most $1/16$, $1/4$ and $1/16$, respectively. These follow from learning MSE,
reference MSE and Chebyshev; independence of events is unnecessary.
Their complementary intersection has probability at least $1/4$ under the
mixture. There $a_f-b\ge r/4$ and the posterior weight is at least
$w_0=\operatorname{sech}^2(KD_0/L)>0$. Risk is therefore at least
$K^4w_0/(16k^2)$. No small-TV assertion is made for this large-$K$ pair.

When $m<k$, $r_N=1/k$ already implies a lower bound for
$r_N+k^{-2}\le2r_N$. Combining the regions and taking
$k_0=\lceil\max(16,32C_2/r_0^2,K^2/R^2)\rceil$ proves the result.
All lower constants are uniform over $f$, so taking its infimum is valid.
\end{proof}

The direct lower argument is important because the fourth-moment transfer alone does not cover the learner class: a bounded rule that outputs $T$ with independent probability $1/k$
and otherwise the clipped mean has uniform MSE $O(1/k)$ but fourth moment at
least $T^4/k$ at $\theta=0$. The direct lower argument avoids that implication.

\paragraph{Adaptive constrained joint selection.}
Now allow a parameter-independent policy $\pi$ to interleave, randomize and
stop, with pathwise caps of $n$ cheap samples and $k$ diagnostics
$N(x\theta,V)$, $|x|\le1$. A learner $f$ and assessor $g$ use the complete
policy transcript and visible independent randomness. Require
$\sup_q\E_q^\pi(a_f-\theta)^2\le C_2/k$, with $k$ the cap, not the
realized diagnostic count. Let $\mathcal R_{\rm ad}^{Q}(m,n,k)$ be the
infimum of maximum gain-assessment risk over such feasible pipelines, and
$\mathcal R_{\rm fix}^{Q}$ the joint risk in Theorem~\ref{thm:quality}.
The quality constant, reference and parameter class remain identical.

\paragraph{Corollary.}
Under the preceding model and threshold,
\[
 \mathcal R_{\rm ad}^{Q}(m,n,k)\asymp r_N+k^{-2}.
\]
To prove it, the exact simulator in Appendix~\ref{app:simulation} gives
\[
 \mathcal R_{\rm fix}^{Q}(m,n+1,k)
 \le\mathcal R_{\rm ad}^{Q}(m,n,k)
 \le\mathcal R_{\rm fix}^{Q}(m,n,k).
\]
The right inclusion is immediate. For the left, simulate any feasible policy's
transcript from the augmented fixed experiment, retaining $O$ unchanged, and
apply its original $f,g$ to that transcript. Exact equality of laws preserves
both the learning-MSE constraint and the risk for
$(\theta-b(O))^2-(\theta-f(X_{\rm sim}))^2$. The composed learner is an
admissible fixed-source rule, not the clipped mean trained on the source.
The fixed theorem permits all such rules and independent simulator randomness.
Taking infima proves the sandwich. Finally replacing $n$ by $n+1$ changes
$E_N$ and $\min(E_N,1/k)$ by at most a factor of two, yielding the rate.

This is rate equivalence, not equality of same-budget risks or exclusion of
constant-factor adaptive improvements. Exact simulation matters: an approximate
bounded-loss comparison would not automatically apply to this second-moment
learner class. Common Gaussian variance, scalar direction and pathwise caps
remain essential; no expected-budget or heterogeneous-variance claim follows.

%% file: appendix/sharp.tex
\section{Diagnostic abstention and the leading assessment exponent}\label{app:sharp}
We use the unknown-parameter model and
observational reference. It is not an oracle construction. Write
$E_k=\Delta_k^2+(n_k+k)^{-1}+m_k^{-1}$,
$\ell_k=\log(1/(k^2E_k))$ and $z_k=k\eta_k$.
All constants below may depend on fixed model/reference bounds.

\paragraph{Proof of Theorem~\ref{thm:sharp}.}

\paragraph{Implementable construction.}
Fix $\epsilon\in(0,1/2)$, put $t^2=(1-\epsilon)\eta$ and use the same
pooled estimate $Q=\operatorname{clip}(\widehat V-\widehat\tau,0,R^2)$.
For $Q\ge t^2$ define
\[
 d(Q)^2=\frac{8V_*}{k}\log\frac{8Q}{\epsilon\eta},\qquad
 T(Q)=\max\{0,\sqrt Q-d(Q)\}.
\]
Set $I=1_{\{Q\ge t^2,\ |M-b|\ge T(Q)\}}$,
$s=\operatorname{sign}(M-b)$, $a=b+Is\sqrt Q$ and $g=IQ$.
When $Q<t^2$, set $I=0$ without evaluating the logarithm.
The rule uses a known variance bound, not the actual variance or effect.

\begin{proof}
Put $r=|h|$, $e=c-b$, $\delta=Q-r^2$, $\beta=A_{\rm ref}/m$ and
$C_Q=\max(1,4V_*^2)$. Then $\E\delta^2\le C_QE_k$ and
$\E e^2\le\beta$. In the direct Gaussian experiment
$\xi=M-\theta\sim N(0,V/k)$ is independent of $(Q,b)$; $Q$ and $b$
need not be independent.

First verify learning feasibility. Write $D=\log(8(1-\epsilon)/\epsilon)$,
so $d(Q)^2\ge8V_*D/k$ on magnitude eligibility. The inactive magnitude
$H=\E[r^2(1-I)]$ is bounded on $Q<t^2$ by $t^2+\E|\delta|$.
On $Q\ge t^2$, the event
$|\sqrt Q-r|+|e|>d(Q)/2$ contributes at most
$k(C_QE_k+R^2\beta)/(V_*D)$, using
$r^2(\sqrt Q-r)^2\le\delta^2$ and the squared-sum inequality.
The event $|\delta|>Q/2$ contributes at most $3\E|\delta|$.
On the complement, a failed gate implies
$\operatorname{sign}(h)\xi<-d(Q)/2$ and $r^2\le3Q/2$.
Its conditional Gaussian probability is at most $\epsilon\eta/(8Q)$.
Terms weighted by $r^2$ vanish when $r=0$. Thus
\[
 H\le(1-13\epsilon/16)\eta+4\sqrt{C_QE_k}
       +\frac{k(C_QE_k+R^2\beta)}{V_*D}.
\]
For wrong diagnostic sign $W$, the earlier bound
$r^2P(W)\le4\beta+8V_*/k$ still applies. Decomposing
$a-\theta=-e-(1-I)h+Is(\sqrt Q-r)+I(sr-h)$ gives
\[
 \|a-\theta\|_2\le\sqrt H+\sqrt\beta+(C_QE_k)^{1/4}
                   +2\sqrt{4\beta+8V_*/k}.
\]
All remainders are $o(\sqrt\eta)$, uniformly over the parameter class,
and the fixed $\epsilon$ supplies positive slack. Learning MSE is
therefore at most $\eta$ eventually.

Next, the exact gain envelope is
\[
 (g-G)^2\le12\delta^2+48IQr^2 1_W+12R^2e^2.
\]
Fix $\alpha\in(0,1/4)$. Events $|\delta|>\alpha r^2$ and
$|e|>\alpha r$ contribute $O(E_k)$ to the weighted middle term:
respectively $Qr^2\le(1+\alpha)\delta^2/\alpha^2$ and
$Qr^2\le R^2e^2/\alpha^2$. On their complement, activation implies
$r^2\ge v\eta$, $v=(1-\epsilon)/(1+\alpha)$, and
\[
 \frac{d(Q)^2}{r^2}\le\frac{8V_*}{z_k}
   \sup_{x\ge v}\frac{\log(8(1+\alpha)x/\epsilon)}x
 =\frac{8V_*D}{v z_k}=:a_k^2\longrightarrow0.
\]
The equality follows by differentiation, since
$8(1+\alpha)v/\epsilon>8$. Eventually $T(Q)>0$ in this region.
Wrong-sign activation implies
$\operatorname{sign}(h)\xi\le-b_kr$, where
$b_k=1+\sqrt{1-\alpha}-\alpha-a_k$.
Conditional integration therefore bounds the good-region term by
\[
 (1+\alpha)r^4\exp[-B_kkr^2/V]1_{\{r^2\ge v\eta\}},
 \qquad B_k=b_k^2/2\longrightarrow B_\alpha.
\]
Split the exponential into fractions $\rho$ and $1-\rho$, with fixed
$\rho\in(0,1)$. Since $V\le U+r^2+\Delta_k$, monotonicity gives
\[
 (1-\rho)B_kkr^2/V\ge
 \frac{(1-\rho)B_kv z_k}{U+v\eta+\Delta_k}.
\]
The other fraction absorbs $r^4$ at cost $Ck^{-2}$ using $V\le V_*$
and boundedness of $y^2e^{-cy}$. By choosing fixed $\alpha,\rho$ small
enough, for every fixed $a_0<2(1-\epsilon)/U$ this proves
\[
 \mathcal R_k(\eta_k)\le C_1E_k+C_2k^{-2}e^{-a_0z_k}. \tag{J1}
\]

For the sharp lower exponent, fix $\varepsilon>0$ and take the legal
pair $c=0$, $\tau=U$, $\zeta=0$, $h=\pm r_k$ with
$r_k^2=(1+\varepsilon)\eta_k$ and $V_k=U+r_k^2$.
The likelihood ratio under the positive point has log
$2s_k+2\sqrt{s_k}Z$, where $s_k=kr_k^2/V_k$ and $Z\sim N(0,1)$.
Choose fixed $d\in((1+\varepsilon)^{-1/2},1)$ and $v_0=(1-d)/2$.
Markov bounds and a sufficiently large fixed Gaussian cutoff $H_0$
give positive constant probability for
\[
 |a-r_k|\le dr_k,\quad |b|\le v_0r_k,\quad
 \log\Lambda\le2s_k+2\sqrt{s_k}H_0.
\]
Indeed the first failure probability is at most
$1/((1+\varepsilon)d^2)<1$, the second tends to zero since
$\beta/\eta_k\to0$, and the third can be made arbitrarily small.
On this event $a-b\ge v_0r_k$. The pointwise two-sign square identity
bounds every assessor's maximum risk below by
$\E_+[8r_k^2(a-b)^2/(1+\Lambda)]$.
It follows uniformly over feasible learners that
\[
 k^2\mathcal R_k(\eta_k)\ge C_\varepsilon z_k^2
                     e^{-2s_k-2\sqrt{s_k}H_0}.
\]
No second moment of the target is required for this nonnegative
pointwise argument. Divide logarithms by $z_k$, then let the fixed
$\varepsilon$ decrease to zero after taking limits, to obtain
\[
 \limsup\frac{\log(1/(k^2\mathcal R_k(\eta_k)))}{z_k}\le2/U. \tag{J2}
\]
The existing nuisance bound is $\mathcal R_k\ge cE_k$ for vanishing
$\eta_k$. Combining it with (J2) gives the upper log limit in the
theorem. For every $a_0<2/U$, choose fixed $\epsilon$ small enough
in (J1). This gives the lower log limit; send $a_0$ upward to $2/U$
only after taking limits. This works even if $\ell_k/z_k$ oscillates.

For $C>U/2$, choose $1/C<a_0<2/U$ in (J1) to obtain $O(E_k)$ at
$\eta=C\ell_k/k$, then use the nuisance lower bound. For necessity,
the earlier order bound implies $k\eta_k\to\infty$. On any subsequence
with bounded $k\eta_k/\ell_k$, the assumption $\ell_k=o(k)$ implies
$\eta_k\to0$, so (J2) forces its limiting ratio to be at least $U/2$.
Unbounded ratios cause no difficulty.
Finally the exact simulator lower sandwich uses source counts
$(m,n+1,k)$, preserves the quality constraint, and changes $\ell_k$ by
$o(1)$. Fixed direct sampling supplies the adaptive upper bound.
\end{proof}

\paragraph{One fast slack schedule for both main results.}\label{app:fast-unified}
The leading-log construction above can be implemented with the same schedule
as the critical-window upper bound, namely
$\epsilon=\min\{1/4,\sqrt{\log z/z}\}$ for $z>e$ and $1/4$ otherwise.
Here we verify this uniformly throughout the full regime of
Theorem~\ref{thm:sharp}, not only at the critical allowance.
Put $A_k=k^2E_k\to0$ and $z=k\eta\to\infty$.
The finite learning inequalities above give
\[
 H/\eta\le1-13\epsilon/16+
 O\!\left(\frac{\sqrt{A_k}}{z}+\frac{A_k}{zD}\right),\qquad
 D=\log\frac{8(1-\epsilon)}\epsilon.
\]
The remainder is $o(\epsilon)$. The other terms in the root-learning bound,
divided by $\sqrt\eta$, are $O((1+A_k^{1/4})/\sqrt z)=o(\epsilon)$.
Consequently $\|a-\theta\|_2\le\sqrt\eta$ eventually, uniformly over $q$;
the positive slack proves feasibility at the declared allowance itself.

For the assessment bound, the finite inequality~\eqref{eq:joint-tail-upper}
in Appendix~\ref{app:window} has a nuisance constant independent of $\epsilon$:
\[
 \E(g-G)^2\le C E_k+
 144\max_{w\in[t^2,R^2]}w^2
 \exp\!\left[-\frac{B_*kw}{U+\Delta_k+w}\right],
 \quad t^2=(1-\epsilon)\eta,
\]
where $a_*^2=8V_*D/(kt^2)=O(\log z/z)\to0$ and
$B_*=(2-a_*)^2/2\to2$. Its conditions $a_*\le1/4$ and $t^2\le R^2$
hold eventually. Fix any $a_0<2/U$ and choose a fixed $\rho\in(0,1)$
with $a_0<2(1-\rho)/U$. Monotonicity of $w/(U+\Delta_k+w)$ gives
\[
 \frac{(1-\rho)B_*kw}{U+\Delta_k+w}
 \ge\frac{(1-\rho)B_*kt^2}{U+\Delta_k+t^2}\ge a_0z
\]
eventually, since $\eta,\Delta_k,\epsilon\to0$. The remaining exponential
fraction absorbs $w^2$ at cost $C_\rho k^{-2}$, using
$U+\Delta_k+w\le V_*$ and $B_*$ bounded away from zero.
Thus this one rule satisfies $\E(g-G)^2\le CE_k+C_\rho k^{-2}e^{-a_0z}$
for every fixed $a_0<2/U$. Combined with the learner-uniform lower bound,
it attains~\eqref{eq:frontier}, as well as the critical-window upper bound.
The finite-window proof uses only the gate's moment and Gaussian tail
identities, so this deduction does not assume the leading-log conclusion.

\paragraph{Relation to rejection.}
If magnitude and reference are revealed and the available repairs are
$0,\pm r$, a gate with acceptance probability $\phi$ and accepted sign-error
probability $q$ has learning risk $r^2(1-\phi)+4r^2q$ and gain-assessment
risk $16r^4q$. A Lagrangian gives a conditional-error rejection threshold,
consistent with the classical rejection framework \citep{franc}.
The unknown-nuisance proof above quantifies the uniform exponent and the
single implementable schedule. The conditional-integration suite in
Appendix~\ref{app:experiments} separately measures sign-tail suppression
and the total nuisance-error floor.

%% file: appendix/critical_window.tex
\section{Joint tail stability and a quantitative critical window}\label{app:window}
We use the model, reference contract and
same-transcript target.
Write $z=k\eta$, $E=\Delta^2+(n+k)^{-1}+m^{-1}$ and
$\ell=\log(1/(k^2E))$. Constants depend only on fixed model/reference bounds,
unless indicated otherwise.

\paragraph{Proof of Theorem~\ref{thm:window}.} The following finite inequalities establish the necessary and sufficient allowances. Write $C_Q=\max(1,4V_*^2)$ and $\beta=A_{\rm ref}/m$.

\subsection{Finite sign lower bound}
Let $\beta=A_{\rm ref}/m$ and choose $0<\varepsilon\le1$ satisfying
$(1+\varepsilon)\eta\le R^2$ and
$\beta/\eta\le\varepsilon^3/4704$.
Use the legal pair $c=0,\tau=U,\zeta=0,h=\pm r$,
$r^2=(1+\varepsilon)\eta$. Put
\[
 d=\frac{1+(1+\varepsilon)^{-1/2}}2,\quad
 v=\frac{1-d}2\ge\frac{\varepsilon}{14},\quad
 \kappa=1-\frac1{(1+\varepsilon)d^2}\ge\frac{\varepsilon}{6}.
\]
The inequalities follow by rationalizing $1-(1+\varepsilon)^{-1/2}$
and using $(\sqrt{1+\varepsilon}+1)^2<6$.
Under the positive point, $\log\Lambda=2s+2\sqrt{s}Z$,
$s=kr^2/(U+r^2)$. Set $H=\sqrt{2\log(24/\varepsilon)}$.
The failures of $|a-r|\le dr$, $|b|\le vr$ and
$\log\Lambda\le\lambda:=2s+2\sqrt{s}H$ have probabilities at most
$1-\kappa$, $196\beta/(\varepsilon^2\eta)\le\kappa/4$ and
$\exp(-H^2/2)\le\kappa/4$, respectively.
Their intersection has probability at least $\kappa/2$, without independence.
There $a-b\ge vr$. The weighted square identity from Appendix~\ref{app:sharp}
therefore bounds every feasible learner and assessor below by
$2v^2\kappa r^4e^{-\lambda}$. Hence
\begin{equation}\label{eq:finite-sign-window}
 k^2\mathcal R_k(\eta)\ge\frac{\varepsilon^3z^2}{588}
 \exp\{-2s-2\sqrt{s}H\},\qquad
 s=\frac{(1+\varepsilon)z}{U+(1+\varepsilon)z/k}.
\end{equation}
No second moment of the random gain is needed for the pointwise identity.
An empty feasible class has infinite minimax risk.

For $z\ge e$, take $\varepsilon=\sqrt{\log z/z}$.
If $k^2E\le1$ and $\eta\le1$, then
$\beta/\eta\le A_{\rm ref}/z^2$. Thus the reference condition holds when
$4704A_{\rm ref}\le\sqrt z(\log z)^{3/2}$, eventually in the sharp regime.
Define $\psi_k(z)=2z/(U+z/k)$. Its derivative is
$\psi'_k(z)=2U/(U+z/k)^2\le2/U$, so
$2s\le\psi_k(z)+2\varepsilon z/U$.
For the square-root remainder use $s\le(1+\varepsilon)z/U$.
With $C_U=2/U+4\sqrt{(\log24+1/2)/U}$, taking logs gives
\begin{equation}\label{eq:window-log-lower}
 \log\frac1{k^2\mathcal R_k(\eta)}\le
 \psi_k(z)+C_U\sqrt{z\log z}
 +\log588-\tfrac12\log z-\tfrac32\log\log z.
\end{equation}
The previous order necessity ensures $z\to\infty$ whenever
$\mathcal R_k(\eta)\le D_0E$ for fixed $D_0$.
On indices where $z<z_c$, $\eta\to0$ because $z_c=O(\ell)=o(k)$;
all preceding conditions hold eventually. Combining
$\ell-\log D_0\le\log(1/(k^2\mathcal R_k))$ with
\eqref{eq:window-log-lower} and $\psi_k(z_c)=\ell$ gives
$\psi_k(z_c)-\psi_k(z)\le O(\sqrt{\ell\log\ell})$.
On this interval $\psi'_k$ is eventually at least $1/(2U)$;
the mean value theorem proves the necessary bound.
On all other indices that bound is automatic.

\subsection{A finite joint-tail inequality}
Use the same $Q$, margin $d(q)$ and gate as Appendix~\ref{app:sharp}. Write
\[
 t^2=(1-\epsilon)\eta,\quad D=\log\frac{8(1-\epsilon)}\epsilon,
 \quad a_*^2=\frac{8V_*D}{kt^2},\quad B_*=(2-a_*)^2/2.
\]
Assume $0<t^2\le R^2$ and $a_*\le1/4$.
With $C_Q=\max(1,4V_*^2)$ we claim
\begin{equation}\label{eq:joint-tail-upper}
 \sup\E(g-G)^2\le(5196C_Q+2316R^2A_{\rm ref})E
 +144\max_{w\in[t^2,R^2]}w^2
          \exp\left\{-\frac{B_*kw}{U+\Delta+w}\right\}.
\end{equation}
The constant multiplying $E$ does not grow as $\epsilon$ decreases.
If $t^2>R^2$ the rule is everywhere inactive, but this alone does not
certify learning quality.

Here is a complete joint comparison proving the claim. Put
$H(s)=\sqrt{\Phi(-s)}$. For $s\ge0$,
\[
 H(s)\le1,\qquad s|H'(s)|\le3.
\]
Indeed integrating the normal density over $[s,s+1/(s+1)]$ gives
$\Phi(-s)\ge e^{-3/2}\phi(s)/(s+1)$. Thus
$|H'(s)|\le c\sqrt{s+1}e^{-s^2/4}$ with
$c=e^{3/4}/(2(2\pi)^{1/4})<1$.
Use $\sqrt{s+1}\le s+1$ and
$\sup s^2e^{-s^2/4}+\sup se^{-s^2/4}=4/e+\sqrt{2/e}<3$.

Condition on $Q=q,b$, put $x=h^2>0$, $e_0=c-b$ and
$\sigma=\sqrt{V/k}$. For eligible $q\ge t^2$, the conditional weighted
wrong-sign term is exactly
\[
 F=xq\Phi\left(-\frac{\sqrt x+\sqrt q-d(q)+\operatorname{sign}(h)e_0}
 {\sigma}\right).
\]
The margin obeys $d(y)/\sqrt y\le a_*\le1/4$ for all $y\ge t^2$,
by differentiating $\log(Cy)/y$. Thus the gate is positive there.
Set $w=\max\{x,t^2\}$. We show the pointwise inequality
\begin{equation}\label{eq:joint-comparison}
 F\le3w^2\Phi\left(-\frac{2\sqrt w-d(w)}\sigma\right)
       +108(q-x)^2+48R^2e_0^2.
\end{equation}
If $q>2x$ or $q<x/2$, then $F\le xq\le2(q-x)^2$.
If $|e_0|>\sqrt x/4$, then $F\le R^2x\le16R^2e_0^2$.
The case $x=0$ has zero weight. Otherwise $x/2\le q\le2x$ and
$|e_0|\le\sqrt x/4$. Since $q\ge t^2$,
$|x-w|,|q-w|\le|q-x|$ and $w\le R^2$.

Along a coordinate path from $(x,q,e_0)$ to $(w,w,0)$ consider
\[
 J(X,Y,e)=\sqrt{XY}\,
 H\left(\frac{\sqrt X+\sqrt Y-d(Y)+\operatorname{sign}(h)e}{\sigma}\right).
\]
Here $X\in[x,2x]$, $Y\in[x/2,2x]\cap[t^2,R^2]$,
$|e|\le\sqrt x/4$ and $X,Y\le R^2$. Its standardized argument is at
least $3(\sqrt X+\sqrt Y)/(4\sigma)>0$.
Also $0\le(\sqrt Y-d(Y))'\le1/(2\sqrt Y)$, since
$d'(Y)=d(Y)/(2Y\log(8Y/(\epsilon\eta)))$.
The preceding tail derivative bound gives
\[
 |\partial_XJ|\le3,\quad |\partial_YJ|\le3,
 \quad |\partial_eJ|\le4R.
\]
For the $X$ derivative the prefactor term is at most
$\sqrt{Y/X}/2\le1$ and the other term is at most
$2s|H'(s)|/3\le2$; the $Y$ derivative uses $X/Y\le4$.
The reference derivative is at most $4R s|H'(s)|/3$.
Integrating derivatives and squaring proves \eqref{eq:joint-comparison}.
The baseline uses the actual fixed $\sigma$, not a fictitious changed data law.
Moving it to $\max\{x,t^2\}$ avoids differentiating the eligibility indicator.

Finally combine \eqref{eq:joint-comparison} with the existing gain envelope
$12(Q-h^2)^2+48IQh^2 1_W+12R^2(c-b)^2$.
Use $\E(Q-h^2)^2\le C_QE$, $\E(c-b)^2\le A_{\rm ref}E$ and
$V\le U+x+\Delta\le U+w+\Delta$. The Gaussian tail bound proves
\eqref{eq:joint-tail-upper}. In particular it requires only the stated
second-moment reference contract, not additional independence or concentration.

\subsection{Learning feasibility and sufficient allowance}
For the theorem's fast slack, eventually
$\epsilon=\sqrt{\log z/z}$, $D=O(\log z)$ and
$a_*=O(\sqrt{\log z/z})$.
The nonasymptotic learning inequalities in Appendix~\ref{app:sharp} imply
an inactive remainder $O(1/z)$ relative to $\eta$ and other root-risk
terms $O(1/\sqrt z)$ relative to $\sqrt\eta$, once $k^2E\le1$.
Both are $o(\epsilon)$. Hence the same positive-slack argument gives
learning MSE at most $\eta$ eventually.

Put $A=U+\Delta$. The derivative of
$\log(w^2e^{-B_*kw/(A+w)})$ is $2/w-B_*kA/(A+w)^2$.
Eventually $B_*kt^2A\ge2(A+R^2)^2$, so it is nonpositive throughout
$[t^2,R^2]$. The maximum in \eqref{eq:joint-tail-upper} occurs at $t^2$.
Put $v=1-\epsilon$. Keeping the exact denominator, rather than its
first-order expansion, gives the explicit exponent comparison
\[
 \psi_k(z)-\frac{B_*vz}{U+vz/k+\Delta}
 \le\frac{2\epsilon z}{U}+\frac{2a_*z}{U}
                         +\frac{2z\Delta}{U^2}.
\]
To verify it, successively replace $z$ by $vz$ in $\psi_k$, add $\Delta$
to the denominator, and replace the numerator coefficient $2$ by $B_*$;
use $\psi'_k\le2/U$ and $2-B_*\le2a_*$. Thus
\[
 \mathcal R_k(\eta)\le C_0E+C_1k^{-2}z^2
 \exp\{-\psi_k(z)+C_2\sqrt{z\log z}+C_3z\Delta\}.
\]
For the stated sufficient allowance, $z/\ell\to U/2$ and $\eta\to0$.
Moreover $\Delta\le\sqrt E=e^{-\ell/2}/k$,
$z\Delta\to0$ and $\sqrt{z\log z}=O(\sqrt{\ell\log\ell})$.
Since $\psi'_k\ge1/(2U)$ eventually between $z_c$ and $z$,
$\psi_k(z)-\ell\ge K\sqrt{\ell\log\ell}/(2U)$.
A sufficiently large fixed $K$ absorbs the remainder and $2\log z$,
proving $O(E)$. The earlier nuisance lower bound gives the matching order.
The exact source simulator transfers the lower bound to adaptive policies;
fixed direct sampling supplies their upper bound.

\paragraph{When the uncorrected equality fails.}
The centering shift is
$z_c-U\ell/2=U\ell^2/(4k-2\ell)$.
For $z_0=U\ell/2$ the exact exponent deficit is
$\ell-\psi_k(z_0)=\ell^2/(2k+\ell)$.
Equation~\eqref{eq:window-log-lower} consequently gives
\[
 \log\frac{\mathcal R_k(z_0/k)}{E_k}
 \ge\frac{\ell^2}{2k+\ell}-O(\sqrt{\ell\log\ell}).
\]
It diverges if $\ell^3/(k^2\log\ell)\to\infty$, proving the final claim.
This regime is nonempty: $\Delta=0$ and
$m=n=\lceil2k^2\exp(k^a)\rceil$, fixed $2/3<a<1$, give
$\ell=k^a+o(1)$. These enormous auxiliary budgets are an existence example,
not an empirical recommendation. Polynomial auxiliary budgets do not meet
this condition; their leading-order equality is not decided here.

For $\ell=k^a$ with $2/3<a<1$, the center shift divided by
$\sqrt{\ell\log\ell}$ diverges although $\ell/k\to0$.
Figure~\ref{fig:window-correction} illustrates this exact algebraic ratio, not simulated risks.
\begin{figure}[t]
\centering
\includegraphics[width=.72\linewidth]{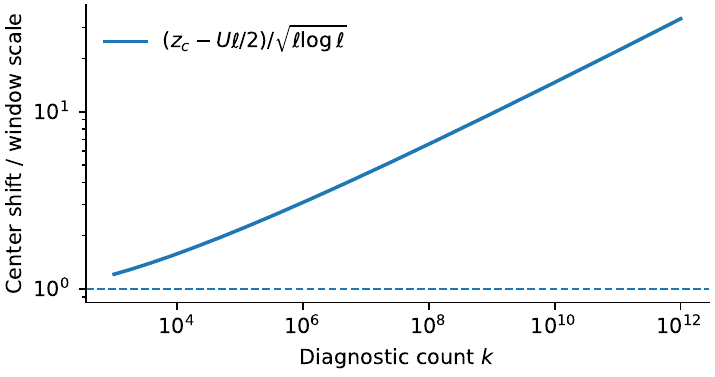}
\caption{Analytic center-shift/window ratio for $\ell=k^{0.8}$ and $U=2.5$.
The condition $\ell=o(k)$ does not make the correction smaller than the allowance window.}
\label{fig:window-correction}
\end{figure}

%% file: appendix/simulation.tex
\section{Exact and approximate simulation of adaptive diagnostics}\label{app:simulation}
This exact experiment comparison transfers the learner-constrained frontier to adaptive policies without changing the assessment target.

\begin{theorem}
In the stated common-variance Gaussian model, the fixed experiment with the same
$m$ observations, $n+1$ cheap samples and $k$ direct diagnostics at $x=1$
can simulate every policy capped at $n$ cheap and $k$ diagnostic requests.
The simulation kernel is independent of all unknown parameters. For $k>0$,
zero additional cheap samples do not suffice for this fixed-direct source,
even allowing the observational sample to be transformed. Hence one is the
minimum integer cheap-sample augmentation with $m,k$ unchanged.
\end{theorem}

\paragraph{Construction.}
Write $\theta=c+h$. For $k\ge1$, an orthogonal transformation of the direct
diagnostics produces $Z\sim N(\sqrt{k}\theta,V)$ and $k-1$ independent
$N(0,V)$ coordinates. With the $n+1$ cheap samples this gives $n+k$ centered
coordinates $E_1,\ldots,E_{n+k}$. Set $r=\sqrt{k}$ and retain the
observational sample $O$ unchanged.

Given the simulated history and fresh policy randomness, a request has mean
coefficient $d=0$ if cheap, or $d=x\in[-1,1]$ if diagnostic. Using the next
centered coordinate $E$, when $r>0$ set
\[
 \alpha=d/r,\quad\beta=\sqrt{1-\alpha^2},\qquad
 Y=\alpha Z+\beta E,\quad Z'=\beta Z-\alpha E,\quad r'=\beta r.
\]
Condition on $O$, the public simulated history and the current request, but
not on retained or future coordinates. Inductively $(Z,E)$ has law
$N((r\theta,0),VI_2)$. The displayed orthogonal map therefore has conditional
joint density $\phi_V(y-d\theta)\phi_V(z'-r'\theta)$. Revealing $Y$
preserves the required conditional law for $Z'$, independently of unused
noise. No unconditional Gaussianity after adaptive rotations is assumed.

After $j$ diagnostics, $r^2=k-\sum_{i\le j}x_i^2\ge k-j$. Thus before
any further diagnostic $r\ge1$, making the rotation feasible. At $r=0$
only $d=0$ can be requested; return $E$ and retain $Z$. At most $n+k$
coordinates are consumed. Copying the original policy kernels and discarding
unused coordinates on stopping matches the entire transcript law. For $k=0$
simply supply cheap outcomes from the bank.

\paragraph{Necessity for this source.}
Restrict to the legal open submodel $h=\zeta=0$, $c=\theta\in(-C_0,C_0)$,
$\tau=V\in(L,U)$ and put $N=n+k$. The same-budget fixed source has
observations $(Z_i,Y_i)$, cheap outcomes $C_j$, and direct outcomes $D_\ell$.
The legal target policy chooses $x=0$ for every diagnostic: it has observations
$(Z'_i,Y'_i)$ with the same law and $N$ centered outcomes $B_j$.

Suppose a parameter-independent kernel mapped source to target. For fixed
$\theta$, their $V$-scores are respectively
\begin{align*}
 s_F&=\frac{\sum_i(Y_i-\theta Z_i)^2+\sum_j C_j^2+
                    \sum_\ell(D_\ell-\theta)^2-(m+N)V}{2V^2},\\
 s_G&=\frac{\sum_i(Y'_i-\theta Z'_i)^2+\sum_j B_j^2-(m+N)V}{2V^2}.
\end{align*}
Both have second moment $(m+N)/(2V^2)$. Differentiation through a common
kernel gives $s_G=\E(s_F\mid\text{output})$; Gaussian square-integrable
scores justify this identity even for kernels without Lebesgue densities.
Equality of the second moments then forces $s_F=s_G$ almost surely.

Fix an interior $V_0$ and three distinct interior values of $\theta$.
Input laws are mutually absolutely continuous; appending the same kernel
preserves this property. Hence all three score identities hold on a common
full-probability set. After cancellation, they form a quadratic polynomial
identity in $\theta$. Its quadratic coefficients require
\[
       \sum_i(Z'_i)^2=\sum_i Z_i^2+k.
\]
Both covariate sums have expectation $m$, giving the contradiction $m=m+k$.
The argument even permits changing $O$ while preserving its prescribed joint
law. It proves the claimed zero-augmentation impossibility for $k>0$.

\paragraph{Interpretation and boundaries.}
For targets unchanged on shared $O$, simulation and inclusion of fixed policies give
$\mathcal R_{\rm fix}(n+1,k)\le\mathcal R_{\rm ad}(n,k)
\le\mathcal R_{\rm fix}(n,k)$.
Nonexistence of a same-budget kernel does not imply strict minimax-risk
improvement for any particular target. Nor is minimality claimed over other
changes to the source design or information resources.

If the total variance $V$ is supplied exactly, no augmentation is needed:
each requested diagnostic can be generated from an unused direct sample $D$
as $xD+\sqrt{1-x^2}\sqrt V\,\xi$, with fresh $\xi\sim N(0,1)$.
Cheap outcomes use the original bank. Knowing $\tau$ alone is not knowing
$V=\tau+h^2+\zeta$. Common Gaussian variance and bounded scalar assignments
are essential to the unknown-variance construction. Neither heterogeneous
variance nor general multivariate sensing is covered here.

\subsection{Same-budget approximate simulation}
Exact non-simulability does not imply a substantial approximate information gap.
The following result combines the reservoir construction with the
sufficient-statistic amplification approach of \citet{axelrod}.
The proof specializes this approach to a one-coordinate Gaussian scale
extension and derives the displayed constant.

\begin{theorem}
For $k\ge1$ and $d=n+k-1\ge1$, each admissible adaptive policy $\pi$ admits
a parameter-independent kernel from the same-budget fixed experiment such that
\[
 \sup_q \operatorname{TV}(K_\pi F_q(n,k),P_q^\pi(n,k))\le \frac1{2d}.
\]
The kernel preserves $O$. For a target $T(q,O)$ taking values in a fixed known
interval of length $B$, the minimax squared-error risks consequently satisfy
\[
 0\le\mathcal R_{\rm fix}(n,k)-\mathcal R_{\rm ad}(n,k)
       \le \frac{B^2}{2(n+k-1)}.
\]
\end{theorem}

\begin{proof}
Compress the fixed direct samples as before, retaining $Z$ and now only $d$
independent centered coordinates $E_i$. Set $S=\sum_i E_i^2$ and independently
draw $U$ uniformly on the unit sphere in $\R^{d+1}$. Replace the centered
reservoir by $\widetilde E=\sqrt{(d+1)S/d}\,U$, retaining $O,Z$.
This regenerates the reservoir rather than preserving the original samples.
Relative to $d+1$ genuine centered Gaussian coordinates, only the squared
radius differs: after division by $V$, the two laws are
$((d+1)/d)\chi_d^2$ and $\chi_{d+1}^2$.

Put $a=d/2$, $b=(d+1)/2$. These gamma laws have the same mean, and their
forward KL divergence is
\[
 F(b)=\log\Gamma(b)-\log\Gamma(a)-(b-a)\psi(a)
          +b\log(a/b)+(b-a).
\]
Here $F(a)=F'(a)=0$, while $F''(x)=\psi_1(x)-1/x\le1/x^2$:
the series $\psi_1(x)=\sum_{j\ge0}(x+j)^{-2}$ is at most
$x^{-2}+\int_0^\infty(x+t)^{-2}dt$.
Thus $F(b)\le(b-a)^2/(2a^2)=1/(2d^2)$ and Pinsker gives TV at most
$1/(2d)$. The common independent direction, $O$ and $Z$ preserve this bound.
Apply the exact reservoir kernel to obtain the transcript bound by contraction.

For the risk bound, clip decisions to the target interval. Squared loss is
then in $[0,B^2]$. Simulating an arbitrarily near-optimal adaptive policy and
decision changes its risk by at most $B^2/(2d)$ uniformly in the parameter.
Taking the minimax infimum proves the upper bound; inclusion of fixed policies
proves the lower bound.
\end{proof}

This implies relative risk equivalence along sequences with
$d\mathcal R_{\rm ad}\to\infty$, but not from this bound alone at risks of
order $1/d$ or smaller. No matching lower bound on approximation error is
claimed. Targets depending on intervention transcripts need a separate
target-preservation argument. When $k=0$ simulation is already exact.

%% file: appendix/strength.tex
\section{Supporting strength-assessment results}\label{app:strength}
For fixed $1<p<2$, define
\[
F(e,j)=\begin{cases}
e^{p/2},&j\le e^{-1/2},\\
e^{p-1}j^{p-2},&e^{-1/2}\le j\le e^{-1},\\
j^{-1},&j\ge e^{-1}.
\end{cases}
\]
The branches agree at their boundaries; at $j=0$ use the first branch.
\begin{theorem}\label{thm:strength}
For $m\ge32$, $n,k\ge0$, and $n+k\ge16$, minimax MSEs for fixed direct and adaptive sampling satisfy
\begin{align*}
 \mathcal R(|h|^p)&\asymp F(E_N,K)=\max\{F(E_N,k),m^{-1}\},\quad K=\min(k,m),\\
 \mathcal R(|c+h|^p)&\asymp\min\{1,k^{-1}\}.
\end{align*}
Constants depend on fixed model bounds and $p$.
\end{theorem}
\paragraph{Estimator.}
Fit raw OLS and let $\widehat c=\clip(c_{\rm raw},-C_0,C_0)$ and $D=M-\widehat c$. With $Q_{\rm raw}=\widehat V-\widehat\tau$, choose a measurable global minimizer
\[
 \widehat h\in\mathop{\rm argmin}_{|t|\le R}
 \left\{\frac{(Q_{\rm raw}-t^2)^2}{E_N}+K(D-t)^2\right\}
\]
and report $|\widehat h|^p$. Omit the mean term at $k=0$. The variance statistic is pooled as in Appendix~\ref{app:pooled}. For $K>0$, a minimizer has the sign of $D$. On $u=|t|\ge0$, the derivative equation is
\[
 2u^3+(K E_N-2Q_{\rm raw})u-K E_N|D|=0.
\]
For $D\ne0$ this has a unique positive root; clip it to $[0,R]$. The zero-$D$ case is $u=\sqrt{\clip(Q_{\rm raw}-K E_N/2,0,R^2)}$. The code uses bracketed bisection with deterministic tie handling.

\paragraph{Modulus behind the three regimes.}
The moment bounds yield $\E W\le C$ for
$W=(Q_{\rm raw}-h^2)^2/E_N+K(D-h)^2$. Comparison with the feasible true parameter gives
\[
 (\widehat h^2-h^2)^2/E_N+K(\widehat h-h)^2\le4W.
\]
Writing $s=|t|+|h|$ and $d=||t|-|h||$ bounds the squared functional gap by $p^2s^{2p-2}d^2$, while distance controls $d^2(K+s^2/E_N)$. Maximizing
$\sup_{0<s\le2R}s^{2p-2}\min\{s^2,(K+s^2/E_N)^{-1}\}$ yields the three branches. This model-specific minimum-distance construction uses the modulus perspective of \citet{donoholu}. Appendix~\ref{app:proofs} supplies feasible pairs and all lower bounds.

%% file: appendix/proofs.tex
\section{Functional and oracle-target proofs}\label{app:proofs}
These proofs use the fixed model bounds in the main text.
Write $V_*=U+R^2+1/2$, $E=E_N$, $N=n+k$, and $\theta=c+h$.

\subsection{Information and testing tools}
For $u,w\ge L$ put $A(u,w)=[\log(w/u)+u/w-1]/2$. Taylor expansion in $u$
gives $A(u,w)\le(u-w)^2/(4L^2)$. Direct Gaussian integration and the chain
rule give, for every permitted policy $\pi$,
\begin{align*}
\operatorname{KL}(P_0^\pi,P_1^\pi)
={}&m\left[A(\tau_0,\tau_1)+\frac{(c_1-c_0)^2}{2\tau_1}\right]
 +\E_0N_{\rm exp} A(V_0,V_1)\\
&+\frac{(\theta_1-\theta_0)^2}{2V_1}
 \E_0\sum_{\rm diagnostics}x_t^2.
\end{align*}
At a shared history, policy and stopping kernels coincide and contribute zero
conditional KL. Pad stopped experiments with parameter-independent null symbols.
Fresh Gaussian responses supply each remaining term. All $m$ observational pairs
are included; omitting recorded information can only decrease KL. Pathwise caps yield
\[
\operatorname{KL}\le
\frac{m(\tau_1-\tau_0)^2+(n+k)(V_1-V_0)^2}{4L^2}
+\frac{m(c_1-c_0)^2+k(\theta_1-\theta_0)^2}{2L}.
\]
If KL is at most $1/8$, Pinsker gives total variation at most $1/4$. For two
fixed targets with separation $g$, midpoint testing then gives maximum MSE at
least $g^2(1-\mathrm{TV})/8\ge3g^2/32$.

\subsection{Residual upper bound}
In this subsection $Q=Q_{\rm raw}=\widehat V-\widehat\tau$. For raw OLS, $\sum Z_i^2\sim\chi_m^2$ and
$(m-1)\hat\tau/\tau\sim\chi_{m-1}^2$. Therefore
\[
\E(c_{\rm raw}-c)^2=\frac{\tau}{m-2},\qquad
\E\hat\tau=\tau,\qquad \operatorname{Var}(\hat\tau)=\frac{2\tau^2}{m-1}.
\]
Clipping decreases coefficient error. Independence of fresh datasets gives
\[
\E(Q-h^2)^2=\zeta^2+\frac{2V^2}{\nu}+\frac{2\tau^2}{m-1}\le CE,
\quad \E(D-h)^2\le\frac Vk+\frac{\tau}{m-2}.
\]
Here $\nu$ is the pooled residual degree of freedom from Appendix~\ref{app:pooled}. The latter cross term vanishes because diagnostic noise is centered, even when
clipping biases $\hat c$. Thus the $W$ in Appendix~\ref{app:strength} has bounded expectation without
requiring independence of $Q,D$. A measurable global minimizer exists on the
compact interval; use a deterministic tie rule. Weighted Euclidean distance,
comparison to $h$, and the triangle inequality give the stated distance bound.

For $s=|t|+|h|$ and $d=||t|-|h||$, the target gap is at most $ps^{p-1}d$,
$d\le s$, and squared distance is at least $d^2(K+s^2/E)$. At distance at
most $z$, squared target gap is bounded by $C\max(1,z)H$, where
\[
H=\sup_{0<s\le2R}s^{2p-2}\min\{s^2,(K+s^2/E)^{-1}\}.
\]
Put $N=1/E$. If $K\le\sqrt N$, split $s^2$ at $N^{-1/2}$ to obtain
$H\le N^{-p/2}$. If $\sqrt N\le K\le N$, substitute $s^2=(K/N)v$ and use
$\sup_{v>0}v^{p-1}/(1+v)<\infty$ to get $H\le C N^{1-p}K^{p-2}$.
If $K\ge N$, $H\le(2R)^{2p-2}/K$. Taking expectations at $z=4W$ proves
the upper bound, including $k=0$ with its omitted mean term.

\subsection{Residual lower bounds}
Fix $c=0$, $\tau_0=(L+U)/2$, and write $d=h_1-h_0\ge0$,
$D_2=h_1^2-h_0^2$. For $s\ge4$ use:
\begin{center}
\begin{tabular}{llll}
\toprule
Range & $h_0$ & $d$ & Squared target gap order\\
\midrule
$k\le\sqrt s$ & $0$ & $\epsilon s^{-1/4}$ & $s^{-p/2}$\\
$\sqrt s\le k\le s$ & $\epsilon\sqrt{k/s}$ & $\epsilon/\sqrt k$ & $s^{1-p}k^{p-2}$\\
$k\ge s$ & $\epsilon$ & $\epsilon/\sqrt k$ & $1/k$\\
\bottomrule
\end{tabular}
\end{center}
Choose
\[
0<\epsilon\le\min\{R/2,\sqrt{(U-L)/3},1/\sqrt3,\sqrt{L/16}\}.
\]
Then $h_1\le2\epsilon$, $D_2\le3\epsilon^2/\sqrt s$, and
$kd^2\le\epsilon^2$. The displayed target gaps follow from $d^p$ in the
first row and the lower mean-value bound $ph_0^{p-1}d$ otherwise.

For $s=N=n+k$, keep $\tau$ fixed and $\zeta=0$. The KL is at most
$9\epsilon^4/(4L^2)+\epsilon^2/(2L)\le25/512<1/8$, since $k\le N$.
For $s=m$, set $\tau_1=\tau_0-D_2$ and both contaminations zero. Experimental
variances coincide; the KL is no larger. Baseline feasibility follows from
$D_2\le3\epsilon^2/2\le(U-L)/2$. For $s=\Delta^{-2}$ when $\Delta>0$,
keep $\tau$ fixed and set $\zeta_0=\Delta$, $\zeta_1=\Delta-D_2$.
Here $D_2\le3\epsilon^2\Delta\le\Delta$, and only diagnostic means
contribute KL. Omit this case for $\Delta=0$. The three lower bounds are
$F(1/N,k)$, $F(1/m,k)$ and $F(\Delta^2,k)$.

For fixed $k$, $F(e,k)$ is continuous, nondecreasing, and has branchwise
logarithmic slope between zero and $p/2$. Therefore their maximum is at least
$3^{-p/2}F(E,k)$, since $E\le3\max(1/N,1/m,\Delta^2)$.

For the extra floor take $r_0=R/4$, $h_0=r_0$, $h_1=r_0+d$,
$c_0=0$, $c_1=-d$, $d=b/\sqrt m$, and
$\tau_1=\tau_0-2r_0d-d^2$, with zero contamination. Let
\[
H_0=\frac1{2L}+\frac{(2r_0+1)^2}{4L^2},\quad
0<b\le\min\{1,R/4,C_0,(U-L)/(2(2r_0+1)),(8H_0)^{-1/2}\}.
\]
All parameters are feasible. Every experimental conditional law coincides,
and observational KL is at most $b^2H_0\le1/8$. The target gap is at least
$pr_0^{p-1}d$, giving risk of order $1/m$. If $k\le m$ then
$F(E,k)\ge F(1/m,k)\ge1/m$; if $k>m$, both $m,k$ lie in the last branch
because $E\ge1/m$. Thus the combined lower bound is $F(E,\min(k,m))$.

\subsection{Total strength}
Clip the diagnostic mean to $[-C_0-R,C_0+R]$ and apply $|t|^p$.
Lipschitz continuity gives risk $C/k$; the zero estimator handles $k=0$.
For a lower bound choose fixed $0<c_*\le C_0$, fixed interior $\tau$,
zero contamination, and $h_\pm=\pm r$ with
$r=\min(R/2,c_*/2,b/\sqrt k)$ for $k\ge1$, or
$r=\min(R/2,c_*/2,b)$ at zero. Cheap and observational laws coincide.
Diagnostic KL is at most $2kr^2/L\le1/8$ for small fixed $b$.
The target gap is at least $2p(c_*/2)^{p-1}r$. Since $r^2$ is bounded below
by a fixed positive multiple of $1/k$ for integer $k\ge1$, testing proves the result.

\subsection{Squared-loss gains}
For population energy, projected $Q$ has risk $CE$, and the squared clipped
corrected diagnostic mean has risk $C/K$. Evaluate the residual hard-pair table
at power two to obtain lower components $\min(1/N,1/k)$,
$\min(1/m,1/k)$ and $\min(\Delta^2,1/k)$. The nonzero compensated pair
also gives $1/m$. Their maximum is $\min(E,1/K)$ up to constants.

For the fitted target, write $e=c-\bar c(O)$. Boundedness and OLS give
$\E e^2\le C/m$ and $|A_O-h^2|\le2(R+C_0)|e|$. Thus projected $Q$
has MSE $CE$ for $A_O$, without independence. The diagnostic estimate has
MSE $C/k$ conditional on $O$ by the Lipschitz square map. Count selection
gives the minimum.

For its lower bound set $r_0=R/4$, $c_0=c_1=0$, $h_0=\sigma r_0$,
$h_1=\sigma(r_0+d)$ and $\tau_0=(L+U)/2$, with the fixed sign selected
so that $P_{c=0,\tau_0}(\sigma b\le0)\ge1/2$. The shared-transcript overlap inequality below gives
maximum risk at least $r_0^2d^2/4$ whenever KL is at most $1/8$.
For completeness this overlap argument does not need a common training law.
On the full transcript space take $\mu=P_0+P_1$ and $p_i=dP_i/d\mu$.
For the same transcript and the two target functions $t_i(O)$,
$(a-t_0)^2+(a-t_1)^2\ge(t_0-t_1)^2/2$.
Thus a gap at least $g$ on $B$ gives
$\max_i R_i\ge(g^2/4)\int_B\min(p_0,p_1)d\mu
\ge(g^2/4)[P_0(B)-\operatorname{TV}(P_0,P_1)]$.
This also covers stopped transcripts and randomized assessments.
Use, respectively: (i) $d=b/\sqrt{n+k}$ with fixed $\tau$ and zero
contamination; (ii) $d=b/\sqrt{m+k}$ with $\tau_1=\tau_0-D_2$ and zero
contamination; (iii) $d=b\min(\Delta,k^{-1/2})$ with fixed $\tau$ and
$(\zeta_0,\zeta_1)=(\Delta,\Delta-D_2)$ when $\Delta>0$.
Take $k^{-1/2}=\infty$ at zero and choose
\[
0<b\le\min\{1,R/4,(U-L)/(2(2r_0+1)),(2r_0+1)^{-1},(8H_0)^{-1/2}\}.
\]
Then $D_2\le(2r_0+1)d$, all compensations are feasible, and the KL bounds
are at most $b^2H_0$, $b^2H_0$, and $b^2/(2L)$ respectively. The lower
scales are $(n+k)^{-1}$, $(m+k)^{-1}$ and $\min(\Delta^2,1/k)$.
Their maximum is comparable to $\min(E,1/k)$ since $k\le N$.
This proves the OLS companion.

For the training-rule extension, let $e=c-b_m(O)$. Then
$|e|\le C_0+B$, $\E e^2\le A_{\rm ref}/m$, and
\[
 |(h+e)^2-h^2|\le(2R+C_0+B)|e|.
\]
The same $Q$, still using the observational OLS residual variance regardless
of the reference training rule, therefore has risk $CE$ after projection onto
$[0,(C_0+R+B)^2]$. The diagnostic square map is
$2(C_0+R+B)$-Lipschitz, giving risk $C/k$.
For the lower bound choose $\sigma$ with
$P_{c=0,\tau_0}(\sigma b_m(O)\le0)\ge1/2$; the two half-line
probabilities sum to at least one, including any atom at zero. The sign is
chosen before observing data. The observational law does not depend on $h$,
and at a shared $o$ the target gap is
$d[2r_0+d-2\sigma b_m(o)]\ge2r_0d$ on this event.
All three pairs above have unchanged squared mean and variance differences,
so their KL and overlap bounds apply. Only the event probability under
$P_0$ is needed even when $\tau_1\ne\tau_0$.
For example, clipped $c_{\rm raw}+\lambda/\sqrt m$ with fixed
$\lambda\ne0$ need not be symmetric but satisfies the assumptions:
its coefficient MSE is at most $(32U/15+2\lambda^2)/m$ for $m\ge32$.
This extension does not cover references trained on the diagnostic outcomes
or assert an upper bound without the stated accuracy and boundedness.

These are sufficient, not necessary, assumptions. For any bounded reference,
the diagnostic upper bound and the preceding lower bound match at $1/k$
when $k\ge E^{-1}$, without coefficient accuracy. But boundedness alone
does not give the rate for all budgets: the zero reference targets $(c+h)^2$
and has risk $\asymp\min(1,1/k)$. To see its lower bound, fix
$c_*=C_0/2$, interior $\tau$, zero contamination and compare $h=\pm r$,
where $r=b/\sqrt{k}$ for $k\ge1$ and $r=b$ at zero, with fixed
$0<b\le\min(R/2,c_*/2,\sqrt L/4)$. Observational and cheap laws coincide,
adaptive KL is at most $2kr^2/L\le1/8$, and the target gap is $4c_*r$.
Testing proves the claim; a squared clipped diagnostic mean gives its upper
bound. Thus at $k=0$, $\Delta=0$, $n=m\to\infty$, a bounded inaccurate
reference can have constant risk while an accurate one has order $1/m$.

%% file: appendix/learned.tex
\section{Conditional gain of a learned repair}\label{app:learned}
Fix the direct-design experiment and the bounded, accurate observational reference
$b=b_m(O)$. Let $M$ be the diagnostic mean and train
$a=\operatorname{clip}(M,[-T,T])$, $T=C_0+R$, with $k\ge1$.
For a fresh evaluation assignment uniform on $\{-1,1\}$ and response with
conditional mean $\theta Z_*$ and variance $V$, the conditional population gain is
\[
 G_q(O,D)=(\theta-b)^2-(\theta-a)^2=A_q(O)-(\theta-a)^2.
\]
This is the gain of this particular trained repair, averaged over fresh evaluation
data only. It can be negative. It is neither an individual evaluation loss nor
a training-averaged algorithm risk. The learner is fixed before optimizing its
assessor; no optimization over adaptive repair-training policies is asserted.

\begin{theorem}
Let $R_A^{\rm fix},R_G^{\rm fix}$ denote minimax squared assessment risks in
the same fixed direct-design experiment. With $V_*=U+R^2+1/2$,
\[
 \left|\sqrt{R_G^{\rm fix}}-\sqrt{R_A^{\rm fix}}\right|
       \le \frac{\sqrt3 V_*}{k}.
\]
Consequently, when $k^2r\to\infty$, where $r=\min(E_N,1/k)$, $R_G^{\rm fix}/R_A^{\rm fix}\to1$.
Moreover, for $m\ge32$, $n+k\ge16$ with $k\ge1$,
\[
                 R_G^{\rm fix}\asymp r+k^{-2}.
\]
In particular, if $kE\ge1$ and $k\to\infty$, the learned-gain assessment
rate is $1/k$. Along $\Delta=0,n=k=m^2$, it is $m^{-2}$, compared with
$m^{-1}$ for the population-reference target $h^2$.
\end{theorem}
\begin{proof}
Projection cannot increase distance to $\theta\in[-T,T]$, so
\[
 \sup_q\E(A-G)^2\le\sup_q\E(M-\theta)^4
       \le 3V_*^2/k^2.
\]
For any assessor, the $L_2$ triangle inequality, followed by the parameter
supremum and assessor infimum, bounds the difference of root minimax risks by
this target distance. Reverse the two targets for the other direction.
The oracle-potential lower theorem applies to every policy, hence to this fixed
design, while its upper estimator already uses this design. Thus
$R_A^{\rm fix}\asymp r$. The asserted ratio follows when the target distance
is $o(\sqrt r)$. The squared triangle inequality also gives the all-budget
upper bound $R_G^{\rm fix}\le 2Cr+6V_*^2/k^2$.

For the additional lower bound choose fixed positive
$a_0\le\min(R/4,T/4,\sqrt L/8)$ and
$t_0\le\min(T/4,\sqrt L/16)$. Compare
$c=0,\tau=(L+U)/2,\zeta=0,h_\pm=\pm a_0/\sqrt k$.
The cheap and observational laws coincide and $V$ agrees, while
diagnostic KL is $2a_0^2/V\le1/32$, hence TV is at most $1/8$.
At a shared transcript, $G_+-G_-=4a_0(a-b)/\sqrt k$.

Conditional on $O$ under $P_+$, the interior density of $a$ is bounded by
$\sqrt{k/(2\pi V)}$. Each clipping atom has mass at most $1/2$, since
the mean is inside the clipping interval. An interval of radius $t_0/\sqrt k$
cannot contain both atoms. Uniformly in $b$,
\[
 P_+(|a-b|<t_0/\sqrt k\mid O)
 \le\frac12+\frac{2t_0}{\sqrt{2\pi L}}<\frac58.
\]
Thus on an event of probability at least $3/8$ the absolute target gap is
at least $g=4a_0t_0/k$. The shared-transcript overlap inequality applies also
to this diagnostic-dependent event, giving risk at least
$(g^2/4)(3/8-1/8)=a_0^2t_0^2/k^2$.

Finally let $R_A^{\rm fix}\ge cr$ and $\delta=\sqrt3 V_*/k$.
If $cr\ge4\delta^2$, root-risk transfer gives $R_G^{\rm fix}\ge cr/4$;
combine with the new floor. Otherwise $r<12V_*^2/(ck^2)$, so the
floor alone bounds $r+k^{-2}$ up to constants. This proves the full rate.
\end{proof}
The transfer alone is elementary and cannot establish the additional floor;
the sign-pair argument is essential to the all-budget statement. The floor is
minimax, not pointwise. Along $\Delta=0,n=m=k^3$, oracle-potential assessment
has rate $k^{-3}$ whereas this learned-gain assessment has rate $k^{-2}$.
Theorem~\ref{thm:quality} addresses joint learner selection under a quality constraint.

\paragraph{Two assessors and classical precedent.}
The plug-in $\widehat G=(a-b)^2$ satisfies
$\widehat G-G=2(a-b)(a-\theta)$, giving MSE at most
$4(T+B)^2V_*/k$ when $|b|\le B$. Its nonnegativity does not certify
that the true gain is positive. Data-dependent loss assessment and the
distinction from average risk have established theory \citep{fourdrinier}.
A classical Stein baseline, using the independent pooled variance from
Appendix~\ref{app:pooled} with $\nu=n+k-1>0$, is
\[
 U_G=2(a-b)M+b^2-a^2-\frac{2\widehat V}{k}{\bf1}\{|M|<T\}.
\]
The weak derivative of clipping is the displayed indicator, apart from
probability-zero endpoints. Gaussian integration by parts gives
$\E[(a-b)(M-\theta)\mid O]=(V/k)\E[{\bf1}\{|M|<T\}\mid O]$.
Independence and $\E\widehat V=V$ imply $\E[U_G-G\mid O]=0$.
This expectation still averages over repair training data, not just fresh
evaluation data; it is no conditional certificate for the observed repair.
Without clipping, the correction is $2\widehat V/k$, twice the
oracle-potential correction. Negative gain estimates are retained.

\paragraph{Numerical implementation.}
The implementation of the clipped-mean learner and the Stein assessor is
\path{experiments/core.py}. The target-scale and clipping-boundary suites in
Appendix~\ref{app:experiments} recompute the realized target on every draw;
negative gain estimates are retained.

%% file: appendix/cost.tex
\section{Acquisition cost for the three assessment targets}\label{app:cost}
Fix $\Delta=0$ and the unrestricted-budget direct experiment of Appendix~\ref{app:pooled},
with $m\ge32$, $N=n+k\ge16$. Fix a reference-rule family satisfying the
same uniform boundedness and accuracy assumptions for every $m$.
Define $C_J(\epsilon)$ as the infimum of $wm+n+\lambda k$ over admissible
integer budgets whose minimax squared assessment risk for target $J$ is at
most $\epsilon$. For $G$, keep the clipped-mean repair fixed and require
$k\ge1$. All constants below depend only on model/reference bounds, not on
positive prices $w,\lambda$, which may vary with $\epsilon$.

\paragraph{Corollary and proof.}
For sufficiently small $\epsilon$, the target-specific costs are
\begin{align*}
 C_H(\epsilon)&\asymp \frac{w+\min\{1,\lambda\}}{\epsilon},\\
 C_A(\epsilon)&\asymp \min\left\{\frac{w+1}{\epsilon},\;w+\frac{\lambda}{\epsilon}\right\},\\
 C_G(\epsilon)&\asymp \min\left\{\frac{w+1}{\epsilon}+\frac{\lambda}{\sqrt\epsilon},\;w+\frac{\lambda}{\epsilon}\right\}.
\end{align*}
Write $E_N=N^{-1}+m^{-1}$ and use $1/0=\infty$.
The population rate $\min(E_N,1/\min(k,m))$ forces
$m\gtrsim\epsilon^{-1}$ and $N\gtrsim\epsilon^{-1}$ on either route.
Thus its cost is at least a constant times
$(w+\min(1,\lambda))/\epsilon$. Buying the required observational samples
and then only the cheaper experimental type attains this order.

For fitted oracle potential, $\min(E_N,1/k)\lesssim\epsilon$ forces
either $m,N\gtrsim\epsilon^{-1}$ or $k\gtrsim\epsilon^{-1}$.
The second route costs at least order $w+\lambda/\epsilon$.
If $\lambda\ge1$, the first costs at least order $(w+1)/\epsilon$.
If $\lambda<1$, it instead costs at least order
$(w+\lambda)/\epsilon\ge w+\lambda/\epsilon$ for $\epsilon\le1$.
The minimum of the two stated costs is therefore a lower bound in order.
It is attained by either $m,n\asymp\epsilon^{-1},k=0$, or fixed admissible
$m$, $n=0$, $k\asymp\epsilon^{-1}$.

For learned gain, the extra $k^{-2}$ term additionally requires
$k\gtrsim\epsilon^{-1/2}$. The diagnostic route is unchanged.
On the variance route when $\lambda\ge1$,
$n+\lambda k=N+(\lambda-1)k$ is at least both $N$ and $\lambda k$,
and hence at least half their sum. This gives the lower cost
$(w+1)/\epsilon+\lambda/\sqrt\epsilon$ in order.
For $\lambda<1$, the fitted-potential lower argument already bounds the
diagnostic route's cost. Upper constructions use either
$m,n\asymp\epsilon^{-1}, k\asymp\epsilon^{-1/2}$ or the same diagnostic
route. Fixed minimum counts and rounding change only uniform constants.

This corollary optimizes budgets for assessing specified targets, not the repair
learner itself. Choosing never to change the reference would make its gain
identically zero and create a different, trivial assessment problem.
The clean-model restriction is essential; an unknown nonzero variance offset
can obstruct the cheap route at small tolerances. The formulas are order
statements rather than exact switching prices.

%% file: appendix/relaxation.tex
\section{A logarithmic allowance for easier gain assessment}\label{app:relaxation}
This appendix keeps the unknown-parameter model and fitted reference of
Appendix~\ref{app:quality}, but uses fixed direct sampling and a variable
learning-MSE tolerance $\eta$. Let $m\ge32$, $n\ge0$, $k\ge16$,
$N=n+k$, $E_N=\Delta^2+N^{-1}+m^{-1}$ and
\[
 \mathcal R(\eta)=\inf_{f:\sup_q\E_q(a_f-\theta)^2\le\eta}
                 \inf_g\sup_q\E_q(g-G_f)^2.
\]
All learners may use the full training transcript and parameter-independent
randomization. No oracle magnitude is provided to the construction.

\begin{theorem}\label{thm:logquality}
Fix $\beta>2$, $\Delta=0$, and $m=n=\lceil k^\beta\rceil$.
If $\mathcal R(\eta_k)=O(k^{-\beta})$, then
$\eta_k=\Omega(\log(k)/k)$. Conversely there is an explicit
$\eta_k\asymp\log(k)/k$ with
$\mathcal R(\eta_k)\asymp k^{-\beta}$.
Constants may depend on $\beta$ and fixed model/reference bounds.
At every fixed $C_2\ge V_*$,
Theorem~\ref{thm:quality} instead gives
$\mathcal R(C_2/k)\asymp k^{-2}$ on this sequence.
\end{theorem}

\subsection{Selective upper construction}
Use the raw OLS residual variance $\widehat\tau$ and pooled experimental
variance $\widehat V$ of Appendix~\ref{app:pooled}. For deterministic $t^2>0$ put
\[
 Q=\operatorname{clip}(\widehat V-\widehat\tau,[0,R^2]),\quad
 u=\sqrt Q,\quad I=\mathbf1\{Q\ge t^2\}.
\]
\[
 s=\begin{cases}1&M\ge b,\\-1&M<b,\end{cases}
 \qquad a=b+Isu,\quad g=IQ.
\]
Here $M$ is the direct diagnostic mean; the report assesses this learner's
actual gain, not oracle potential. With $C_Q=\max(1,4V_*^2)$,
independence of the two variance estimators and projection give
\[
 \E(Q-h^2)^2\le\Delta^2+4V_*^2/N+4U^2/m\le C_QE_N.
\]
The following bounds are uniform over the full parameter class:
\begin{align}
 \E(a-\theta)^2&\le
 2t^2+6\sqrt{C_QE_N}+66A_{\rm ref}/m+128V_*/k,
 \label{eq:selectivequality}\\
 \E(g-G_f)^2&\le300C_QE_N+204R^2A_{\rm ref}/m
 +\frac{73728V_*^2}{k^2}
           \exp\{-kt^2/(24V_*)\}.\label{eq:selectivegain}
\end{align}

For completeness write $e=c-b$, $\xi=M-\theta$ and $\delta=Q-h^2$.
Let $W$ denote an incorrect sign choice; terms multiplied by $h^2$ vanish
at $h=0$. Then
\[
 W\subset\{|e|\ge |h|/2\}\cup
           \{\operatorname{sign}(h)\xi\le-|h|/2\},\qquad
 h^2P(W)\le4A_{\rm ref}/m+8V_*/k.
\]
The last inequality uses Markov and a one-sided Gaussian tail; it does not
factor either event from the data-dependent magnitude estimate.
Since $(u-|h|)^2\le|\delta|$ and, on $I=0$,
$h^2\le t^2+|\delta|$, the pointwise inequality
\[
 (a-\theta)^2\le2e^2+2h^2(1-I)
             +4I(u-|h|)^2+16Ih^2\mathbf1_W
\]
proves~\eqref{eq:selectivequality}.

On $I=1$, $g-G_f=2u[u-s(h+e)]$. Factoring
$\delta=(u-|h|)(u+|h|)$ gives
\[
 (g-G_f)^2\le12\delta^2+48IQh^2\mathbf1_W+12R^2e^2.
\]
The reference-error part of $W$ contributes at most
$192R^2\E e^2$. On $Q>3h^2/2$, $Qh^2\le6\delta^2$;
this contributes at most $288\E\delta^2$. On its complement,
$Qh^2\le3h^4/2$, and activation implies $h^2\ge2t^2/3$.
The remaining contribution is at most
\[
 72h^4e^{-kh^2/(8V_*)}\mathbf1\{h^2\ge2t^2/3\}
 \le\frac{73728V_*^2}{k^2}e^{-kt^2/(24V_*)},
\]
using $x^2e^{-x/2}\le16$. This proves~\eqref{eq:selectivegain}
without assuming independence of $Q$ and $b$.

\subsection{A lower bound retaining dependence on the allowance}
Set $x=k\eta+kA_{\rm ref}/m+1$ and $r^2=64x/k$.
Whenever $r^2\le R^2$, a bound uniform over all eligible learners is
\begin{equation}
 \mathcal R(\eta)\ge\frac{256x^2}{k^2}
 \exp\{-128x/L-64\sqrt{V_*x}/L\}.\label{eq:qualityexponential}
\end{equation}
To prove it take the legal pair $c=0$, $\tau=(L+U)/2$, $\zeta=0$,
$h=\pm r$. Its common variance is $V=\tau+r^2$ and only $M$ informs
the sign. The equal-prior Bayes gain risk for any chosen learner is
$\E[4r^2(a-b)^2\operatorname{sech}^2(krM/V)]$.
The pointwise conditional square identity in Appendix~\ref{app:quality} justifies this
even without a globally square-integrable gain target.

Under the positive sign the event
\[
 B=\{|a-r|\le r/2,\ |b|\le r/4,
              \ |M-r|\le4\sqrt{V_*/k}\}
\]
has probability at least $5/8$: the complementary probabilities are at
most $4\eta/r^2\le1/16$, $16A_{\rm ref}/(mr^2)\le1/4$ and $1/16$.
No independence of events is used. Its mixture probability is at least
$1/4$. On $B$, $a-b\ge r/4$ and
$|krM/V|\le(kr^2+4r\sqrt{kV_*})/L$.
Using $\operatorname{sech}^2(z)\ge e^{-2|z|}$ yields
\[
 \mathcal R(\eta)\ge\frac{r^4}{16}
       e^{-2kr^2/L-8r\sqrt{kV_*}/L},
\]
which is~\eqref{eq:qualityexponential}. The pair depends on declared
$\eta$ and counts, not observed data; no small-TV claim is needed.

\subsection{Proof of Theorem~\ref{thm:logquality}}
Suppose $\mathcal R(\eta_k)\le Dk^{-\beta}$ eventually. If
$\eta_k\le L(\beta-2)\log(k)/(512k)$ on an unbounded subsequence,
then $x\le L(\beta-2)\log(k)/512+1+o(1)$. The pair above is feasible
eventually, its linear exponent is at most
$(\beta-2)\log(k)/4+O(1)$, and its square-root exponent is $o(\log k)$.
Since $x\ge1$,~\eqref{eq:qualityexponential} is eventually at least
$256k^{-2-(\beta-2)/2}$ on this subsequence, contradicting the fixed $D$.
This proves necessity.

For sufficiency take $t_k^2=24V_*(\beta-2)\log(k)/k$ and define
$\eta_k$ as the right side of~\eqref{eq:selectivequality}.
It is $\asymp\log(k)/k$ and~\eqref{eq:selectivegain} is $O(k^{-\beta})$.
For a matching lower bound, the first three-pair argument in Appendix~\ref{app:quality}
works uniformly for any $\eta\le R^2/512$: its null event probability is
at least $1/2-4\eta/(R/4)^2\ge3/8$. It therefore gives
$\mathcal R(\eta)\ge c\min(E_N,1/k)$ independently of $\eta$.
Our $\eta_k$ eventually satisfies this condition, completing the proof.

This proves an order threshold for each fixed $\beta>2$, not a sharp constant,
a complete frontier, or beneficial deployment. The next result extends the
budget sequences and sampling scope explicitly.
The exponential comparison uses constrained-risk reasoning \citep{duchiruan}; Appendix~\ref{app:sharp} identifies the leading constant.

\subsection{Information-ratio threshold and adaptive extension}
For varying known bounds $\Delta_k$ and budgets define
\[
 E_k=\Delta_k^2+(n_k+k)^{-1}+m_k^{-1},\qquad
 \ell_k=\log\{1/(k^2E_k)\}.
\]
\begin{theorem}\label{thm:informationquality}
Suppose $k^2E_k\to0$ and $\ell_k=o(k)$. Assessment risk $O(E_k)$
requires a uniform learning-MSE tolerance $\eta_k=\Omega(\ell_k/k)$.
Conversely the selective construction attains joint minimax assessment
risk of order $E_k$ at an explicit tolerance
$\eta_k\sim48V_*\ell_k/k$. Both conclusions hold also for adaptive
sampling with pathwise caps and the model of Appendix~\ref{app:quality}.
\end{theorem}
\begin{proof}
For fixed sampling, $kA_{\rm ref}/m_k\le kA_{\rm ref}E_k=o(1)$.
If $\eta_k\le L\ell_k/(512k)$ along an unbounded subsequence, then
the $x$ in~\eqref{eq:qualityexponential} is at most
$L\ell_k/512+1+o(1)$. Its sign pair is feasible eventually by
$\ell_k=o(k)$, and its exponent is at most $\ell_k/2$ eventually.
The lower bound divided by $E_k=k^{-2}e^{-\ell_k}$ is at least
$256e^{\ell_k/2}$, ruling out assessment $O(E_k)$.

For sufficiency use $t_k^2=24V_*\ell_k/k$ and the tolerance in
\eqref{eq:selectivequality}. Since $\ell_k\to\infty$ and
$\sqrt{E_k}=o(1/k)$, that tolerance is asymptotic to
$48V_*\ell_k/k$. Equation~\eqref{eq:selectivegain} is $O(E_k)$.
The tolerance tends to zero, so the preceding uniform three-pair argument
gives the matching lower bound $cE_k$.

For adaptive sampling use the exact simulator of Appendix~\ref{app:simulation}, retaining
the original learner and assessor on the simulated transcript and the
unchanged observational reference. Equality of laws preserves both the
arbitrary $\eta$ constraint and the random gain target, giving
\[
 \mathcal R_{\rm fix}(m,n+1,k;\eta)
 \le\mathcal R_{\rm ad}(m,n,k;\eta)
 \le\mathcal R_{\rm fix}(m,n,k;\eta).
\]
For $N=n+k$, the source precision $E'_k$ satisfies
$NE_k/(N+1)\le E'_k\le E_k$, and
$0\le\ell'_k-\ell_k\le\log(1+1/N)$. Thus the fixed-source necessity
applies, while fixed selective sampling supplies the adaptive upper bound.
The source three-pair bound also supplies the matching adaptive lower bound.
No bounded-loss approximation or fourth-moment assumption on all learners
is introduced.
\end{proof}

The construction's actual worst-case learning MSE has the same order as
its allowance. At a legal local point $c=0$, $h=t_k/2$, interior $\tau$
and $\zeta=0$, Markov bounds yield
\[
 \sup_q\E(a-\theta)^2\ge\frac{t_k^2}{16}
 \left[1-\frac{16C_QE_k}{9t_k^4}
          -\frac{16A_{\rm ref}}{m_kt_k^2}\right]_+.
\]
Indeed activation requires $Q-h^2\ge3t_k^2/4$; on inactivity and
$|c-b|\le t_k/4$, the learning error is at least $t_k/4$.
The bracket tends to one. This is a worst-case local sequence, not poor
learning at every fixed parameter.

For example, clean auxiliary counts $m=n\asymp k^2\log k$ give required
allowance $\log\log(k)/k$, whereas $m=n\asymp k^2e^{\sqrt k}$ give
allowance $k^{-1/2}$. Each targets its corresponding $E_k$; the assessment
requirements differ. The theorem characterizes risk orders; Theorem~\ref{thm:sharp} resolves the leading exponential constant.

\paragraph{Intermediate risks on a logarithmic scale.}
More generally, if $k^2E_k\to0$, $\eta_k\to0$ and $z_k=k\eta_k\to\infty$,
then for either sampling class
\[
 \log\{1/[k^2\mathcal R(\eta_k)]\}
       \asymp\min(\ell_k,z_k).
\]
For an upper bound choose $t_k^2=\eta_k/4$ in the selective rule.
Its learning bound is $\eta_k/2+o(\eta_k)$, hence feasible at the
requested tolerance, and its assessment bound is
$CE_k+Ck^{-2}e^{-z_k/(96V_*)}$.
For a lower bound retain $\mathcal R(\eta_k)\ge cE_k$ and use
\eqref{eq:qualityexponential} with $x_k=z_k+1+o(1)\le2z_k$ eventually.
The pair is legal since $x_k/k\to0$, and that bound is at least
$k^{-2}e^{-Az_k}$ for a fixed $A>0$.
Combining the two lower bounds and taking logarithms proves the claim;
the adaptive source changes $E_k$ by at most a factor two.
This does not require $\ell_k=o(k)$. It matches logarithms only:
different constants inside the exponent leave a potentially unbounded
multiplicative risk gap. This particular magnitude-only argument does not
identify the leading constant. Appendix~\ref{app:sharp} strengthens the
leading-log statement using a different, diagnostic-dependent rule; even
that refinement does not establish multiplicative risk equivalence.

\paragraph{Boundary for consistent learning.}
Under $k^2E_k\to0$, a vanishing tolerance $\eta_k\to0$ can yield
assessment $O(E_k)$ if and only if $\ell_k=o(k)$, for either sampling class
above. Sufficiency is Theorem~\ref{thm:informationquality}. For necessity,
$x_k=k\eta_k+kA_{\rm ref}/m_k+1$ obeys $x_k\ge1$ and $x_k/k\to0$,
so its sign pair is eventually feasible. Equation~\eqref{eq:qualityexponential}
gives
\[
 \log\frac{\mathcal R(\eta_k)}{E_k}
 \ge\log256+2\log x_k+\ell_k
       -128x_k/L-64\sqrt{V_*x_k}/L.
\]
The negative terms are $o(k)$. If $\ell_k\not=o(k)$, a subsequence
has $\ell_k\ge ak$ for some $a>0$, contradicting bounded risk ratio.
The same source sandwich proves adaptive necessity. Thus, for example,
clean $m=n\asymp e^{ak}$ precludes simultaneous uniform learning
consistency and assessment at full auxiliary precision. Slower assessment
and nonvanishing learning allowances are not ruled out; their full frontier,
and expected-only sampling budgets, remain outside these results.

\subsection{Relation to assessing the learned loss}
Fix the direct-design experiment and a nonempty common learner class
$\mathcal F_\eta$. Write $A=(\theta-b)^2$, $L_f=(\theta-a_f)^2$ and
$G_f=A-L_f$. Let $\mathcal R_G$ and $\mathcal R_L$ be the respective
joint minimax squared assessment risks, with the same quality constraint.
The pooled fitted-potential estimator obeys
$\sup_q\E(\widehat A-A)^2\le\epsilon_A^2
=C\min(E_N,1/k)$ independently of the choice of learner.
For any gain assessor $g$ use the loss assessor $l=\widehat A-g$.
Pointwise, $l-L_f=(\widehat A-A)-(g-G_f)$. Minkowski's inequality,
then the supremum over parameters and infima over assessors and learners,
give $\sqrt{\mathcal R_L}\le\epsilon_A+\sqrt{\mathcal R_G}$.
Near-minimizers suffice; attainment is not assumed. Reversing the
transformation gives, when the infima are finite,
\[
 |\sqrt{\mathcal R_G}-\sqrt{\mathcal R_L}|\le\epsilon_A.
\]
Only assessment errors need second moments; neither target is assumed
square-integrable for every feasible learner. A finite-risk choice on
one side and the conversion ensure finiteness on the other. No independence
between $\widehat A$, the learner and the assessor is needed.

If $\mathcal R_G/E_N\to\infty$, this inequality yields
$\mathcal R_L/\mathcal R_G\to1$. In particular this holds in the
log-scale regime above when $\ell_k/(k\eta_k)\to\infty$, since the
gain lower bound divided by $E_k$ diverges. Thus this portion of the frontier also characterizes assessment of the learned loss. At $\mathcal R_G=O(E_N)$, the transfer gives only
$\mathcal R_L=O(E_N)$, not a matching loss lower bound.
Nor does it grant the requisite $\widehat A$ to an arbitrary adaptive
transcript; this target-transfer statement is fixed-design only.

\subsection{Fixed contamination can make more data harmful for this rule}
Choose $0<\gamma<\min(R^2,1/2)$ and the legal point
$c=h=\theta=0$, $\tau=(L+U)/2$, $\zeta=\gamma$.
If $m,n+k\to\infty$, $k\ge16$ and $t_k^2\to0$, the selective rule satisfies
\[
 \E a^2\to\gamma,\qquad \E G\to-\gamma,\qquad
 \E(g-G)^2\to4\gamma^2,\qquad \Pr(G<0)\to1.
\]
Indeed the pooled variance and observational residual estimates converge
in $L^2$ to $\tau+\gamma$ and $\tau$. Thus bounded $Q$ tends to $\gamma$
in every finite $L^p$, activation $I$ tends to one in probability, and
$b\to0$ in $L^2$. For either diagnostic sign $s$,
\[
 a^2=b^2+2bIs\sqrt Q+IQ,\quad
 G=-2bIs\sqrt Q-IQ,\quad g-G=2IQ+2bIs\sqrt Q.
\]
The cross term vanishes in $L^2$ since $Q\le R^2$; these identities imply
the asserted limits without independence or sign consistency.
This is failure of this construction, not impossibility of robust learning:
the direct mean has learning MSE $(\tau+\gamma)/k$.
It lies outside the rich-information theorem because $E_k\ge\gamma^2$.

The paired robustness suite in Appendix~\ref{app:experiments} includes zero-effect contamination cases and a diagnostic fallback.

%% file: appendix/robust.tex
\section{Robust learning and projected assessment}\label{app:robust}
This extension retains fixed direct sampling and all model/reference assumptions.
It does not extend to arbitrary adaptive transcripts. Constants below depend
only on the fixed model/reference bounds. Write $V_*=U+R^2+1/2$, $T=C_0+R$,
and $E_N=\Delta^2+(n+k)^{-1}+m^{-1}$, with $m\ge32$, $k\ge16$, $n\ge0$.

For a declared $\eta\ge4V_*/k$, set $w=\sqrt\eta/2$ and
$t^2=\eta/64$. Construct $a_0=b+I s\sqrt Q$ and $g_0=IQ$ as in
Appendix~\ref{app:relaxation} at this threshold. On $F=\{|a_0-M|>w\}$ use $a=M$;
otherwise use $a=a_0$. The mean fallback is not clipped.
The retained report $g_0$ initially assesses the fallback-learner target
$G_a=(\theta-b)^2-(\theta-a)^2$.

For report projection, set $u=\sqrt{12V_*\log(k)/k}$ and
$J_\theta=[M-u,M+u]\cap[-T,T]$. If empty, replace $J_\theta$ by
$[-T,T]$ without changing the learner. Project $g_0$ onto the image of
$J_\theta$ under $z\mapsto2(a-b)z+b^2-a^2$ to obtain $g_{\rm p}$.
Sort image endpoints when $a-b<0$.

\begin{theorem}\label{thm:robustfallback}
The learner satisfies $\sup_{\zeta\le1/2}\E(a-\theta)^2\le\eta$,
where the supremum also ranges over all other allowed parameters.
For each subclass $\zeta\le\Delta$,
\[
 \sup_q\E(g_0-G_a)^2\le CE_N+Ck^{-2}e^{-k\eta/(2048V_*)},
\]
and
\[
 \sup_q\E(g_{\rm p}-G_a)^2
 \le C\min\{E_N+k^{-2}e^{-k\eta/(2048V_*)},\log(k)/k\}+Ck^{-3}.
\]
Neither rule uses the subclass bound $\Delta$ as an input.
\end{theorem}
\begin{proof}
Pathwise $|a-M|\le w$, so Minkowski gives
$\|a-\theta\|_2\le w+\sqrt{V_*/k}\le\sqrt\eta$ globally.
For assessment put $e=c-b$, $\delta=Q-h^2$, $v=a_0-\theta$ and
$\xi=M-\theta$. We first verify
\[
 \E[v^4 1_{\{|v|>w/2\}}]
 \le CE_N+Ck^{-2}e^{-kt^2/(32V_*)}. \tag{*}
\]
On activation, $|v|\le|e|+|\sqrt Q-|h||+2|h|1_W$,
where $W$ is wrong diagnostic sign. Bounded reference error gives
$\E e^4\le C/m$, and $(\sqrt Q-|h|)^4\le\delta^2$.
If $h^2<t^2/2$, activation forces $\delta>h^2$, hence $Ih^4\le\delta^2$.
Otherwise $h^2\ge t^2/2$ and the sign union bound from Appendix~\ref{app:relaxation} gives
$h^4\Pr(W)\le4R^2A_{\rm ref}/m+h^4e^{-kh^2/(8V_*)}$.
Absorbing the polynomial in half this exponential bounds its last term by
$Ck^{-2}e^{-kt^2/(32V_*)}$. No activation probability is factored out.

On inactivity $v=-h-e$. If $|e|>|h|/2$, then $v^4\le81e^4$.
Otherwise $|v|\le3|h|/2$; imposing $|v|>w/2$ gives $|h|>w/3$.
Since $Q<t^2=w^2/16$, we have $Q<9h^2/16$ and
$|\delta|>7h^2/16$, so $v^4\le C\delta^2$.
Together with $\E\delta^2\le C_QE_N$, this proves (*).

The gain change is $(G_a-G_{a_0})=(v^2-\xi^2)1_F$.
Writing $B=\{|v|>w/2\}$ and $H=\{|\xi|>w/2\}$, elementary splitting yields
\[
 (v^2-\xi^2)^2 1_F\le4v^4 1_B+2\xi^4 1_H+(w^4/8)1_H.
\]
Indeed $F\setminus B\subset H$, and on $B\setminus H$,
$\xi^4\le v^4$. Gaussian exponential moments imply
$\E[\xi^4 1_H]+w^4\Pr(H)\le Ck^{-2}e^{-kw^2/(16V_*)}$.
For the fourth moment, use
$\E[Z^4 1_{|Z|>z}]\le\E[Z^4e^{Z^2/4}]e^{-z^2/4}$;
for the probability term absorb its polynomial into half of the tail exponent.
Combining (*) with the magnitude-only report bound in Appendix~\ref{app:relaxation} and
$(g_0-G_a)^2\le2(g_0-G_{a_0})^2+2(G_{a_0}-G_a)^2$
proves the first assessment bound, since $kt^2/(32V_*)=k\eta/(2048V_*)$.

Finally let $H_u=\{|M-\theta|\le u\}$, whose complement has probability
at most $2k^{-6}$. On $H_u$ the gain image contains the actual random target.
Projection reduces absolute error, and its diameter is at most $4u|a-b|$.
For this specific learner, $a$ is either bounded $a_0$ or Gaussian $M$;
thus $\E|a|^8\le C$ uniformly, without independence of the fallback event.
Every projected endpoint and the target is bounded in magnitude by
$b^2+a^2+2T|a-b|$. Therefore
$\E|g_{\rm p}-G_a|^4\le C$, and Cauchy--Schwarz bounds the contribution
of $H_u^c$ to squared risk by $Ck^{-3}$. On $H_u$, separately use error
reduction and the diameter bound with $\E(a-b)^2\le C$.
Taking the smaller resulting bound proves the claim.
\end{proof}

For clean $m=n=k^3$ choose $\eta=2048V_*\log(k)/k$. Eventually it
meets the eligibility condition, gives global learning consistency and
clean assessment $O(k^{-3})$. Under fixed contamination, projected
assessment is $O(\log(k)/k)$. This is not sharp adaptation over all subclasses:
the $k^{-3}$ confidence remainder also precludes claiming arbitrary finer
precision with this particular interval. No statement guarantees positive gain.

\paragraph{Numerical implementation.}
The paired robustness suite uses the same fallback learner before and after
report projection. Appendix~\ref{app:experiments} records the full grid,
including delayed fallback, clean near-threshold points, and comparisons to a
separate direct-mean pipeline. A favorable sampled projection does not verify
a uniform risk guarantee or a tiny noncoverage probability.

%% file: appendix/experiments.tex
\section{Experimental protocol and complete baseline comparisons}\label{app:experiments}
\subsection{Design, provenance, and what is estimated}
The released code generates all figures and numerical entries in this manuscript. The original seven suites contain 241 independently seeded cells and 4,340,000 sufficient-statistic draws. Two additional suites contribute 1,920 finite-budget configurations and 42 variance-bound configurations, each with 12,000 nuisance/sufficient-statistic draws (23,544,000 in total); their protocol is in Appendix~\ref{app:finite}. Applying multiple methods to a cell produces 481 method records; those records are not independent datasets. The master seed for the original seven suites is 20260923. A cell with suite number $s$ and zero-based cell index $i$ uses NumPy \texttt{SeedSequence([20260923,s,i])}. Methods within a cell share draws, allowing paired differences. No seed or hyperparameter is selected by minimizing the reported errors.

Unless a suite specifies otherwise, $R=C_0=1$, $L=1.5$, $U=2.5$, $\tau=2$, $c=.3$, and $V_*=4$. For the original seven suites, the observational reference is OLS clipped to $[-1,1]$; its residual variance is computed with the raw OLS fit. Direct-mean repairs are clipped to $[-2,2]$, except the explicitly un-clipped robust fallback. Every target uses the ground-truth parameter and the \emph{same realized reference and learner} as its competing report. Unknown parameters are used only to generate data and score risk, never as rule inputs.

The primitives in Appendix~\ref{app:pooled} reproduce the exact joint distribution of raw-data sufficient statistics, subject to floating-point arithmetic. This avoids materializing $k^3$-length arrays. The largest degrees of freedom are about $4.72\times10^{21}$, for which chi-square draws and arithmetic use double precision; the simulation is distributionally exact as a mathematical construction, not arbitrary-precision computation. Numerical tests check the finite-degree moments and independence on smaller budgets. Far-tail diagnostic probabilities are treated separately in logarithmic arithmetic below.

A reported MSE is the mean of per-draw squared errors, with Monte Carlo SE equal to the sample standard deviation divided by the square root of the replicate count. Paired SEs apply this formula to same-draw error differences. They quantify finite simulation uncertainty, not uncertainty over a minimax supremum. The displayed grids are not a search over all parameters. Zeros in ordinary Monte Carlo summaries mean no contribution was observed at available precision; they are not claims of zero population risk or probability. No multiplicity-adjusted significance claims are made.

\begin{center}
\begin{tabular}{lrrp{.46\linewidth}}
\toprule
Suite & Cells & Draws/cell & Purpose\\
\midrule
Target scale & 16 & 50,000 & Same variance report, oracle and learned-gain targets.\\
Diagnostic gate & 120 & 10,000 & Conditional sign-error integration and slack ablation.\\
Robustness & 24 & 20,000 & Learning fallback and same-target report projection.\\
Residual strength & 27 & 20,000 & Joint, corrected-mean and variance methods.\\
Fitted oracle & 24 & 20,000 & Count selection versus unbiased assessment.\\
Pooled budgets & 18 & 20,000 & Domains including zero cheap observations.\\
Clipping / Stein & 12 & 40,000 & Unbiasedness, clipping and actual negative gain.\\
\bottomrule
\end{tabular}
\end{center}

\subsection{Target-scale and fixed-learner checks}
The target-scale grid uses $k\in\{16,32,64,128,256,512,1024,4096\}$, $m=n=k^3$, $c=0$, $\Delta=\zeta=0$ and $h\in\{.3,.25/\sqrt k\}$. The clipped report $Q$ assesses both $A$ and $G_a$ for the clipped diagnostic-mean learner. An additional report $Q-\widehat V/k$ subtracts estimated average residual loss. It changes the assessment constant but not the realized learner or target. Table~\ref{tab:all-targets} gives the full suite, including the Stein comparator from Appendix~\ref{app:learned}.
\input{appendix/table_targets}

\subsection{Conditional integration of rare diagnostic errors}
The gate grid crosses $k\in\{2^{16},2^{20},2^{24}\}$,
$z=k\eta\in\{64,256,1024,4096,16384\}$,
$h^2/\eta\in\{.5,.9,1.1,2\}$, and both signs. Set $m=n=k^3$ and $\Delta=\zeta=0$. Magnitude-only and diagnostic gates share $\epsilon=.1$, so their paired comparison isolates the diagnostic gate. A third method uses the fast schedule of Algorithm~\ref{alg:gate}. The bounds $R,C_0,L,U$ are fixed for every cell; no actual variance or effect is supplied to the gate.

Conditional on $Q=q,b$, let $u=\sqrt q$, $\mu=\theta-b$, $\sigma^2=V/k$, and $T=T(q)$. On magnitude eligibility, set
\[
 p_+=\Phi((\mu-T)/\sigma),\quad p_-=\Phi((-\mu-T)/\sigma),\quad
 p_0=\Phi((T-\mu)/\sigma)-\Phi((-T-\mu)/\sigma).
\]
On ineligibility set $p_+=p_-=0,p_0=1$. Then the conditional learning and gain-assessment MSEs are exactly
\begin{align*}
 L(q,b)&=p_+(u-\mu)^2+p_-(-u-\mu)^2+p_0\mu^2,\\
 R(q,b)&=4q\{p_+(u-\mu)^2+p_-(u+\mu)^2\}.
\end{align*}
The conditional mean gain is $p_+(2u\mu-q)+p_-(-2u\mu-q)$. These formulas remove diagnostic-mean Monte Carlo variability. Standard errors arise only from the 10,000 nuisance draws.

For the accepted wrong-sign component, the implementation evaluates \texttt{log\_ndtr} and uses a shifted log-sum-exp to average contributions. A separately scaled variance gives its log-domain Monte Carlo SE. This retains informative logarithms even when exponentiation underflows. It is an analytic conditional integration followed by nuisance Monte Carlo, not a population tail bound or an importance-sampling confidence interval. The ordinary total-risk calculation remains dominated by the nuisance term in the large-$z$ cells. Table~\ref{tab:gate} is the full plotted slice; every other ratio, sign and budget is in the JSON/CSV records.
\input{appendix/table_gate}

\subsection{Strong assessment and projection baselines}
\paragraph{Residual strength.}
With $p=1.5$, $m=512$, $n=4096$, cross $k\in\{16,128,1024\}$,
$h\in\{0,.3,.7\}$ and $(\Delta,\zeta)\in\{(0,0),(.1,0),(.1,.1)\}$.
The joint estimator is the global scalar minimizer in Appendix~\ref{app:strength}. Competitors are $|\clip(M-b,-1,1)|^{1.5}$ and $Q^{.75}$. The rate-based weights use declared bounds and counts; they are not tuned to finite-cell risk. The joint estimator has smaller empirical MSE than the mean method in 21 of 27 cells and than the variance method in 26 of 27 cells. These are descriptive counts, not uniform dominance claims.
\input{appendix/table_strength}

\paragraph{Fitted oracle potential.}
Set $n=4096$ and cross $m\in\{128,512\}$, $h\in\{0,.3\}$,
$k\in\{32,256,4096\}$, and $(\Delta,\zeta)\in\{(0,0),(.1,.1)\}$.
The rate-selected assessor uses $\clip(\widehat V-\widehat\tau,0,9)$ when $E_N\le1/k$, and $\clip((M-b)^2,0,9)$ otherwise. The unbiased comparator is
\[
 U_A=(M-b)^2-S_D^2/k,\qquad
 S_D^2=\sum_j(D_j-M)^2/(k-1),
\]
shown both raw and clipped to the valid target interval $[0,9]$. Conditional on $O$, its exact MSE is
\[
 \E[(U_A-A)^2\mid O]=\frac{4VA}{k}+\frac{2V^2}{k(k-1)}.
\]
The raw squared errors agree with this formula within Monte Carlo uncertainty in the released checks. The rate-selected method has lower empirical MSE than the clipped unbiased method in only 2 of 24 cells. The count rule is an attainment construction for risk orders; its weights do not optimize finite-sample constants.
\input{appendix/table_oracle}

\paragraph{Pooled budgets.}
At $m=128$, $p=1.5$, cross $(n,k)\in\{(0,64),(16,256),(256,16)\}$,
$h\in\{0,.3,.7\}$ and $(\Delta,\zeta)\in\{(0,0),(.1,.1)\}$.
The full comparison in Table~\ref{tab:pooled-full} includes cases where adding magnitude information under rate-based weights worsens finite-sample risk. The validity of pooling does not assert that this projection dominates the mean method.
\input{appendix/table_pooled}

\paragraph{Stein correction with clipping.}
Set $m=128$, $n=64$, $\Delta=\zeta=0$ and cross
$(c,h)\in\{(.3,0),(.3,.3),(.3,.7),(.9,.9)\}$ with $k\in\{16,128,512\}$.
Both the plug-in $(a-b)^2$ and Stein report from Appendix~\ref{app:learned} assess the same clipped-mean learner. At $(c,h,k)=(.3,0,16)$, the actual gain is negative in 78.345\% of draws (MC SE .206 percentage points), even though the plug-in is nonnegative. At $(.9,.9,16)$, 31.8\% of repairs hit the clipping boundary, and the Stein report has higher MSE. Unbiasedness averages over repair training data and is not a certificate for the observed repair.
\input{appendix/table_stein}

\subsection{Contamination, fallback, and report projection}
The robustness grid uses $k\in\{2^{16},2^{18},2^{20},2^{22}\}$,
$m=n=k^3$, $c=0$, and $\eta=8192\log(k)/k$, with $t^2=\eta/64$.
Each budget has four clean effects $h\in\{0,.95t,1.05t,.3\}$ and two zero-effect cases with $\zeta\in\{.05,.2\}$. Set $\Delta=\zeta$ for scoring the associated subclass; the robust rules themselves do not use $\Delta$.

The base magnitude-only rule has target $G_{a_0}$. Fallback changes the learner to $a$ and therefore changes its target to $G_a$. Both $g_0$ and its projected version $g_{\rm p}$ are scored against this same $G_a$; the comparison does not silently retain the old target. The Stein comparator is scored against its own direct-mean learner, and is explicitly a different pipeline. Tables~\ref{tab:robust-learning} and~\ref{tab:robust-assessment} preserve clean near-threshold learning degradation and delayed fallback at smaller budgets.
\input{appendix/table_robust_learning}
\input{appendix/table_robust_assessment}

\subsection{Finite-budget stress and variance-bound sensitivity}\label{app:finite}
\paragraph{A fixed grid of joint learning--assessment pipelines.}
The finite-budget suite crosses $k\in\{64,256,1024,4096\}$,
$c\in\{-.6,.6\}$, $\tau\in\{1.5,2.5\}$ and
$b=\clip(c_{\rm raw}+\lambda/\sqrt m,-1,1)$ with
$\lambda\in\{0,.5\}$. Both reference rules satisfy the bounded-MSE
contract; the shifted rule tests an asymmetric reference without altering
how the raw OLS residual variance is computed.
For the clean magnitude-rich path define
\[
 E_k^{\rm clean}=(k^3+k)^{-1}+k^{-3},\quad
 \ell_k^{\rm clean}=\log\{1/(k^2E_k^{\rm clean})\},\quad
 z_c^{\rm clean}=\frac{2.5\ell_k^{\rm clean}}{2-\ell_k^{\rm clean}/k}.
\]
Use $k\eta\in\{z_c^{\rm clean},4z_c^{\rm clean}\}$ unchanged across
all regimes. This is a clean-path benchmark allowance; it is \emph{not}
a critical-center prediction for the stress regimes. The six effects are
$h\in\{-.7,0,.7,-\sqrt{1.1\eta},\sqrt{1.1\eta},\sqrt{.5\eta}\}$;
all lie in the declared parameter bounds. The five regimes are:
\begin{center}
\begin{tabular}{lrrrr}
\toprule
Regime & $m$ & $n$ & $\Delta$ & $\zeta$\\
\midrule
Clean, magnitude-rich & $k^3$ & $k^3$ & 0 & 0\\
Balanced & $4k$ & $4k$ & 0 & 0\\
No cheap responses & $4k$ & 0 & 0 & 0\\
Contamination bound only & $k^3$ & $k^3$ & .1 & 0\\
Contaminated & $k^3$ & $k^3$ & .1 & .1\\
\bottomrule
\end{tabular}
\end{center}
There are $4\cdot2\cdot2\cdot2\cdot2\cdot6\cdot5=1920$ configurations.
They use independent seeds \texttt{SeedSequence([20260926,8,i])} and
12,000 draws per cell. Some physical parameters recur at different
allowances or reference rules; these are independent \emph{configuration}
cells, not distinct points in a minimax parameter search.

The eight reports are: clipped diagnostic mean with plug-in, Stein,
variance $Q$, or corrected variance $Q-\widehat V/k$; magnitude-only and
diagnostic gates with identical fast slack; and the fallback learner with
raw or projected report. Each target is recomputed from its own learner
and the realized reference. The four clipped-mean reports share a target,
as do the two fallback reports. Comparisons across the three other
learners are comparisons of pipelines, not assessors of an identical
random target. Gate risks integrate the diagnostic mean analytically;
other reports use ordinary Monte Carlo. Paired differences are provided
for the gate ablation, Stein versus plug-in, and projected versus raw
fallback. The released per-cell SEs measure simulation error; P90 and
maximum summaries do not have simultaneous confidence coverage.

Both allowances are reported in the machine-readable results. In the
clean magnitude-rich regime, all recorded pipelines have estimated
learning MSE below their allowance; this does not establish a uniform
finite-sample guarantee. At the larger allowance, the maximum learning
ratio across the four stress regimes is 10.73 for the magnitude gate and
3.27 for the diagnostic gate. The projected fallback gives .291 and
satisfies its analytic eligibility condition $k\eta\ge4V_*$ in all cells
at this allowance. Some smaller-allowance cells fail that eligibility
condition, which is recorded explicitly; no theorem-based feasibility
claim is attached to them. These stress results are outside the
magnitude-rich asymptotic assumptions and do not contradict the frontier.
The grid compares specified, untuned procedures and cannot locate the
minimax threshold at $z_c$.

\paragraph{Changing only the rule's variance bound.}
The second new suite uses $k\in\{256,1024,4096,16384\}$,
$z\in\{16,64,256,512\}$, retaining only $\eta=z/k\le.5$, and
$h^2/\eta\in\{.9,1.1,2\}$ with positive $h$.
Set $m=n=k^3$, $c=.6$, $\tau=2$, $\Delta=\zeta=0$ and use unshifted
clipped OLS. The 42 cells use independent seeds
\texttt{SeedSequence([20260926,9,i])} and 12,000 draws each.
Within a cell, multiply the valid input variance bound $V_*=4$ by
$\rho\in\{1,2,4,16\}$, keeping the data, allowance, and fast slack
unchanged. The implementation changes only the declared bound in the
rule; the true noise law is unchanged. Conditional integration records
total MSE, learning MSE, acceptance, and log wrong-sign MSE. All methods
also share a magnitude-only comparator.

Table~\ref{tab:bound-sensitivity} displays the $k=16384$, $h^2=1.1\eta$
slice; all other cells are retained in the JSON and CSV. For $z\in\{16,64\}$,
no eligible positive diagnostic threshold is observed on this slice for
any bound multiplier. At $z=256$, a factor-four bound again makes the
diagnostic gate coincide with the magnitude rule in the nuisance draws.
At $z=512$, the factor-sixteen bound has the same effect. These are
observed draw fractions, not assertions of zero population probability.
The total MSE can remain indistinguishable at the displayed precision even
when log wrong-sign MSE changes substantially.
\input{appendix/table_bound_sensitivity}

\subsection{Executable reproduction and numerical checks}
From the project root, run
\begin{verbatim}
python experiments/run.py
python experiments/make_assets.py
python experiments/run_finite.py
python experiments/make_refinement_assets.py
python -m pytest -q
\end{verbatim}
The first command writes \path{experiments/results/results.json}, a flattened CSV, and a source/result hash manifest. The second generates four vector figures and eight table files. The two additional scripts regenerate both new suites and the three aggregate tables; their JSON/CSV records and source manifest use the prefix \texttt{finite\_}. The complete source also includes the exact adaptive reservoir simulator. Test cases cover gain identities, OLS and pooled-variance moments, Gaussian independence, the quartic global optimizer, conditional integration, log-domain tails, the fallback's pathwise bound, report projection, Stein bias, the exact oracle-risk formula, the corrected-center identities, a joint-tail inequality, and adaptive innovations. These are numerical consistency checks; the mathematical guarantees rely on the proofs above.

The final recorded environment is Python 3.12.14, NumPy 2.3.5, SciPy 1.17.0 and Matplotlib 3.10.8. The version manifest and source hashes make the reported environment and implementation inspectable. Exact byte replay is checked within that environment; different random-number libraries or platform versions need not yield byte-identical floating-point summaries. A \texttt{--quick} mode exists only for smoke tests and is rejected by the paper-asset generator.

%% file: appendix/table_targets.tex
\begingroup
\small
\setlength{\tabcolsep}{4pt}
\begin{longtable}{rrllll}
\caption{Full target-scale suite. Entries are scaled MSE (one Monte Carlo SE). The variance report is $Q$; the corrected report is $Q-\widehat V/k$.}\label{tab:all-targets}\\
\toprule
$k$ & $h$ & $k^3R_A(Q)$ & $k^2R_G(Q)$ & $k^2R_G(\mathrm{corr.})$ & $k^2R_G(\mathrm{Stein})$\\
\midrule\endfirsthead
\multicolumn{6}{l}{\tablename\ \thetable\ (continued)}\\
\toprule
$k$ & $h$ & $k^3R_A(Q)$ & $k^2R_G(Q)$ & $k^2R_G(\mathrm{corr.})$ & $k^2R_G(\mathrm{Stein})$\\
\midrule\endhead
\bottomrule\endfoot
16 & .3 & \shortstack[r]{$15.4$\\$(0.0893)$} & \shortstack[r]{$14.2$\\$(0.195)$} & \shortstack[r]{$9.67$\\$(0.149)$} & \shortstack[r]{$46.8$\\$(0.755)$}\\*[2pt]
16 & $ .25/\sqrt{k}$ & \shortstack[r]{$8.08$\\$(0.081)$} & \shortstack[r]{$13.9$\\$(0.176)$} & \shortstack[r]{$8.44$\\$(0.132)$} & \shortstack[r]{$32.4$\\$(0.505)$}\\[2pt]
32 & .3 & \shortstack[r]{$17.4$\\$(0.11)$} & \shortstack[r]{$13.6$\\$(0.19)$} & \shortstack[r]{$9.21$\\$(0.145)$} & \shortstack[r]{$58.1$\\$(0.866)$}\\[2pt]
32 & $ .25/\sqrt{k}$ & \shortstack[r]{$8.08$\\$(0.0804)$} & \shortstack[r]{$13.4$\\$(0.183)$} & \shortstack[r]{$8.36$\\$(0.141)$} & \shortstack[r]{$33$\\$(0.544)$}\\[2pt]
64 & .3 & \shortstack[r]{$17.4$\\$(0.11)$} & \shortstack[r]{$13.3$\\$(0.191)$} & \shortstack[r]{$8.92$\\$(0.146)$} & \shortstack[r]{$83.1$\\$(1.18)$}\\[2pt]
64 & $ .25/\sqrt{k}$ & \shortstack[r]{$8.15$\\$(0.0802)$} & \shortstack[r]{$12.9$\\$(0.184)$} & \shortstack[r]{$8.18$\\$(0.142)$} & \shortstack[r]{$32.7$\\$(0.566)$}\\[2pt]
128 & .3 & \shortstack[r]{$17.6$\\$(0.111)$} & \shortstack[r]{$13.2$\\$(0.193)$} & \shortstack[r]{$8.88$\\$(0.148)$} & \shortstack[r]{$132$\\$(1.65)$}\\[2pt]
128 & $ .25/\sqrt{k}$ & \shortstack[r]{$8.35$\\$(0.0805)$} & \shortstack[r]{$12.4$\\$(0.177)$} & \shortstack[r]{$8.01$\\$(0.136)$} & \shortstack[r]{$32.4$\\$(0.55)$}\\[2pt]
256 & .3 & \shortstack[r]{$17.3$\\$(0.109)$} & \shortstack[r]{$13.4$\\$(0.196)$} & \shortstack[r]{$9.03$\\$(0.15)$} & \shortstack[r]{$225$\\$(2.32)$}\\[2pt]
256 & $ .25/\sqrt{k}$ & \shortstack[r]{$8.4$\\$(0.079)$} & \shortstack[r]{$12.5$\\$(0.178)$} & \shortstack[r]{$8.15$\\$(0.136)$} & \shortstack[r]{$33.1$\\$(0.55)$}\\[2pt]
512 & .3 & \shortstack[r]{$17.5$\\$(0.11)$} & \shortstack[r]{$13.1$\\$(0.186)$} & \shortstack[r]{$8.7$\\$(0.141)$} & \shortstack[r]{$418$\\$(3.59)$}\\[2pt]
512 & $ .25/\sqrt{k}$ & \shortstack[r]{$8.74$\\$(0.0772)$} & \shortstack[r]{$12.2$\\$(0.176)$} & \shortstack[r]{$8.03$\\$(0.134)$} & \shortstack[r]{$32.5$\\$(0.534)$}\\[2pt]
1024 & .3 & \shortstack[r]{$17.6$\\$(0.111)$} & \shortstack[r]{$13.5$\\$(0.209)$} & \shortstack[r]{$9.11$\\$(0.163)$} & \shortstack[r]{$814$\\$(6.6)$}\\[2pt]
1024 & $ .25/\sqrt{k}$ & \shortstack[r]{$9.45$\\$(0.0765)$} & \shortstack[r]{$12$\\$(0.171)$} & \shortstack[r]{$7.91$\\$(0.129)$} & \shortstack[r]{$32.2$\\$(0.525)$}\\*[2pt]
4096 & .3 & \shortstack[r]{$17.4$\\$(0.112)$} & \shortstack[r]{$13.4$\\$(0.199)$} & \shortstack[r]{$8.96$\\$(0.153)$} & \shortstack[r]{$3.14\!\times\!10^{3}$\\$(21.4)$}\\*[2pt]
4096 & $ .25/\sqrt{k}$ & \shortstack[r]{$12.1$\\$(0.0757)$} & \shortstack[r]{$12$\\$(0.167)$} & \shortstack[r]{$7.96$\\$(0.125)$} & \shortstack[r]{$32.4$\\$(0.511)$}\\[2pt]
\end{longtable}
\endgroup

%% file: appendix/table_gate.tex
\begingroup
\small
\setlength{\tabcolsep}{4pt}
\begin{longtable}{rlrll}
\caption{Conditional integration at $k=2^{24}$, $h^2=1.1\eta$, positive $h$. The exponent uses the log-domain mean of the wrong-sign MSE. Remaining columns use nuisance Monte Carlo means (one SE).}\label{tab:gate}\\
\toprule
$z$ & Gate & $-\log(k^2R_W)/z$ & $\E(a-\theta)^2/\eta$ & $\E(g-G)^2/E_k$\\
\midrule\endfirsthead
\multicolumn{5}{l}{\tablename\ \thetable\ (continued)}\\
\toprule
$z$ & Gate & $-\log(k^2R_W)/z$ & $\E(a-\theta)^2/\eta$ & $\E(g-G)^2/E_k$\\
\midrule\endhead
\bottomrule\endfoot
64 & Magnitude & 0.1413 & \shortstack[r]{$6.60\!\times\!10^{-9}$\\$(7.54\!\times\!10^{-13})$} & \shortstack[r]{$998$\\$(0.114)$}\\*[2pt]
64 & Diagnostic, fixed & 0.1413 & \shortstack[r]{$6.60\!\times\!10^{-9}$\\$(7.54\!\times\!10^{-13})$} & \shortstack[r]{$998$\\$(0.114)$}\\[2pt]
64 & Diagnostic, fast & 0.1413 & \shortstack[r]{$6.60\!\times\!10^{-9}$\\$(7.54\!\times\!10^{-13})$} & \shortstack[r]{$998$\\$(0.114)$}\\[2pt]
256 & Magnitude & 0.2334 & \shortstack[r]{$3.28\!\times\!10^{-12}$\\$(4.59\!\times\!10^{-14})$} & \shortstack[r]{$7.93$\\$(0.111)$}\\[2pt]
256 & Diagnostic, fixed & 0.4146 & \shortstack[r]{$3.28\!\times\!10^{-12}$\\$(4.59\!\times\!10^{-14})$} & \shortstack[r]{$7.93$\\$(0.111)$}\\[2pt]
256 & Diagnostic, fast & 0.4373 & \shortstack[r]{$3.28\!\times\!10^{-12}$\\$(4.59\!\times\!10^{-14})$} & \shortstack[r]{$7.93$\\$(0.111)$}\\[2pt]
1024 & Magnitude & 0.2626 & \shortstack[r]{$2.10\!\times\!10^{-13}$\\$(3.03\!\times\!10^{-15})$} & \shortstack[r]{$8.14$\\$(0.117)$}\\[2pt]
1024 & Diagnostic, fixed & 0.7307 & \shortstack[r]{$2.10\!\times\!10^{-13}$\\$(3.03\!\times\!10^{-15})$} & \shortstack[r]{$8.14$\\$(0.117)$}\\[2pt]
1024 & Diagnostic, fast & 0.7237 & \shortstack[r]{$2.10\!\times\!10^{-13}$\\$(3.03\!\times\!10^{-15})$} & \shortstack[r]{$8.14$\\$(0.117)$}\\[2pt]
4096 & Magnitude & 0.2713 & \shortstack[r]{$1.29\!\times\!10^{-14}$\\$(1.85\!\times\!10^{-16})$} & \shortstack[r]{$7.99$\\$(0.114)$}\\[2pt]
4096 & Diagnostic, fixed & 0.9090 & \shortstack[r]{$1.29\!\times\!10^{-14}$\\$(1.85\!\times\!10^{-16})$} & \shortstack[r]{$7.99$\\$(0.114)$}\\[2pt]
4096 & Diagnostic, fast & 0.8938 & \shortstack[r]{$1.29\!\times\!10^{-14}$\\$(1.85\!\times\!10^{-16})$} & \shortstack[r]{$7.99$\\$(0.114)$}\\[2pt]
16384 & Magnitude & 0.2738 & \shortstack[r]{$8.10\!\times\!10^{-16}$\\$(1.15\!\times\!10^{-17})$} & \shortstack[r]{$8.03$\\$(0.114)$}\\*[2pt]
16384 & Diagnostic, fixed & 1.0026 & \shortstack[r]{$8.25\!\times\!10^{-16}$\\$(1.15\!\times\!10^{-17})$} & \shortstack[r]{$8.03$\\$(0.114)$}\\*[2pt]
16384 & Diagnostic, fast & 0.9889 & \shortstack[r]{$8.10\!\times\!10^{-16}$\\$(1.15\!\times\!10^{-17})$} & \shortstack[r]{$8.03$\\$(0.114)$}\\[2pt]
\end{longtable}
\endgroup

%% file: appendix/table_strength.tex
\begingroup
\small
\setlength{\tabcolsep}{4pt}
\begin{longtable}{rrrrlll}
\caption{Residual-strength suite, $m=512$, $n=4096$, $p=1.5$. MSE (one SE); all three procedures share the same target and data. The mean method clips the corrected coefficient before applying the power.}\label{tab:strength-full}\\
\toprule
$k$ & $h$ & $\Delta$ & $\zeta$ & Joint & Mean & Variance\\
\midrule\endfirsthead
\multicolumn{7}{l}{\tablename\ \thetable\ (continued)}\\
\toprule
$k$ & $h$ & $\Delta$ & $\zeta$ & Joint & Mean & Variance\\
\midrule\endhead
\bottomrule\endfoot
16 & 0.0 & 0.0 & 0.0 & \shortstack[r]{$0.02$\\$(2.44\!\times\!10^{-4})$} & \shortstack[r]{$0.072$\\$(0.00102)$} & \shortstack[r]{$0.0201$\\$(2.55\!\times\!10^{-4})$}\\*[2pt]
16 & 0.0 & 0.1 & 0.0 & \shortstack[r]{$0.0212$\\$(2.21\!\times\!10^{-4})$} & \shortstack[r]{$0.0724$\\$(0.00104)$} & \shortstack[r]{$0.0206$\\$(2.57\!\times\!10^{-4})$}\\[2pt]
16 & 0.0 & 0.1 & 0.1 & \shortstack[r]{$0.046$\\$(3.63\!\times\!10^{-4})$} & \shortstack[r]{$0.0764$\\$(0.00107)$} & \shortstack[r]{$0.0526$\\$(4.33\!\times\!10^{-4})$}\\[2pt]
16 & 0.3 & 0.0 & 0.0 & \shortstack[r]{$0.0172$\\$(1.47\!\times\!10^{-4})$} & \shortstack[r]{$0.0811$\\$(0.0011)$} & \shortstack[r]{$0.0198$\\$(1.51\!\times\!10^{-4})$}\\[2pt]
16 & 0.3 & 0.1 & 0.0 & \shortstack[r]{$0.0149$\\$(1.44\!\times\!10^{-4})$} & \shortstack[r]{$0.0812$\\$(0.0011)$} & \shortstack[r]{$0.0197$\\$(1.54\!\times\!10^{-4})$}\\[2pt]
16 & 0.3 & 0.1 & 0.1 & \shortstack[r]{$0.0297$\\$(2.57\!\times\!10^{-4})$} & \shortstack[r]{$0.0876$\\$(0.00117)$} & \shortstack[r]{$0.0354$\\$(2.82\!\times\!10^{-4})$}\\[2pt]
16 & 0.7 & 0.0 & 0.0 & \shortstack[r]{$0.0149$\\$(1.58\!\times\!10^{-4})$} & \shortstack[r]{$0.113$\\$(6.49\!\times\!10^{-4})$} & \shortstack[r]{$0.0156$\\$(1.73\!\times\!10^{-4})$}\\[2pt]
16 & 0.7 & 0.1 & 0.0 & \shortstack[r]{$0.0147$\\$(1.58\!\times\!10^{-4})$} & \shortstack[r]{$0.114$\\$(6.50\!\times\!10^{-4})$} & \shortstack[r]{$0.0159$\\$(1.76\!\times\!10^{-4})$}\\[2pt]
16 & 0.7 & 0.1 & 0.1 & \shortstack[r]{$0.0201$\\$(1.77\!\times\!10^{-4})$} & \shortstack[r]{$0.114$\\$(6.52\!\times\!10^{-4})$} & \shortstack[r]{$0.0213$\\$(1.86\!\times\!10^{-4})$}\\[2pt]
128 & 0.0 & 0.0 & 0.0 & \shortstack[r]{$0.00839$\\$(1.28\!\times\!10^{-4})$} & \shortstack[r]{$0.00442$\\$(6.89\!\times\!10^{-5})$} & \shortstack[r]{$0.0205$\\$(2.57\!\times\!10^{-4})$}\\[2pt]
128 & 0.0 & 0.1 & 0.0 & \shortstack[r]{$0.00395$\\$(5.52\!\times\!10^{-5})$} & \shortstack[r]{$0.00441$\\$(6.70\!\times\!10^{-5})$} & \shortstack[r]{$0.0204$\\$(2.57\!\times\!10^{-4})$}\\[2pt]
128 & 0.0 & 0.1 & 0.1 & \shortstack[r]{$0.00583$\\$(8.02\!\times\!10^{-5})$} & \shortstack[r]{$0.00463$\\$(7.19\!\times\!10^{-5})$} & \shortstack[r]{$0.0532$\\$(4.32\!\times\!10^{-4})$}\\[2pt]
128 & 0.3 & 0.0 & 0.0 & \shortstack[r]{$0.0108$\\$(1.07\!\times\!10^{-4})$} & \shortstack[r]{$0.0135$\\$(1.63\!\times\!10^{-4})$} & \shortstack[r]{$0.0199$\\$(1.53\!\times\!10^{-4})$}\\[2pt]
128 & 0.3 & 0.1 & 0.0 & \shortstack[r]{$0.00888$\\$(8.61\!\times\!10^{-5})$} & \shortstack[r]{$0.0134$\\$(1.65\!\times\!10^{-4})$} & \shortstack[r]{$0.0199$\\$(1.55\!\times\!10^{-4})$}\\[2pt]
128 & 0.3 & 0.1 & 0.1 & \shortstack[r]{$0.0125$\\$(1.25\!\times\!10^{-4})$} & \shortstack[r]{$0.0141$\\$(1.72\!\times\!10^{-4})$} & \shortstack[r]{$0.0352$\\$(2.83\!\times\!10^{-4})$}\\[2pt]
128 & 0.7 & 0.0 & 0.0 & \shortstack[r]{$0.0121$\\$(1.23\!\times\!10^{-4})$} & \shortstack[r]{$0.0348$\\$(3.11\!\times\!10^{-4})$} & \shortstack[r]{$0.0156$\\$(1.71\!\times\!10^{-4})$}\\[2pt]
128 & 0.7 & 0.1 & 0.0 & \shortstack[r]{$0.0121$\\$(1.25\!\times\!10^{-4})$} & \shortstack[r]{$0.0346$\\$(3.12\!\times\!10^{-4})$} & \shortstack[r]{$0.0155$\\$(1.70\!\times\!10^{-4})$}\\[2pt]
128 & 0.7 & 0.1 & 0.1 & \shortstack[r]{$0.0147$\\$(1.40\!\times\!10^{-4})$} & \shortstack[r]{$0.0349$\\$(3.11\!\times\!10^{-4})$} & \shortstack[r]{$0.0212$\\$(1.85\!\times\!10^{-4})$}\\[2pt]
1024 & 0.0 & 0.0 & 0.0 & \shortstack[r]{$8.92\!\times\!10^{-4}$\\$(1.72\!\times\!10^{-5})$} & \shortstack[r]{$7.33\!\times\!10^{-4}$\\$(1.16\!\times\!10^{-5})$} & \shortstack[r]{$0.0203$\\$(2.56\!\times\!10^{-4})$}\\[2pt]
1024 & 0.0 & 0.1 & 0.0 & \shortstack[r]{$7.12\!\times\!10^{-4}$\\$(1.12\!\times\!10^{-5})$} & \shortstack[r]{$7.21\!\times\!10^{-4}$\\$(1.15\!\times\!10^{-5})$} & \shortstack[r]{$0.0204$\\$(2.57\!\times\!10^{-4})$}\\[2pt]
1024 & 0.0 & 0.1 & 0.1 & \shortstack[r]{$7.98\!\times\!10^{-4}$\\$(1.22\!\times\!10^{-5})$} & \shortstack[r]{$7.31\!\times\!10^{-4}$\\$(1.13\!\times\!10^{-5})$} & \shortstack[r]{$0.0533$\\$(4.27\!\times\!10^{-4})$}\\[2pt]
1024 & 0.3 & 0.0 & 0.0 & \shortstack[r]{$0.00433$\\$(4.57\!\times\!10^{-5})$} & \shortstack[r]{$0.00406$\\$(4.26\!\times\!10^{-5})$} & \shortstack[r]{$0.0197$\\$(1.53\!\times\!10^{-4})$}\\[2pt]
1024 & 0.3 & 0.1 & 0.0 & \shortstack[r]{$0.00364$\\$(3.76\!\times\!10^{-5})$} & \shortstack[r]{$0.00402$\\$(4.32\!\times\!10^{-5})$} & \shortstack[r]{$0.0194$\\$(1.49\!\times\!10^{-4})$}\\[2pt]
1024 & 0.3 & 0.1 & 0.1 & \shortstack[r]{$0.0041$\\$(4.33\!\times\!10^{-5})$} & \shortstack[r]{$0.00409$\\$(4.36\!\times\!10^{-5})$} & \shortstack[r]{$0.035$\\$(2.80\!\times\!10^{-4})$}\\[2pt]
1024 & 0.7 & 0.0 & 0.0 & \shortstack[r]{$0.00714$\\$(7.06\!\times\!10^{-5})$} & \shortstack[r]{$0.00994$\\$(9.90\!\times\!10^{-5})$} & \shortstack[r]{$0.0151$\\$(1.69\!\times\!10^{-4})$}\\*[2pt]
1024 & 0.7 & 0.1 & 0.0 & \shortstack[r]{$0.00658$\\$(6.55\!\times\!10^{-5})$} & \shortstack[r]{$0.01$\\$(1.01\!\times\!10^{-4})$} & \shortstack[r]{$0.0152$\\$(1.68\!\times\!10^{-4})$}\\*[2pt]
1024 & 0.7 & 0.1 & 0.1 & \shortstack[r]{$0.00744$\\$(7.28\!\times\!10^{-5})$} & \shortstack[r]{$0.0102$\\$(1.03\!\times\!10^{-4})$} & \shortstack[r]{$0.0207$\\$(1.80\!\times\!10^{-4})$}\\[2pt]
\end{longtable}
\endgroup

%% file: appendix/table_oracle.tex
\begingroup
\small
\setlength{\tabcolsep}{4pt}
\begin{longtable}{rrrrlll}
\caption{Fitted oracle-potential suite, $n=4096$. MSE (one SE); $\Delta=\zeta$. Count selection chooses between the variance and mean plug-ins using $E_N\le1/k$. The competing unbiased estimator is shown both raw and clipped.}\label{tab:oracle-full}\\
\toprule
$m$ & $k$ & $h$ & $\Delta$ & Selected & \shortstack{Unbiased,\\clipped} & \shortstack{Unbiased,\\raw}\\
\midrule\endfirsthead
\multicolumn{7}{l}{\tablename\ \thetable\ (continued)}\\
\toprule
$m$ & $k$ & $h$ & $\Delta$ & Selected & \shortstack{Unbiased,\\clipped} & \shortstack{Unbiased,\\raw}\\
\midrule\endhead
\bottomrule\endfoot
128 & 32 & 0.0 & 0.0 & \shortstack[r]{$0.0282$\\$(4.18\!\times\!10^{-4})$} & \shortstack[r]{$0.00968$\\$(3.00\!\times\!10^{-4})$} & \shortstack[r]{$0.012$\\$(2.97\!\times\!10^{-4})$}\\*[2pt]
128 & 32 & 0.0 & 0.1 & \shortstack[r]{$0.0527$\\$(6.27\!\times\!10^{-4})$} & \shortstack[r]{$0.011$\\$(3.41\!\times\!10^{-4})$} & \shortstack[r]{$0.0135$\\$(3.39\!\times\!10^{-4})$}\\[2pt]
128 & 256 & 0.0 & 0.0 & \shortstack[r]{$6.99\!\times\!10^{-4}$\\$(1.59\!\times\!10^{-5})$} & \shortstack[r]{$5.95\!\times\!10^{-4}$\\$(1.41\!\times\!10^{-5})$} & \shortstack[r]{$6.36\!\times\!10^{-4}$\\$(1.41\!\times\!10^{-5})$}\\[2pt]
128 & 256 & 0.0 & 0.1 & \shortstack[r]{$7.24\!\times\!10^{-4}$\\$(1.58\!\times\!10^{-5})$} & \shortstack[r]{$6.09\!\times\!10^{-4}$\\$(1.39\!\times\!10^{-5})$} & \shortstack[r]{$6.55\!\times\!10^{-4}$\\$(1.39\!\times\!10^{-5})$}\\[2pt]
128 & 4096 & 0.0 & 0.0 & \shortstack[r]{$3.14\!\times\!10^{-5}$\\$(6.88\!\times\!10^{-7})$} & \shortstack[r]{$3.11\!\times\!10^{-5}$\\$(6.79\!\times\!10^{-7})$} & \shortstack[r]{$3.11\!\times\!10^{-5}$\\$(6.79\!\times\!10^{-7})$}\\[2pt]
128 & 4096 & 0.0 & 0.1 & \shortstack[r]{$3.25\!\times\!10^{-5}$\\$(6.77\!\times\!10^{-7})$} & \shortstack[r]{$3.22\!\times\!10^{-5}$\\$(6.66\!\times\!10^{-7})$} & \shortstack[r]{$3.23\!\times\!10^{-5}$\\$(6.66\!\times\!10^{-7})$}\\[2pt]
128 & 32 & 0.3 & 0.0 & \shortstack[r]{$0.0373$\\$(4.57\!\times\!10^{-4})$} & \shortstack[r]{$0.0342$\\$(7.44\!\times\!10^{-4})$} & \shortstack[r]{$0.0381$\\$(7.41\!\times\!10^{-4})$}\\[2pt]
128 & 32 & 0.3 & 0.1 & \shortstack[r]{$0.0611$\\$(6.53\!\times\!10^{-4})$} & \shortstack[r]{$0.0335$\\$(7.37\!\times\!10^{-4})$} & \shortstack[r]{$0.0376$\\$(7.34\!\times\!10^{-4})$}\\[2pt]
128 & 256 & 0.3 & 0.0 & \shortstack[r]{$0.00374$\\$(5.74\!\times\!10^{-5})$} & \shortstack[r]{$0.00365$\\$(5.43\!\times\!10^{-5})$} & \shortstack[r]{$0.00367$\\$(5.42\!\times\!10^{-5})$}\\[2pt]
128 & 256 & 0.3 & 0.1 & \shortstack[r]{$0.00385$\\$(6.04\!\times\!10^{-5})$} & \shortstack[r]{$0.00376$\\$(5.72\!\times\!10^{-5})$} & \shortstack[r]{$0.00378$\\$(5.72\!\times\!10^{-5})$}\\[2pt]
128 & 4096 & 0.3 & 0.0 & \shortstack[r]{$2.14\!\times\!10^{-4}$\\$(2.93\!\times\!10^{-6})$} & \shortstack[r]{$2.13\!\times\!10^{-4}$\\$(2.92\!\times\!10^{-6})$} & \shortstack[r]{$2.13\!\times\!10^{-4}$\\$(2.92\!\times\!10^{-6})$}\\[2pt]
128 & 4096 & 0.3 & 0.1 & \shortstack[r]{$2.26\!\times\!10^{-4}$\\$(3.14\!\times\!10^{-6})$} & \shortstack[r]{$2.26\!\times\!10^{-4}$\\$(3.13\!\times\!10^{-6})$} & \shortstack[r]{$2.26\!\times\!10^{-4}$\\$(3.13\!\times\!10^{-6})$}\\[2pt]
512 & 32 & 0.0 & 0.0 & \shortstack[r]{$0.00815$\\$(1.26\!\times\!10^{-4})$} & \shortstack[r]{$0.00748$\\$(2.49\!\times\!10^{-4})$} & \shortstack[r]{$0.00935$\\$(2.47\!\times\!10^{-4})$}\\[2pt]
512 & 32 & 0.0 & 0.1 & \shortstack[r]{$0.0245$\\$(2.51\!\times\!10^{-4})$} & \shortstack[r]{$0.00791$\\$(2.44\!\times\!10^{-4})$} & \shortstack[r]{$0.01$\\$(2.42\!\times\!10^{-4})$}\\[2pt]
512 & 256 & 0.0 & 0.0 & \shortstack[r]{$0.00781$\\$(1.21\!\times\!10^{-4})$} & \shortstack[r]{$1.98\!\times\!10^{-4}$\\$(6.19\!\times\!10^{-6})$} & \shortstack[r]{$2.37\!\times\!10^{-4}$\\$(6.16\!\times\!10^{-6})$}\\[2pt]
512 & 256 & 0.0 & 0.1 & \shortstack[r]{$3.46\!\times\!10^{-4}$\\$(8.25\!\times\!10^{-6})$} & \shortstack[r]{$2.36\!\times\!10^{-4}$\\$(6.78\!\times\!10^{-6})$} & \shortstack[r]{$2.77\!\times\!10^{-4}$\\$(6.74\!\times\!10^{-6})$}\\[2pt]
512 & 4096 & 0.0 & 0.0 & \shortstack[r]{$8.57\!\times\!10^{-6}$\\$(1.83\!\times\!10^{-7})$} & \shortstack[r]{$8.18\!\times\!10^{-6}$\\$(1.75\!\times\!10^{-7})$} & \shortstack[r]{$8.31\!\times\!10^{-6}$\\$(1.75\!\times\!10^{-7})$}\\[2pt]
512 & 4096 & 0.0 & 0.1 & \shortstack[r]{$9.19\!\times\!10^{-6}$\\$(2.03\!\times\!10^{-7})$} & \shortstack[r]{$8.78\!\times\!10^{-6}$\\$(1.94\!\times\!10^{-7})$} & \shortstack[r]{$8.92\!\times\!10^{-6}$\\$(1.94\!\times\!10^{-7})$}\\[2pt]
512 & 32 & 0.3 & 0.0 & \shortstack[r]{$0.0123$\\$(1.33\!\times\!10^{-4})$} & \shortstack[r]{$0.0281$\\$(6.17\!\times\!10^{-4})$} & \shortstack[r]{$0.0323$\\$(6.13\!\times\!10^{-4})$}\\[2pt]
512 & 32 & 0.3 & 0.1 & \shortstack[r]{$0.0267$\\$(2.57\!\times\!10^{-4})$} & \shortstack[r]{$0.0322$\\$(7.54\!\times\!10^{-4})$} & \shortstack[r]{$0.0368$\\$(7.49\!\times\!10^{-4})$}\\[2pt]
512 & 256 & 0.3 & 0.0 & \shortstack[r]{$0.0125$\\$(1.35\!\times\!10^{-4})$} & \shortstack[r]{$0.0032$\\$(4.08\!\times\!10^{-5})$} & \shortstack[r]{$0.00322$\\$(4.08\!\times\!10^{-5})$}\\[2pt]
512 & 256 & 0.3 & 0.1 & \shortstack[r]{$0.00352$\\$(5.05\!\times\!10^{-5})$} & \shortstack[r]{$0.00341$\\$(4.66\!\times\!10^{-5})$} & \shortstack[r]{$0.00343$\\$(4.66\!\times\!10^{-5})$}\\*[2pt]
512 & 4096 & 0.3 & 0.0 & \shortstack[r]{$1.92\!\times\!10^{-4}$\\$(2.18\!\times\!10^{-6})$} & \shortstack[r]{$1.91\!\times\!10^{-4}$\\$(2.16\!\times\!10^{-6})$} & \shortstack[r]{$1.91\!\times\!10^{-4}$\\$(2.16\!\times\!10^{-6})$}\\*[2pt]
512 & 4096 & 0.3 & 0.1 & \shortstack[r]{$2.03\!\times\!10^{-4}$\\$(2.27\!\times\!10^{-6})$} & \shortstack[r]{$2.02\!\times\!10^{-4}$\\$(2.25\!\times\!10^{-6})$} & \shortstack[r]{$2.02\!\times\!10^{-4}$\\$(2.25\!\times\!10^{-6})$}\\[2pt]
\end{longtable}
\endgroup

%% file: appendix/table_pooled.tex
\begingroup
\small
\setlength{\tabcolsep}{4pt}
\begin{longtable}{rrrrll}
\caption{Pooled-budget suite, $m=128$, $p=1.5$, $\Delta=\zeta$. MSE (one SE). The pooled joint estimator uses diagnostic residuals even at $n=0$.}\label{tab:pooled-full}\\
\toprule
$n$ & $k$ & $h$ & $\Delta$ & Pooled joint & Mean\\
\midrule\endfirsthead
\multicolumn{6}{l}{\tablename\ \thetable\ (continued)}\\
\toprule
$n$ & $k$ & $h$ & $\Delta$ & Pooled joint & Mean\\
\midrule\endhead
\bottomrule\endfoot
0 & 64 & 0.0 & 0.0 & \shortstack[r]{$0.03$\\$(5.55\!\times\!10^{-4})$} & \shortstack[r]{$0.0164$\\$(2.59\!\times\!10^{-4})$}\\*[2pt]
0 & 64 & 0.0 & 0.1 & \shortstack[r]{$0.0291$\\$(5.04\!\times\!10^{-4})$} & \shortstack[r]{$0.0171$\\$(2.75\!\times\!10^{-4})$}\\[2pt]
0 & 64 & 0.3 & 0.0 & \shortstack[r]{$0.0401$\\$(6.02\!\times\!10^{-4})$} & \shortstack[r]{$0.031$\\$(4.42\!\times\!10^{-4})$}\\[2pt]
0 & 64 & 0.3 & 0.1 & \shortstack[r]{$0.0427$\\$(6.05\!\times\!10^{-4})$} & \shortstack[r]{$0.0325$\\$(4.67\!\times\!10^{-4})$}\\[2pt]
0 & 64 & 0.7 & 0.0 & \shortstack[r]{$0.0615$\\$(4.57\!\times\!10^{-4})$} & \shortstack[r]{$0.0653$\\$(4.84\!\times\!10^{-4})$}\\[2pt]
0 & 64 & 0.7 & 0.1 & \shortstack[r]{$0.0615$\\$(4.58\!\times\!10^{-4})$} & \shortstack[r]{$0.0671$\\$(4.90\!\times\!10^{-4})$}\\[2pt]
16 & 256 & 0.0 & 0.0 & \shortstack[r]{$0.00914$\\$(1.68\!\times\!10^{-4})$} & \shortstack[r]{$0.00594$\\$(9.69\!\times\!10^{-5})$}\\[2pt]
16 & 256 & 0.0 & 0.1 & \shortstack[r]{$0.00785$\\$(1.23\!\times\!10^{-4})$} & \shortstack[r]{$0.00592$\\$(9.45\!\times\!10^{-5})$}\\[2pt]
16 & 256 & 0.3 & 0.0 & \shortstack[r]{$0.0176$\\$(2.25\!\times\!10^{-4})$} & \shortstack[r]{$0.0156$\\$(1.96\!\times\!10^{-4})$}\\[2pt]
16 & 256 & 0.3 & 0.1 & \shortstack[r]{$0.0177$\\$(2.12\!\times\!10^{-4})$} & \shortstack[r]{$0.0161$\\$(2.02\!\times\!10^{-4})$}\\[2pt]
16 & 256 & 0.7 & 0.0 & \shortstack[r]{$0.0327$\\$(2.92\!\times\!10^{-4})$} & \shortstack[r]{$0.0375$\\$(3.31\!\times\!10^{-4})$}\\[2pt]
16 & 256 & 0.7 & 0.1 & \shortstack[r]{$0.0316$\\$(2.88\!\times\!10^{-4})$} & \shortstack[r]{$0.0386$\\$(3.34\!\times\!10^{-4})$}\\[2pt]
256 & 16 & 0.0 & 0.0 & \shortstack[r]{$0.0627$\\$(7.73\!\times\!10^{-4})$} & \shortstack[r]{$0.083$\\$(0.00114)$}\\[2pt]
256 & 16 & 0.0 & 0.1 & \shortstack[r]{$0.0886$\\$(9.22\!\times\!10^{-4})$} & \shortstack[r]{$0.085$\\$(0.00116)$}\\[2pt]
256 & 16 & 0.3 & 0.0 & \shortstack[r]{$0.0508$\\$(5.96\!\times\!10^{-4})$} & \shortstack[r]{$0.0902$\\$(0.00119)$}\\[2pt]
256 & 16 & 0.3 & 0.1 & \shortstack[r]{$0.0683$\\$(7.32\!\times\!10^{-4})$} & \shortstack[r]{$0.0928$\\$(0.00121)$}\\*[2pt]
256 & 16 & 0.7 & 0.0 & \shortstack[r]{$0.0628$\\$(5.11\!\times\!10^{-4})$} & \shortstack[r]{$0.117$\\$(6.53\!\times\!10^{-4})$}\\*[2pt]
256 & 16 & 0.7 & 0.1 & \shortstack[r]{$0.0618$\\$(4.71\!\times\!10^{-4})$} & \shortstack[r]{$0.117$\\$(6.58\!\times\!10^{-4})$}\\[2pt]
\end{longtable}
\endgroup

%% file: appendix/table_stein.tex
\begingroup
\small
\setlength{\tabcolsep}{4pt}
\begin{longtable}{rrrllll}
\caption{Clipping-boundary suite, $m=128$, $n=64$. Both reports assess the same clipped-mean repair. MSE and mean bias (one SE); the last column is the observed negative-gain fraction, not a confidence guarantee.}\label{tab:stein-full}\\
\toprule
$c$ & $h$ & $k$ & Gain plug-in & Stein & Stein bias & $\Pr_n(G<0)$\\
\midrule\endfirsthead
\multicolumn{7}{l}{\tablename\ \thetable\ (continued)}\\
\toprule
$c$ & $h$ & $k$ & Gain plug-in & Stein & Stein bias & $\Pr_n(G<0)$\\
\midrule\endhead
\bottomrule\endfoot
0.3 & 0.0 & 16 & \shortstack[r]{$0.193$\\$(0.00311)$} & \shortstack[r]{$0.133$\\$(0.00237)$} & \shortstack[r]{$-8.83\!\times\!10^{-4}$\\$(0.00182)$} & \shortstack[r]{$0.783$\\$(0.00206)$}\\*[2pt]
0.3 & 0.0 & 128 & \shortstack[r]{$0.004$\\$(6.77\!\times\!10^{-5})$} & \shortstack[r]{$0.003$\\$(5.39\!\times\!10^{-5})$} & \shortstack[r]{$4.59\!\times\!10^{-4}$\\$(2.74\!\times\!10^{-4})$} & \shortstack[r]{$0.506$\\$(0.0025)$}\\[2pt]
0.3 & 0.0 & 512 & \shortstack[r]{$4.39\!\times\!10^{-4}$\\$(7.47\!\times\!10^{-6})$} & \shortstack[r]{$3.76\!\times\!10^{-4}$\\$(6.39\!\times\!10^{-6})$} & \shortstack[r]{$9.63\!\times\!10^{-5}$\\$(9.70\!\times\!10^{-5})$} & \shortstack[r]{$0.294$\\$(0.00228)$}\\[2pt]
0.3 & 0.3 & 16 & \shortstack[r]{$0.264$\\$(0.00447)$} & \shortstack[r]{$0.197$\\$(0.00357)$} & \shortstack[r]{$0.00208$\\$(0.00222)$} & \shortstack[r]{$0.43$\\$(0.00248)$}\\[2pt]
0.3 & 0.3 & 128 & \shortstack[r]{$0.00986$\\$(1.48\!\times\!10^{-4})$} & \shortstack[r]{$0.00888$\\$(1.26\!\times\!10^{-4})$} & \shortstack[r]{$-0.00116$\\$(4.71\!\times\!10^{-4})$} & \shortstack[r]{$0.0883$\\$(0.00142)$}\\[2pt]
0.3 & 0.3 & 512 & \shortstack[r]{$0.00192$\\$(2.29\!\times\!10^{-5})$} & \shortstack[r]{$0.00186$\\$(2.10\!\times\!10^{-5})$} & \shortstack[r]{$-1.14\!\times\!10^{-4}$\\$(2.16\!\times\!10^{-4})$} & \shortstack[r]{$0.0265$\\$(8.02\!\times\!10^{-4})$}\\[2pt]
0.3 & 0.7 & 16 & \shortstack[r]{$0.573$\\$(0.00774)$} & \shortstack[r]{$0.508$\\$(0.00735)$} & \shortstack[r]{$0.00193$\\$(0.00356)$} & \shortstack[r]{$0.091$\\$(0.00144)$}\\[2pt]
0.3 & 0.7 & 128 & \shortstack[r]{$0.0441$\\$(4.69\!\times\!10^{-4})$} & \shortstack[r]{$0.0426$\\$(4.19\!\times\!10^{-4})$} & \shortstack[r]{$-2.01\!\times\!10^{-4}$\\$(0.00103)$} & \shortstack[r]{$1.25\!\times\!10^{-4}$\\$(5.59\!\times\!10^{-5})$}\\[2pt]
0.3 & 0.7 & 512 & \shortstack[r]{$0.0102$\\$(8.88\!\times\!10^{-5})$} & \shortstack[r]{$0.0101$\\$(8.49\!\times\!10^{-5})$} & \shortstack[r]{$-3.41\!\times\!10^{-4}$\\$(5.03\!\times\!10^{-4})$} & \shortstack[r]{$0$\\$(0)$}\\[2pt]
0.9 & 0.9 & 16 & \shortstack[r]{$0.128$\\$(6.67\!\times\!10^{-4})$} & \shortstack[r]{$0.643$\\$(0.00469)$} & \shortstack[r]{$3.05\!\times\!10^{-6}$\\$(0.00401)$} & \shortstack[r]{$0.0163$\\$(6.34\!\times\!10^{-4})$}\\*[2pt]
0.9 & 0.9 & 128 & \shortstack[r]{$0.0584$\\$(3.37\!\times\!10^{-4})$} & \shortstack[r]{$0.0776$\\$(6.27\!\times\!10^{-4})$} & \shortstack[r]{$0.00252$\\$(0.00139)$} & \shortstack[r]{$0$\\$(0)$}\\*[2pt]
0.9 & 0.9 & 512 & \shortstack[r]{$0.0192$\\$(1.49\!\times\!10^{-4})$} & \shortstack[r]{$0.0193$\\$(1.54\!\times\!10^{-4})$} & \shortstack[r]{$0.00171$\\$(6.94\!\times\!10^{-4})$} & \shortstack[r]{$0$\\$(0)$}\\[2pt]
\end{longtable}
\endgroup

%% file: appendix/table_robust_learning.tex
\begingroup
\small
\setlength{\tabcolsep}{4pt}
\begin{longtable}{rrrllll}
\caption{Robustness suite: learning MSE and realized mean gain (one SE). Here $t^2=\eta/64$, $m=n=k^3$, and $\eta=8192\log(k)/k$. The base learner is magnitude-only; fallback changes the learner.}\label{tab:robust-learning}\\
\toprule
$k$ & $h$ & $\zeta$ & \shortstack{Base\\learning} & \shortstack{Fallback\\learning} & \shortstack{Base\\gain} & \shortstack{Fallback\\gain}\\
\midrule\endfirsthead
\multicolumn{7}{l}{\tablename\ \thetable\ (continued)}\\
\toprule
$k$ & $h$ & $\zeta$ & \shortstack{Base\\learning} & \shortstack{Fallback\\learning} & \shortstack{Base\\gain} & \shortstack{Fallback\\gain}\\
\midrule\endhead
\bottomrule\endfoot
65536 & 0 & 0.0 & \shortstack[r]{$6.99\!\times\!10^{-15}$\\$(7.06\!\times\!10^{-17})$} & \shortstack[r]{$6.99\!\times\!10^{-15}$\\$(7.06\!\times\!10^{-17})$} & \shortstack[r]{$0$\\$(0)$} & \shortstack[r]{$0$\\$(0)$}\\*[2pt]
65536 & $ .95t$ & 0.0 & \shortstack[r]{$0.0195$\\$(1.66\!\times\!10^{-10})$} & \shortstack[r]{$0.0195$\\$(1.66\!\times\!10^{-10})$} & \shortstack[r]{$0$\\$(0)$} & \shortstack[r]{$0$\\$(0)$}\\[2pt]
65536 & $1.05t$ & 0.0 & \shortstack[r]{$6.13\!\times\!10^{-13}$\\$(6.09\!\times\!10^{-15})$} & \shortstack[r]{$6.13\!\times\!10^{-13}$\\$(6.09\!\times\!10^{-15})$} & \shortstack[r]{$0.0239$\\$(1.84\!\times\!10^{-10})$} & \shortstack[r]{$0.0239$\\$(1.84\!\times\!10^{-10})$}\\[2pt]
65536 & .3 & 0.0 & \shortstack[r]{$1.71\!\times\!10^{-13}$\\$(1.72\!\times\!10^{-15})$} & \shortstack[r]{$1.71\!\times\!10^{-13}$\\$(1.72\!\times\!10^{-15})$} & \shortstack[r]{$0.09$\\$(3.56\!\times\!10^{-10})$} & \shortstack[r]{$0.09$\\$(3.56\!\times\!10^{-10})$}\\[2pt]
65536 & 0 & 0.05 & \shortstack[r]{$0.05$\\$(1.73\!\times\!10^{-9})$} & \shortstack[r]{$0.05$\\$(1.73\!\times\!10^{-9})$} & \shortstack[r]{$-0.05$\\$(1.73\!\times\!10^{-9})$} & \shortstack[r]{$-0.05$\\$(1.73\!\times\!10^{-9})$}\\[2pt]
65536 & 0 & 0.2 & \shortstack[r]{$0.2$\\$(1.86\!\times\!10^{-9})$} & \shortstack[r]{$0.2$\\$(1.86\!\times\!10^{-9})$} & \shortstack[r]{$-0.2$\\$(1.86\!\times\!10^{-9})$} & \shortstack[r]{$-0.2$\\$(1.86\!\times\!10^{-9})$}\\[2pt]
262144 & 0 & 0.0 & \shortstack[r]{$1.10\!\times\!10^{-16}$\\$(1.09\!\times\!10^{-18})$} & \shortstack[r]{$1.10\!\times\!10^{-16}$\\$(1.09\!\times\!10^{-18})$} & \shortstack[r]{$0$\\$(0)$} & \shortstack[r]{$0$\\$(0)$}\\[2pt]
262144 & $ .95t$ & 0.0 & \shortstack[r]{$0.0055$\\$(1.11\!\times\!10^{-11})$} & \shortstack[r]{$0.0055$\\$(1.11\!\times\!10^{-11})$} & \shortstack[r]{$0$\\$(0)$} & \shortstack[r]{$0$\\$(0)$}\\[2pt]
262144 & $1.05t$ & 0.0 & \shortstack[r]{$3.28\!\times\!10^{-14}$\\$(3.24\!\times\!10^{-16})$} & \shortstack[r]{$3.28\!\times\!10^{-14}$\\$(3.24\!\times\!10^{-16})$} & \shortstack[r]{$0.00672$\\$(1.22\!\times\!10^{-11})$} & \shortstack[r]{$0.00672$\\$(1.22\!\times\!10^{-11})$}\\[2pt]
262144 & .3 & 0.0 & \shortstack[r]{$2.70\!\times\!10^{-15}$\\$(2.70\!\times\!10^{-17})$} & \shortstack[r]{$2.70\!\times\!10^{-15}$\\$(2.70\!\times\!10^{-17})$} & \shortstack[r]{$0.09$\\$(4.46\!\times\!10^{-11})$} & \shortstack[r]{$0.09$\\$(4.46\!\times\!10^{-11})$}\\[2pt]
262144 & 0 & 0.05 & \shortstack[r]{$0.05$\\$(2.18\!\times\!10^{-10})$} & \shortstack[r]{$0.05$\\$(2.18\!\times\!10^{-10})$} & \shortstack[r]{$-0.05$\\$(2.18\!\times\!10^{-10})$} & \shortstack[r]{$-0.05$\\$(2.18\!\times\!10^{-10})$}\\[2pt]
262144 & 0 & 0.2 & \shortstack[r]{$0.2$\\$(2.32\!\times\!10^{-10})$} & \shortstack[r]{$8.36\!\times\!10^{-6}$\\$(8.37\!\times\!10^{-8})$} & \shortstack[r]{$-0.2$\\$(2.32\!\times\!10^{-10})$} & \shortstack[r]{$-8.36\!\times\!10^{-6}$\\$(8.37\!\times\!10^{-8})$}\\[2pt]
1048576 & 0 & 0.0 & \shortstack[r]{$1.77\!\times\!10^{-18}$\\$(1.75\!\times\!10^{-20})$} & \shortstack[r]{$1.77\!\times\!10^{-18}$\\$(1.75\!\times\!10^{-20})$} & \shortstack[r]{$0$\\$(0)$} & \shortstack[r]{$0$\\$(0)$}\\[2pt]
1048576 & $ .95t$ & 0.0 & \shortstack[r]{$0.00153$\\$(7.25\!\times\!10^{-13})$} & \shortstack[r]{$0.00153$\\$(7.25\!\times\!10^{-13})$} & \shortstack[r]{$0$\\$(0)$} & \shortstack[r]{$0$\\$(0)$}\\[2pt]
1048576 & $1.05t$ & 0.0 & \shortstack[r]{$1.85\!\times\!10^{-15}$\\$(1.84\!\times\!10^{-17})$} & \shortstack[r]{$1.85\!\times\!10^{-15}$\\$(1.84\!\times\!10^{-17})$} & \shortstack[r]{$0.00187$\\$(8.00\!\times\!10^{-13})$} & \shortstack[r]{$0.00187$\\$(8.00\!\times\!10^{-13})$}\\[2pt]
1048576 & .3 & 0.0 & \shortstack[r]{$4.15\!\times\!10^{-17}$\\$(4.12\!\times\!10^{-19})$} & \shortstack[r]{$4.15\!\times\!10^{-17}$\\$(4.12\!\times\!10^{-19})$} & \shortstack[r]{$0.09$\\$(5.61\!\times\!10^{-12})$} & \shortstack[r]{$0.09$\\$(5.61\!\times\!10^{-12})$}\\[2pt]
1048576 & 0 & 0.05 & \shortstack[r]{$0.05$\\$(2.70\!\times\!10^{-11})$} & \shortstack[r]{$1.96\!\times\!10^{-6}$\\$(1.97\!\times\!10^{-8})$} & \shortstack[r]{$-0.05$\\$(2.70\!\times\!10^{-11})$} & \shortstack[r]{$-1.96\!\times\!10^{-6}$\\$(1.97\!\times\!10^{-8})$}\\[2pt]
1048576 & 0 & 0.2 & \shortstack[r]{$0.2$\\$(2.88\!\times\!10^{-11})$} & \shortstack[r]{$2.10\!\times\!10^{-6}$\\$(2.08\!\times\!10^{-8})$} & \shortstack[r]{$-0.2$\\$(2.88\!\times\!10^{-11})$} & \shortstack[r]{$-2.10\!\times\!10^{-6}$\\$(2.08\!\times\!10^{-8})$}\\[2pt]
4194304 & 0 & 0.0 & \shortstack[r]{$2.73\!\times\!10^{-20}$\\$(2.71\!\times\!10^{-22})$} & \shortstack[r]{$2.73\!\times\!10^{-20}$\\$(2.71\!\times\!10^{-22})$} & \shortstack[r]{$0$\\$(0)$} & \shortstack[r]{$0$\\$(0)$}\\[2pt]
4194304 & $ .95t$ & 0.0 & \shortstack[r]{$4.20\!\times\!10^{-4}$\\$(4.77\!\times\!10^{-14})$} & \shortstack[r]{$4.20\!\times\!10^{-4}$\\$(4.77\!\times\!10^{-14})$} & \shortstack[r]{$0$\\$(0)$} & \shortstack[r]{$0$\\$(0)$}\\[2pt]
4194304 & $1.05t$ & 0.0 & \shortstack[r]{$1.06\!\times\!10^{-16}$\\$(1.05\!\times\!10^{-18})$} & \shortstack[r]{$1.06\!\times\!10^{-16}$\\$(1.05\!\times\!10^{-18})$} & \shortstack[r]{$5.13\!\times\!10^{-4}$\\$(5.24\!\times\!10^{-14})$} & \shortstack[r]{$5.13\!\times\!10^{-4}$\\$(5.24\!\times\!10^{-14})$}\\[2pt]
4194304 & .3 & 0.0 & \shortstack[r]{$6.48\!\times\!10^{-19}$\\$(6.60\!\times\!10^{-21})$} & \shortstack[r]{$6.48\!\times\!10^{-19}$\\$(6.60\!\times\!10^{-21})$} & \shortstack[r]{$0.09$\\$(6.98\!\times\!10^{-13})$} & \shortstack[r]{$0.09$\\$(6.98\!\times\!10^{-13})$}\\*[2pt]
4194304 & 0 & 0.05 & \shortstack[r]{$0.05$\\$(3.34\!\times\!10^{-12})$} & \shortstack[r]{$4.91\!\times\!10^{-7}$\\$(4.89\!\times\!10^{-9})$} & \shortstack[r]{$-0.05$\\$(3.34\!\times\!10^{-12})$} & \shortstack[r]{$-4.91\!\times\!10^{-7}$\\$(4.89\!\times\!10^{-9})$}\\*[2pt]
4194304 & 0 & 0.2 & \shortstack[r]{$0.2$\\$(3.64\!\times\!10^{-12})$} & \shortstack[r]{$5.26\!\times\!10^{-7}$\\$(5.25\!\times\!10^{-9})$} & \shortstack[r]{$-0.2$\\$(3.64\!\times\!10^{-12})$} & \shortstack[r]{$-5.26\!\times\!10^{-7}$\\$(5.25\!\times\!10^{-9})$}\\[2pt]
\end{longtable}
\endgroup

%% file: appendix/table_robust_assessment.tex
\begingroup
\small
\setlength{\tabcolsep}{4pt}
\begin{longtable}{rrrllll}
\caption{Robustness suite: MSE (one SE). Raw and projected reports share the same fallback-learner target. The Stein column assesses a different direct-mean learner and is a pipeline comparator, not a same-target dominance comparison.}\label{tab:robust-assessment}\\
\toprule
$k$ & $h$ & $\zeta$ & \shortstack{Base report/\\learner} & \shortstack{Raw /\\fallback} & \shortstack{Projected /\\fallback} & \shortstack{Stein /\\mean}\\
\midrule\endfirsthead
\multicolumn{7}{l}{\tablename\ \thetable\ (continued)}\\
\toprule
$k$ & $h$ & $\zeta$ & \shortstack{Base report/\\learner} & \shortstack{Raw /\\fallback} & \shortstack{Projected /\\fallback} & \shortstack{Stein /\\mean}\\
\midrule\endhead
\bottomrule\endfoot
65536 & 0 & 0.0 & \shortstack[r]{$0$\\$(0)$} & \shortstack[r]{$0$\\$(0)$} & \shortstack[r]{$0$\\$(0)$} & \shortstack[r]{$7.39\!\times\!10^{-9}$\\$(1.88\!\times\!10^{-10})$}\\*[2pt]
65536 & $ .95t$ & 0.0 & \shortstack[r]{$0$\\$(0)$} & \shortstack[r]{$0$\\$(0)$} & \shortstack[r]{$0$\\$(0)$} & \shortstack[r]{$2.43\!\times\!10^{-6}$\\$(2.51\!\times\!10^{-8})$}\\[2pt]
65536 & $1.05t$ & 0.0 & \shortstack[r]{$5.85\!\times\!10^{-14}$\\$(5.82\!\times\!10^{-16})$} & \shortstack[r]{$5.85\!\times\!10^{-14}$\\$(5.82\!\times\!10^{-16})$} & \shortstack[r]{$5.85\!\times\!10^{-14}$\\$(5.82\!\times\!10^{-16})$} & \shortstack[r]{$2.89\!\times\!10^{-6}$\\$(2.90\!\times\!10^{-8})$}\\[2pt]
65536 & .3 & 0.0 & \shortstack[r]{$6.17\!\times\!10^{-14}$\\$(6.21\!\times\!10^{-16})$} & \shortstack[r]{$6.17\!\times\!10^{-14}$\\$(6.21\!\times\!10^{-16})$} & \shortstack[r]{$6.17\!\times\!10^{-14}$\\$(6.21\!\times\!10^{-16})$} & \shortstack[r]{$1.15\!\times\!10^{-5}$\\$(1.14\!\times\!10^{-7})$}\\[2pt]
65536 & 0 & 0.05 & \shortstack[r]{$0.01$\\$(6.87\!\times\!10^{-10})$} & \shortstack[r]{$0.01$\\$(6.87\!\times\!10^{-10})$} & \shortstack[r]{$0.00179$\\$(9.18\!\times\!10^{-7})$} & \shortstack[r]{$7.68\!\times\!10^{-9}$\\$(1.99\!\times\!10^{-10})$}\\[2pt]
65536 & 0 & 0.2 & \shortstack[r]{$0.16$\\$(2.87\!\times\!10^{-9})$} & \shortstack[r]{$0.16$\\$(2.87\!\times\!10^{-9})$} & \shortstack[r]{$0.00719$\\$(3.81\!\times\!10^{-6})$} & \shortstack[r]{$9.18\!\times\!10^{-9}$\\$(2.56\!\times\!10^{-10})$}\\[2pt]
262144 & 0 & 0.0 & \shortstack[r]{$0$\\$(0)$} & \shortstack[r]{$0$\\$(0)$} & \shortstack[r]{$0$\\$(0)$} & \shortstack[r]{$4.85\!\times\!10^{-10}$\\$(1.37\!\times\!10^{-11})$}\\[2pt]
262144 & $ .95t$ & 0.0 & \shortstack[r]{$0$\\$(0)$} & \shortstack[r]{$0$\\$(0)$} & \shortstack[r]{$0$\\$(0)$} & \shortstack[r]{$1.69\!\times\!10^{-7}$\\$(1.74\!\times\!10^{-9})$}\\[2pt]
262144 & $1.05t$ & 0.0 & \shortstack[r]{$8.82\!\times\!10^{-16}$\\$(8.71\!\times\!10^{-18})$} & \shortstack[r]{$8.82\!\times\!10^{-16}$\\$(8.71\!\times\!10^{-18})$} & \shortstack[r]{$8.82\!\times\!10^{-16}$\\$(8.71\!\times\!10^{-18})$} & \shortstack[r]{$2.07\!\times\!10^{-7}$\\$(2.11\!\times\!10^{-9})$}\\[2pt]
262144 & .3 & 0.0 & \shortstack[r]{$9.71\!\times\!10^{-16}$\\$(9.72\!\times\!10^{-18})$} & \shortstack[r]{$9.71\!\times\!10^{-16}$\\$(9.72\!\times\!10^{-18})$} & \shortstack[r]{$9.71\!\times\!10^{-16}$\\$(9.72\!\times\!10^{-18})$} & \shortstack[r]{$2.93\!\times\!10^{-6}$\\$(2.91\!\times\!10^{-8})$}\\[2pt]
262144 & 0 & 0.05 & \shortstack[r]{$0.01$\\$(8.65\!\times\!10^{-11})$} & \shortstack[r]{$0.01$\\$(8.65\!\times\!10^{-11})$} & \shortstack[r]{$5.01\!\times\!10^{-4}$\\$(2.43\!\times\!10^{-7})$} & \shortstack[r]{$4.96\!\times\!10^{-10}$\\$(1.36\!\times\!10^{-11})$}\\[2pt]
262144 & 0 & 0.2 & \shortstack[r]{$0.16$\\$(3.60\!\times\!10^{-10})$} & \shortstack[r]{$0.04$\\$(3.35\!\times\!10^{-8})$} & \shortstack[r]{$9.19\!\times\!10^{-8}$\\$(1.01\!\times\!10^{-9})$} & \shortstack[r]{$5.60\!\times\!10^{-10}$\\$(1.61\!\times\!10^{-11})$}\\[2pt]
1048576 & 0 & 0.0 & \shortstack[r]{$0$\\$(0)$} & \shortstack[r]{$0$\\$(0)$} & \shortstack[r]{$0$\\$(0)$} & \shortstack[r]{$2.82\!\times\!10^{-11}$\\$(7.58\!\times\!10^{-13})$}\\[2pt]
1048576 & $ .95t$ & 0.0 & \shortstack[r]{$0$\\$(0)$} & \shortstack[r]{$0$\\$(0)$} & \shortstack[r]{$0$\\$(0)$} & \shortstack[r]{$1.17\!\times\!10^{-8}$\\$(1.18\!\times\!10^{-10})$}\\[2pt]
1048576 & $1.05t$ & 0.0 & \shortstack[r]{$1.38\!\times\!10^{-17}$\\$(1.37\!\times\!10^{-19})$} & \shortstack[r]{$1.38\!\times\!10^{-17}$\\$(1.37\!\times\!10^{-19})$} & \shortstack[r]{$1.38\!\times\!10^{-17}$\\$(1.37\!\times\!10^{-19})$} & \shortstack[r]{$1.44\!\times\!10^{-8}$\\$(1.47\!\times\!10^{-10})$}\\[2pt]
1048576 & .3 & 0.0 & \shortstack[r]{$1.50\!\times\!10^{-17}$\\$(1.48\!\times\!10^{-19})$} & \shortstack[r]{$1.50\!\times\!10^{-17}$\\$(1.48\!\times\!10^{-19})$} & \shortstack[r]{$1.50\!\times\!10^{-17}$\\$(1.48\!\times\!10^{-19})$} & \shortstack[r]{$7.15\!\times\!10^{-7}$\\$(7.25\!\times\!10^{-9})$}\\[2pt]
1048576 & 0 & 0.05 & \shortstack[r]{$0.01$\\$(1.07\!\times\!10^{-11})$} & \shortstack[r]{$0.0025$\\$(1.97\!\times\!10^{-9})$} & \shortstack[r]{$5.91\!\times\!10^{-9}$\\$(6.45\!\times\!10^{-11})$} & \shortstack[r]{$3.10\!\times\!10^{-11}$\\$(8.95\!\times\!10^{-13})$}\\[2pt]
1048576 & 0 & 0.2 & \shortstack[r]{$0.16$\\$(4.46\!\times\!10^{-11})$} & \shortstack[r]{$0.04$\\$(8.32\!\times\!10^{-9})$} & \shortstack[r]{$6.35\!\times\!10^{-9}$\\$(6.83\!\times\!10^{-11})$} & \shortstack[r]{$3.46\!\times\!10^{-11}$\\$(8.61\!\times\!10^{-13})$}\\[2pt]
4194304 & 0 & 0.0 & \shortstack[r]{$0$\\$(0)$} & \shortstack[r]{$0$\\$(0)$} & \shortstack[r]{$0$\\$(0)$} & \shortstack[r]{$1.80\!\times\!10^{-12}$\\$(4.82\!\times\!10^{-14})$}\\[2pt]
4194304 & $ .95t$ & 0.0 & \shortstack[r]{$0$\\$(0)$} & \shortstack[r]{$0$\\$(0)$} & \shortstack[r]{$0$\\$(0)$} & \shortstack[r]{$8.09\!\times\!10^{-10}$\\$(8.21\!\times\!10^{-12})$}\\[2pt]
4194304 & $1.05t$ & 0.0 & \shortstack[r]{$2.18\!\times\!10^{-19}$\\$(2.15\!\times\!10^{-21})$} & \shortstack[r]{$2.18\!\times\!10^{-19}$\\$(2.15\!\times\!10^{-21})$} & \shortstack[r]{$2.18\!\times\!10^{-19}$\\$(2.15\!\times\!10^{-21})$} & \shortstack[r]{$9.65\!\times\!10^{-10}$\\$(9.67\!\times\!10^{-12})$}\\[2pt]
4194304 & .3 & 0.0 & \shortstack[r]{$2.33\!\times\!10^{-19}$\\$(2.38\!\times\!10^{-21})$} & \shortstack[r]{$2.33\!\times\!10^{-19}$\\$(2.38\!\times\!10^{-21})$} & \shortstack[r]{$2.33\!\times\!10^{-19}$\\$(2.38\!\times\!10^{-21})$} & \shortstack[r]{$1.77\!\times\!10^{-7}$\\$(1.77\!\times\!10^{-9})$}\\*[2pt]
4194304 & 0 & 0.05 & \shortstack[r]{$0.01$\\$(1.33\!\times\!10^{-12})$} & \shortstack[r]{$0.0025$\\$(4.89\!\times\!10^{-10})$} & \shortstack[r]{$4.04\!\times\!10^{-10}$\\$(4.34\!\times\!10^{-12})$} & \shortstack[r]{$1.91\!\times\!10^{-12}$\\$(4.94\!\times\!10^{-14})$}\\*[2pt]
4194304 & 0 & 0.2 & \shortstack[r]{$0.16$\\$(5.64\!\times\!10^{-12})$} & \shortstack[r]{$0.04$\\$(2.10\!\times\!10^{-9})$} & \shortstack[r]{$4.35\!\times\!10^{-10}$\\$(4.71\!\times\!10^{-12})$} & \shortstack[r]{$2.20\!\times\!10^{-12}$\\$(5.98\!\times\!10^{-14})$}\\[2pt]
\end{longtable}
\endgroup

%% file: appendix/table_bound_sensitivity.tex
\begin{table}[t]
\centering
\small
\caption{Variance-bound ablation at $k=16384$, $m=n=k^3$, $h^2=1.1\eta$, $c=.6$, $\tau=2$, and fast slack. The rule uses $\rho V_*$ with $V_*=4$; the data and allowance are unchanged within a cell. $P(T>0,I_Q=1)$ is the nuisance-draw fraction with a positive eligible threshold. The last column compares total MSE with the matched magnitude-only rule. Negative normalized exponents at small $z$ are permitted.}\label{tab:bound-sensitivity}
\setlength{\tabcolsep}{4pt}
\begin{tabular}{rrrrrr}
\toprule
$z$ & $\rho$ & $P(T>0,I_Q=1)$ & $-\log(k^2R_W)/z$ & $R/E_k$ & $R/R_{\rm mag}$\\
\midrule
16 & 1 & $0$ & $-0.126$ & $6.13\times10^{4}$ & $1$\\
16 & 2 & $0$ & $-0.126$ & $6.13\times10^{4}$ & $1$\\
16 & 4 & $0$ & $-0.126$ & $6.13\times10^{4}$ & $1$\\
16 & 16 & $0$ & $-0.126$ & $6.13\times10^{4}$ & $1$\\
64 & 1 & $0$ & $0.141$ & $8.99$ & $1$\\
64 & 2 & $0$ & $0.141$ & $8.99$ & $1$\\
64 & 4 & $0$ & $0.141$ & $8.99$ & $1$\\
64 & 16 & $0$ & $0.141$ & $8.99$ & $1$\\
256 & 1 & $1$ & $0.433$ & $8.24$ & $1$\\
256 & 2 & $1$ & $0.251$ & $8.24$ & $1$\\
256 & 4 & $0$ & $0.231$ & $8.24$ & $1$\\
256 & 16 & $0$ & $0.231$ & $8.24$ & $1$\\
512 & 1 & $1$ & $0.587$ & $8.11$ & $1$\\
512 & 2 & $1$ & $0.431$ & $8.11$ & $1$\\
512 & 4 & $1$ & $0.249$ & $8.11$ & $1$\\
512 & 16 & $0$ & $0.247$ & $8.11$ & $1$\\
\bottomrule
\end{tabular}
\end{table}